\documentclass{article}

\usepackage{arxiv}

\usepackage[utf8]{inputenc} 
\usepackage[T1]{fontenc}    
\usepackage{hyperref}       
\usepackage{url}            
\usepackage{booktabs}       
\usepackage{amsfonts}       
\usepackage{nicefrac}       
\usepackage{microtype}      
\usepackage{amssymb}        
\usepackage{lipsum}
\usepackage{graphicx}
\graphicspath{ {./images/} }

\usepackage[authoryear,longnamesfirst]{natbib}

\usepackage{amsmath} 
\usepackage{subcaption}
\usepackage{tikz}
\usetikzlibrary{positioning, arrows.meta, shapes.multipart, fit, calc, backgrounds, decorations.pathreplacing,shadows.blur,shapes.geometric}
\usetikzlibrary{shapes.multipart}
\usepackage{xcolor}
\usepackage{listings}
\usepackage{geometry}

\newtheorem{proposition}{Proposition}
\newtheorem{lemma}{Lemma}

\newtheorem{definition}{Definition}%
\newtheorem{proof}{Proof}

\title{Forgetful but Faithful: \\ A Cognitive Memory Architecture and Benchmark for Privacy‑Aware Generative Agents}

\author{
 Saad Alqithami \\
  Computer Science Depratment, 
  Al-Baha University\\
  \texttt{salqithami@bu.edu.sa} \\
}

\begin{document}
\maketitle
\begin{abstract}
Generative agents must manage long‑horizon memories under strict budget and privacy constraints. We present the Memory‑Aware Retention Schema (MaRS), a cognitively inspired architecture that organizes episodic, semantic, social, and task memories as typed, provenance‑tracked nodes with multiple indices for efficient retrieval. On top of MaRS we formalize six forgetting policies—FIFO, LRU, Priority Decay, Reflection‑Summary, Random‑Drop, and a Hybrid variant—providing complexity analyses and sensitivity‑aware retention with optional $(\varepsilon,\delta)$‑differential privacy guarantees. To evaluate memory‑constrained behavior, we introduce the FiFA benchmark, which measures narrative coherence, goal completion, social recall, privacy preservation, and cost efficiency. Across 300 simulation runs spanning five memory budgets, the Hybrid policy delivers the best composite performance ($\approx$0.911) while maintaining tractable cost and high privacy scores. Our results offer theory‑backed guidance on policy selection and budget sizing, and demonstrate that principled forgetting‑by‑design can simultaneously support coherence, efficiency, and privacy in cognitive agents. The framework, analyses, and benchmark advance memory as a first‑class design concern for cognitive systems and responsible AI.
\end{abstract}


\section{Introduction}

Large language models (LLMs) operating as generative agents are increasingly deployed in open‑ended interactions that span hours or days, blending dialogue, planning, tool use, and reflection across sessions \citep{park2023generative, wang2024survey}. In such settings, the capacity to \emph{remember} and \emph{forget} becomes a first‑class design variable rather than a peripheral implementation detail. The problem resembles long‑standing questions in human memory research: agents must organize experiences into episodic traces, consolidate them into semantic knowledge, track social relationships, and preserve task context under bounded resources \citep{baddeley2002episodic, tulving2002episodic}. Unlike traditional stateless systems, however, contemporary agents continuously accrue interaction history. Unchecked growth stresses inference time through long contexts, increases retrieval noise, and amplifies privacy risk by retaining sensitive details beyond any reasonable need \citep{carlini2021extracting, brown2022does}.

Cognitive architectures have long separated episodic, semantic, and procedural knowledge and introduced activation‑ or recency‑based decay to manage capacity \citep{anderson1983architecture}. Modern agent frameworks have revived these ideas in neural settings, adding reflective summarization and retrieval‑augmented generation to stabilize behavior over long horizons \citep{schank1977scripts, shinn2023reflexion}. Yet current practice often oscillates between two extremes. At one end, systems defer deletion entirely, relying on ever‑longer contexts and vector stores; this improves short‑term recall but degrades latency and raises privacy and governance concerns. At the other, ad‑hoc pruning (e.g., fixed windows, random drops) preserves efficiency at the cost of narrative coherence, goal continuity, and social appropriateness \citep{liu2024agentbench}. Neither approach offers principled guarantees about what is retained, what is forgotten, and why.

This paper advances \emph{forgetting‑by‑design} as a human‑centered principle for generative agents: memory should be structured, budgeted, and explicable, with retention decisions aligned to task value and privacy norms. We introduce the \emph{Memory‑Aware Retention Schema} (MaRS), an ontological and operational layer that represents memories as typed nodes with provenance, sensitivity, and token‑weight metadata, linked by relations that support efficient retrieval and principled removal. MaRS is paired with a palette of forgetting policies that range from temporal heuristics (FIFO, LRU) to importance‑aware methods (priority decay), reflective consolidation (summary‑based compression), and a hybrid scheme that stages these mechanisms to balance fidelity with cost. Crucially, retention decisions can be modulated by a privacy engine that supports sensitivity weighting and optional differentially private noise injection, bringing the practice closer to emerging expectations around transparency, accountability, and the right to be forgotten \citep{amershi2019guidelines, jobin2019global}.

Evaluating memory management requires metrics that capture more than task success. We therefore propose the \emph{Forgetful but Faithful Agent} (FiFA) benchmark, which measures narrative coherence across turns, completion of multi‑step goals, accuracy of social recall, privacy leakage, and cost efficiency under explicit token budgets. FiFA is implemented as a multi‑agent simulation with controllable pressures on memory growth and retrieval selectivity, allowing fair comparisons across policies and budgets. In contrast to single‑task or purely offline evaluations, FiFA targets the lived properties of agent interaction—maintaining continuity with a user while avoiding unnecessary retention of sensitive details.

Our empirical study shows that purposeful forgetting can improve both user‑facing quality and governance posture. Across extensive runs and budget settings, hybrid policy variants consistently preserve coherence and social recall while keeping leakage low and costs tractable, whereas naive strategies either overspend on context or erode interaction quality. Beyond aggregate scores, MaRS's audit and provenance traces render retention choices interpretable post hoc, a prerequisite for trustworthy deployment in domains that demand accountability.

The contributions of this work are conceptual, algorithmic, and evaluative. Conceptually, MaRS frames memory as a relational, provenance‑aware store with explicit budgets and privacy semantics. Algorithmically, we design and analyze a family of forgetting policies, including reflection‑based consolidation and sensitivity‑aware selection with differential privacy guarantees. Evaluatively, FiFA operationalizes human‑centered criteria for memory‑budgeted agents and demonstrates that principled forgetting can be both \emph{faithful}—preserving what matters for interaction—and \emph{forgetful}—discarding what should not be retained. Together, these elements chart a path toward agents that are not only capable over long horizons, but also efficient, respectful of privacy, and understandable by design.

\section{Related Work}

The management of agent memory under computational and ethical constraints sits at the intersection of cognitive architectures, long‑context language modeling, privacy‑preserving machine learning, and human–computer interaction. This section situates our contributions with respect to these lines of work and clarifies how \emph{MaRS} and \emph{FiFA} advance the state of the art.

\subsection{Cognitive Architectures and Memory Systems}

Classical cognitive architectures established memory as a first–class computational substrate that supports perception, reasoning, and action across time. Systems such as ACT‑R and Soar formalized separations between declarative and procedural knowledge, specified mechanisms for retrieval and spreading activation, and modeled forgetting through time‑dependent decay and interference \citep{anderson1983architecture,anderson2004integrated,laird2012soar}. In ACT‑R, for example, the base‑level activation of a memory trace predicts its accessibility as a function of recency and frequency of use, capturing empirically grounded regularities of human recall. Soar’s learning via chunking similarly accumulates generalized productions from episodic experience while relying on a limited working memory that forces selective retention. These traditions also differentiated long‑term stores: episodic traces that bind events to temporal and situational context, and semantic networks that encode atemporal relations among concepts \citep{tulving2002episodic,baddeley2002episodic}. Together they anticipate design decisions that modern neural agents must revisit: how to separate types of content, how to prioritize retrieval, and how to forget without undermining behavioral coherence.

Subsequent work brought episodic memory to the foreground of computational models of agents. Integrated episodic modules were shown to enable long‑horizon credit assignment, autobiographical reasoning, and adaptation to user‑specific regularities \citep{nuxoll2004cognitive,tecuci1998building}. Episodic traces preserve fine‑grained temporal structure and provenance, supporting narrative continuity and post‑hoc explanation, whereas semantic stores trade detail for generality and transfer. Yet the symbolic implementations that pioneered these distinctions were engineered for modest working sets and narrow task environments. They do not directly address the cost profiles of today’s language‑model agents whose interaction histories can span millions of tokens, nor do they provide mechanisms for privacy governance over personal information accumulating in long‑running dialogues.

Neural memory research extended these traditions with continuous controllers over external stores. Neural Turing Machines and Differentiable Neural Computers introduced content‑ and location‑based addressing, differentiable read/write heads, and temporal linkages that support algorithmic tasks requiring persistent state \citep{graves2014neural,graves2016hybrid}. Memory‑augmented meta‑learning architectures further demonstrated rapid adaptation by writing task‑specific information into short‑term slots that can be read on demand \citep{santoro2016meta}. While these models validate that learned controllers can operate over external memory, they target bounded problems with short horizons and stationary distributions; they do not directly confront months‑long, open‑world interactions with humans, nor do they specify retention policies when memory growth is constrained by explicit budgets or governance requirements.

Bridging classical and neural perspectives highlights three gaps that motivate our formulation. First, \emph{typed memory} remains essential: different classes of content (eventful episodes, atemporal facts, social relations, and task state) exhibit distinct dynamics of usefulness, risk, and decay. Second, \emph{retrieval} must be coupled to \emph{governed retention}: what is made available to the policy for action should be shaped by principled decisions about what is kept, summarized, or discarded over time. Third, \emph{provenance and auditability}—central to symbolic systems—are largely absent in neural stores, yet they are prerequisites for trust, reproducibility, and regulatory compliance in human‑facing applications.

Our Memory‑as‑Relational‑Store (MaRS) operationalizes this synthesis. Conceptually, MaRS preserves the classical separation of memory types while instantiating them in a graph‑structured store whose nodes carry content, timestamps, weights (token cost), sensitivity scores, and provenance, and whose edges encode temporal, causal, semantic, and social relations. This representation retains the inspectability and compositionality of symbolic memory, yet is designed to interoperate with neural retrieval and summarization. Practically, MaRS treats retention as a resource‑allocation problem: under explicit token budgets, the store applies forgetting policies that generalize classical decay and recency effects to modern settings by jointly considering access statistics, semantic centrality, task relevance, and privacy sensitivity. In contrast to memory‑augmented networks that optimize differentiable read/write operations but leave governance unspecified, MaRS makes the \emph{policy}—not only the mechanism—the locus of design and evaluation.

Finally, the MaRS abstraction supports properties that are difficult to realize in either tradition alone. Typed nodes with provenance enable explanation of why a particular item was recalled or evicted, echoing the explanatory benefits of symbolic systems. Weighted costs and sensitivity scores make budgeting and privacy explicit, allowing retention decisions to be tied to user norms and legal constraints. Relational structure supports consolidation by reflection—compressing subgraphs of related episodes into semantic summaries—thereby connecting neural summarization capabilities to a principled data model. In sum, MaRS inherits the cognitive plausibility of classical architectures and the operational flexibility of neural memory, while adding the budgeting and governance layers required for long‑lived, human‑centered generative agents.

\subsection{Generative Agents and Memory Management}

LLM‑driven agents have shifted the problem of memory from a peripheral implementation detail to a core design variable: agents now sustain goals across days, reuse prior experiences to shape future behavior, and adapt to evolving social contexts \citep{park2023generative}. In such settings, the memory substrate does far more than cache dialogue tokens; it becomes the medium through which the agent maintains identity and continuity. Reflection mechanisms---periodic self‑summaries or critiques of prior actions---improve self‑consistency and reduce local myopia, but they do not, by themselves, regulate the unbounded growth of stored traces \citep{shinn2023reflexion}. As interactions lengthen, the quadratic cost of attention over long contexts and the accumulation of semantically redundant or weakly relevant episodes introduce retrieval noise, diluting useful signal and making naive “remember everything’’ strategies untenable \citep{brown2020language,openai2023gpt4}. These pressures are well documented in recent surveys that call explicitly for principled memory budgeting and policy‑level evaluation rather than ad‑hoc windowing \citep{wang2024survey,xi2025rise}. The empirical sections of our draft similarly show that, absent explicit retention control, costs rise while narrative coherence and social recall degrade as the store expands, reinforcing the need for policy‑aware memory design.

A second challenge emerges from the multifaceted nature of the information agents must retain. Episodic traces support situated continuity, semantic facts enable generalization beyond specific situations, social representations scaffold personalization and trust, and task memories tie intentions to procedures and deadlines. Treating these distinct substrates as a single undifferentiated buffer encourages pathologies: recent but trivial facts crowd out durable knowledge; low‑value social snippets outlive expired tasks; and sensitive details persist beyond their utility horizon. Classical cognitive architectures anticipated some of these tensions, for example through activation‑based decay and separate declarative/procedural stores, but they targeted symbolic workloads with modest data volumes rather than the torrent of natural‑language traces produced by modern agents \citep{anderson1983architecture,laird2012soar,anderson2004integrated}. The gap is therefore not merely one of scale: people evaluate agents through human‑facing properties---consistency, discretion, and appropriate forgetting---and memory systems must reconcile those socio‑cognitive expectations with computational constraints \citep{shneiderman2020human,amershi2019guidelines,lee2004trust,hoff2015trust}. Our manuscript frames this reconciliation explicitly as human‑centered design: it argues that an agent that forgets salient user preferences erodes trust, while an agent that never forgets risks privacy harms and regulatory non‑compliance.

A purely temporal view of forgetting, whether via fixed windows or least‑recently‑used eviction, yields predictable behavior and low overhead but discards information that is old yet still instrumental for coherence and goals. Conversely, importance‑centric strategies mitigate this failure mode but require principled scoring and can be brittle without temporal safeguards. Reflection‑based consolidation promises a middle path by compressing clusters of related episodes into durable summaries, yet the act of deciding \emph{which} clusters to compress and \emph{when} to retire the constituents is itself a policy choice that must respect resource budgets and privacy constraints \citep{shinn2023reflexion}. The \emph{MaRS} framework in our paper makes these choices explicit: memory is organized as typed nodes with provenance and sensitivity attributes; and retention becomes a policy‑selection problem subject to a budget. In later sections we show that this reframing yields measurable gains in narrative coherence, goal completion, social recall, and privacy preservation under fixed budgets, compared with purely temporal baselines.

Finally, the memory problem is inseparable from evaluation. Popular agent benchmarks privilege tool‑use accuracy or web/GUI control, whereas memory competence is multi‑criteria and long‑horizon \citep{liu2024agentbench,xi2025rise}. Our \emph{FiFA} suite responds by operationalizing faithfulness under memory constraints, including leakage‑aware metrics aligned with privacy concerns. This evaluation lens is crucial: longer context windows or faster attention kernels help but do not answer the normative question of what an agent ought to retain and for how long. Our results in the manuscript show that, when assessed along human‑centered axes rather than single‑task accuracy, sophisticated retention policies dominate naive ones despite modest computational overheads.

\subsection{Memory‑Augmented LLMs and Virtual Memory}

A parallel line of work augments neural systems with external memory, from early differentiable controllers with content‑addressable tapes to contemporary LLM toolchains that page between short‑term and long‑term stores \citep{graves2014neural,graves2016hybrid,santoro2016meta}. Recent agent systems go further, layering “virtual memory’’ abstractions on top of LLMs. \emph{MemGPT} adopts an operating‑system metaphor in which the agent pages relevant snippets between a compact working set and a larger archival store; \emph{MemoryBank} attaches human‑like decay and lightweight importance cues to stabilize growth; \emph{LongMem} separates memory encoding from response generation to improve longevity and recall \citep{packer2024memgptllmsoperatingsystems,zhong2024memorybank,wang2023augmenting}. In the transformer stack itself, efficient attention implementations and length‑generalization schemes increase the affordable horizon of context, lowering latency and amortized cost per token \citep{dao2024flashattention,ding2024longrope}.

These advances address where information can be \emph{stored} and how it can be \emph{found} at inference time, but they typically leave unspecified which items \emph{should} persist under tight budgets, how to trade off recency against instrumental value, and how to integrate privacy sensitivity into eviction choices. Retrieval‑augmented generation pipelines illustrate the point: vector indices and hybrid search reduce miss rates, yet the index itself grows without global retention governance, and as the corpus expands, retrieval noise increases unless admission control is applied. Our framework targets precisely this gap by elevating retention to a first‑class policy objective. Within \emph{MaRS}, memory nodes carry token weight, sensitivity scores, and provenance; forgetting policies are functions that transform an over‑budget store into a compliant one while optimizing a utility surrogate connected to task performance and privacy. The accompanying \emph{FiFA} benchmark stresses this decision space with long‑horizon, multi‑agent simulations and leakage‑aware metrics, thereby complementing engineering‑level efficiency gains with normative selection criteria. The manuscript documents how this integration changes outcomes: policies that combine temporal hygiene with importance estimation and reflection‑driven consolidation deliver higher composite scores and markedly better privacy preservation than windowed or random baselines at the same budget.

Conceptually, treating memory as a \emph{governed resource} aligns these systems with both cognitive theories that model activation and decay, and with legal‑ethical regimes that recognize the right to be forgotten \citep{anderson1983architecture,laird2012soar,dwork2006calibrating}. It also provides a bridge from low‑level hardware or architectural optimizations to human‑level desiderata: improved attention kernels and longer position encodings expand the feasible working set; retention policies decide which elements deserve a place within it; and evaluation protocols such as \emph{FiFA} verify that these choices produce agents that are coherent, useful, and discreet. In that sense, \emph{MaRS} and \emph{FiFA} are orthogonal complements to memory‑augmented LLMs and long‑context transformers: the former specify what to keep and why; the latter make it cheaper to hold and access what is kept.

\subsection{Privacy‑Preserving AI and the Right to be Forgotten}

Modern agentic systems routinely process personally identifiable and sensitive information over extended interactions, which elevates privacy from an implementation detail to a first‑order design constraint. The European Union’s General Data Protection Regulation (GDPR) codifies an actionable ``right to be forgotten’’ (RTBF), requiring systems to delete or render inaccessible personally identifiable data on demand, and to do so in ways that are auditable and defensible \citep{mantelero2013eu,rosen2012right,VoigtBussche2024GDPR}. In learning systems, two methodological families have emerged in response. The first, \emph{machine unlearning}, proposes to remove the influence of specific examples from trained models, either by exact deletion with certified guarantees on the remaining predictor or by practical approximations that trade fidelity for speed \citep{cao2015towards,ginart2019making,bourtoule2021machine,guo2020certified}. The second, differential privacy (DP), ensures that model outputs reveal little about any individual record by injecting carefully calibrated noise at training time, with composition theorems enabling quantification of cumulative privacy loss \citep{dwork2006calibrating,abadi2016deep}. While both lines have borne fruit, they primarily address \emph{training data} and model publication contexts. Interactive generative agents face a different, explicitly \emph{run‑time} challenge: deciding whether to retain, summarize, or discard new memories as they are created; determining how to expunge already‑stored items in response to user preference changes; and auditing these decisions over long horizons.

Our approach addresses this online setting directly by treating retention as a constrained decision problem at the memory layer rather than solely a property of the model. In MaRS, every memory node carries a sensitivity score and provenance, and forgetting policies operate over these attributes under a global token budget. A DP‑aware decision layer can randomize marginal retention choices for near‑threshold items, yielding $(\epsilon,\delta)$‑style guarantees at the \emph{policy} boundary without forcing wholesale retraining of the underlying LLM. This separation of concerns allows organizations to honor RTBF requests with bounded latency, to prefer summarization over verbatim retention when privacy risks are high, and to produce auditable trails explaining when and why sensitive items were compressed or removed. Compared with post‑hoc unlearning on parameters, policy‑level forgetting is fast, targeted to the actual artifacts at risk of disclosure, and compatible with deployment‑time controls such as region‑specific retention rules and user‑level privacy preferences in multi‑tenant settings.

\subsection{Human‑Centered AI, Trust, and Faithfulness}

Trust in interactive AI systems depends not only on the factual quality of responses but also on whether behavior matches user expectations about memory, discretion, and continuity. Human‑centered design guidelines emphasize predictability, appropriate levels of control, and comprehensible failure modes \citep{amershi2019guidelines,shneiderman2020human}. In conversational contexts, users expect agents to remember preferences that matter for the task while gracefully letting go of stale, incidental, or sensitive details \citep{luger2016like,clark2019makes,brandtzaeg2017chatbots,folstad2021future}. The trust literature further stresses calibrated reliance: users must be able to anticipate when the system will recall, when it will ask for clarification, and when it will decline to surface potentially private content \citep{lee2004trust,hoff2015trust}. Recent accounts argue for a broadened notion of \emph{faithfulness} that extends beyond local explanation fidelity to include temporal coherence with user norms and intentions \citep{jacovi2021formalizing}.

MaRS operationalizes these principles by making memory decisions explicit, typed, and inspectable. Typed stores (episodic, semantic, social, task) align with users’ mental models of what an assistant should and should not persist. Provenance fields record how an item entered memory, enabling agents to justify retention or deletion relative to source reliability and consent status. Audit trails capture the rationale of forgetting policies—e.g., age, low importance, or high sensitivity—so that users and administrators can understand the causes of apparent lapses and adjust policy parameters when trust is threatened. In aggregate, these design choices turn retention from an opaque side effect of ever‑growing context windows into a controllable, communicable facet of the agent’s behavior, thereby supporting appropriate reliance over time.

\subsection{Benchmarks and Evaluation of Agents}

The evaluation of LLM‑based agents has moved rapidly from single‑task leaderboards to multi‑capability testbeds that exercise tool use, web navigation, and desktop control \citep{liu2024agentbench,zhou2024webarena,xie2024osworld}. Social reasoning suites such as SocKET probe interpersonal awareness and pragmatic inference \citep{choi2023llms}. In parallel, scalable assessment practices increasingly adopt LLM‑as‑a‑judge protocols—e.g., MT‑Bench and G‑Eval—that, when equipped with carefully designed rubrics and calibration steps, show non‑trivial agreement with expert annotators while keeping costs manageable \citep{zheng2023mtbench,liu2023geval}. Cost‑aware evaluation has also gained prominence, reflecting the operational reality that context growth, tool calls, and retrieval operations translate directly into latency and monetary expenditure \citep{chen2024frugalgpt}.

Despite these advances, existing suites offer limited visibility into \emph{memory governance}. They rarely vary the memory budget, do not track how coherence and goal pursuit degrade as stores are pruned, and seldom measure privacy leakage as an explicit outcome. FiFA is designed to supply this missing axis. It evaluates agents under explicit token budgets and reports metrics that jointly reflect utility and responsibility: narrative coherence across sessions, task completion under pressure, social recall accuracy, privacy leakage per turn, and token‑cost efficiency. By pairing these outcomes with controlled retention policies, FiFA enables principled comparisons across design choices and provides evidence about the real trade‑offs deployment teams face when they cannot ``remember everything’’ indefinitely.

\subsection{Forgetting in Learning Systems: From Catastrophic to Intentional}

The machine learning literature on forgetting has historically focused on \emph{catastrophic forgetting}: the tendency of neural networks to overwrite previously acquired knowledge when trained sequentially on new tasks \citep{mccloskey1989catastrophic}. Prominent mitigations constrain parameter drift through regularization or architectural isolation—Elastic Weight Consolidation limits changes to weights deemed important for prior tasks, while Progressive Neural Networks grow new pathways that reuse frozen features \citep{kirkpatrick2017overcoming,rusu2016progressive}. These methods illuminate stability–plasticity trade‑offs during \emph{training}, but they do not address what interactive agents must decide at \emph{inference} time: which \emph{external} representations to keep, compress, or delete as conversations unfold.

Intentional forgetting for privacy and efficiency is therefore a distinct problem. Rather than preventing parameter interference, the goal is to govern the lifecycle of memory artifacts that sit outside the model—dialogue snippets, plans, social facts, and retrieved knowledge—so that the agent remains coherent, frugal, and privacy‑preserving. Machine unlearning is adjacent but targets the parameters of a model trained on static datasets \citep{bourtoule2021machine,guo2020certified}; it provides limited guidance on streaming scenarios where new memories are continually created and must be triaged under budget. MaRS fills this gap by giving agents policy‑based control over external memory: temporally biased policies maintain recency, importance‑aware policies preserve task‑critical and socially salient items, and privacy‑weighted policies accelerate the decay of sensitive content. Evaluated in FiFA, these policies expose how different forgetting choices shape long‑horizon competence and leakage, turning the management of agent memory into an object of scientific and engineering study in its own right.

\subsection{Positioning of Our Work}

MaRS and FiFA are designed to sit at the intersection of cognitive architectures, memory‑augmented language models, and privacy‑aware AI, but they target a gap that none of these streams fully address: \emph{runtime governance of long‑term agent memory under explicit budgets and social constraints}. In contrast to systems that primarily enlarge or accelerate working memory—through paging layers, decoupled encoders, or efficient attention \citep{packer2024memgptllmsoperatingsystems,wang2023augmenting,zhong2024memorybank,dao2024flashattention,ding2024longrope}--MaRS treats retention as a first‑class decision problem. The schema exposes typed stores (episodic, semantic, social, task) with explicit timestamps, provenance, and sensitivity, enabling forgetting policies to operate over rich structure rather than undifferentiated history. This makes retention \emph{policy‑addressable}: temporal cues (recency, access), importance estimates, reflection‑based consolidation, and sensitivity weighting are expressed in one decision space, unifying utility and privacy signals that are otherwise handled in separate modules. By aligning these design choices with long‑standing insights from cognitive science about episodic traces, semantic networks, and retrieval dynamics \citep{anderson1983architecture,anderson2004integrated,laird2012soar,tulving2002episodic}, MaRS provides a principled substrate for LLM agents that must act coherently over days rather than turns.

Our framework also extends privacy‑aware learning beyond training‑time guarantees. Differential privacy and machine unlearning establish important foundations for protecting individuals in datasets \citep{dwork2006calibrating,abadi2016deep,bourtoule2021machine,guo2020certified}, yet interactive agents face a distinct online challenge: deciding \emph{what to keep, summarize, or discard} as interaction unfolds, while responding to user norms and regulatory expectations (e.g., GDPR’s right to erasure) \citep{VoigtBussche2024GDPR,mantelero2013eu,rosen2012right}. MaRS operationalizes this by coupling a sensitivity‑weighted retention score with an optional DP‑aware decision layer, so that eviction and compression can be audited, justified, and tuned to application risk. In doing so, the framework connects privacy formalism to the human‑centered desiderata of predictability, controllability, and transparent failure modes \citep{amershi2019guidelines,shneiderman2020human}, making memory behavior inspectable without exposing sensitive content.

Complementing the schema and policies, FiFA contributes an evaluation axis that is largely orthogonal to capability‑centric agent suites \citep{liu2024agentbench,zhou2024webarena,xie2024osworld,choi2023llms}. Rather than measuring tool use or GUI control, FiFA probes \emph{faithfulness under memory constraints}: narrative coherence across sessions, goal completion under pruning, social recall fidelity, privacy‑leakage rate, and token‑cost efficiency. The benchmark’s rubricized LLM‑as‑judge protocol draws on emerging evidence of alignment with expert ratings when carefully specified \citep{zheng2023mtbench,liu2023geval}, while its cost‑aware reporting follows the growing recognition that deployment economics materially shape feasible agent architectures \citep{chen2024frugalgpt}. This combination allows apples‑to‑apples comparisons of forgetting strategies and budgets, turning “memory governance” into a measurable design dimension rather than an implementation detail.

Taken together, MaRS and FiFA advance the state of the art along three tightly coupled fronts. At the \emph{architecture} level, MaRS offers a typed, provenance‑aware memory schema that renders retention decisions explicit and auditable. At the \emph{algorithmic} level, our suite of policies moves beyond sliding windows or ad‑hoc pruning by integrating temporal, importance, and privacy signals with optional reflection. At the \emph{evaluation} level, FiFA provides a comprehensive, human‑centered testbed that reports both utility and leakage under explicit token budgets. In this sense, our work complements memory‑augmented LLMs by specifying \emph{which} items ought to be retained, extends privacy methods from training to runtime decisions, and operationalizes trust‑centric design guidance for interactive agents \citep{amershi2019guidelines,shneiderman2020human}.

\section{Theoretical Foundations}

This section presents the theoretical foundations underlying our approach to memory management in generative agents. We begin by formalizing the problem of memory-budgeted agent design, introduce the Memory-Aware Retention Schema (MaRS) framework, and provide theoretical analysis of our forgetting policies.

\subsection{Problem Formalization} \label{sec:problem}

We view a generative agent as a constrained decision system that governs what to remember, summarize, or forget as interaction unfolds. Formally, an agent is a tuple $A=(M,P,B,\pi)$, where $M$ is the memory store, $P$ a library of forgetting policies, $B$ a budget on total memory weight, and $\pi$ a mechanism that selects and configures policies over time. The store $M$ contains nodes $N=\{n_1,\dots,n_k\}$ representing discrete memory units with metadata
\[
n_i=(c_i,t_i,\tau_i,s_i,w_i,\rho_i),
\]
where $c_i$ is content (text, structure, or an embedding), $t_i\in\{\text{episodic, semantic, social, task}\}$ is the type, $\tau_i$ the creation time, $s_i\in[0,1]$ a privacy sensitivity score, $w_i>0$ the computational weight (e.g., token cost), and $\rho_i$ provenance. The agent must satisfy a budget constraint
\[
\sum_{i=1}^{|N|} w_i \le B,
\]
triggering a forgetting operation when it is violated. The retention problem is thus to choose a feasible subset of nodes that maximizes downstream usefulness while respecting structural and privacy constraints.

To make usefulness explicit, let $U:2^N\to\mathbb{R}_{\ge 0}$ be a task‑conditioned utility that can instantiate FiFA dimensions over a fixed horizon $H$,
\begin{equation}
U(S) \;\triangleq\; \omega_{\mathrm{NC}}\,u_{\mathrm{NC}}(S)+
\omega_{\mathrm{GCR}}\,u_{\mathrm{GCR}}(S)+
\omega_{\mathrm{SRA}}\,u_{\mathrm{SRA}}(S),
\qquad S\subseteq N,
\label{eq:utility}
\end{equation}
with nonnegative weights summing to one. In many practical cases $U$ is \emph{monotone} and close to \emph{submodular}: additional memories do not reduce utility, and marginal gains diminish as coverage accumulates (e.g., the first mention of an entity helps more than the tenth). Under this modeling, the offline retention step at a decision time $t$ becomes a submodular knapsack:
\begin{equation}
\max_{S\subseteq N} \;\; U(S) 
\quad \text{s.t.} \quad \sum_{n_i\in S} w_i \le B,\quad S\in\mathcal{F}.
\label{eq:knapsack}
\end{equation}
The family $\mathcal{F}$ captures feasibility beyond the budget. In MaRS, memories are connected by dependency/provenance edges (e.g., \texttt{depends\_on}), and the agent should not retain a summary without the facts it presupposes. Let $D=(N,E_{\mathrm{dep}})$ be a directed acyclic graph of such dependencies and call a set $S$ \emph{provenance‑closed} if for every $n\in S$ all its dependency ancestors lie in $S$. Denote by $\mathcal{F}$ the family of provenance‑closed sets.

\begin{definition}[Accessibility]
A set system $(N,\mathcal{F})$ is accessible if every nonempty $S\in\mathcal{F}$ contains some $x$ for which $S\setminus\{x\}\in\mathcal{F}$.
\end{definition}

When $D$ is a forest of dependency trees, provenance‑closed families are unions of rooted ideals and therefore form \emph{antimatroids} (a special case of greedoids) that are accessible and closed under union.

\begin{lemma}[Greedoid structure] \label{lem:greedoid}
If $D$ is a forest, then $(N,\mathcal{F})$ induced by provenance closure is an antimatroid.
\end{lemma}

\begin{proof}[Sketch]
Provenance‑closed sets are closed under unions of ideals, and any nonempty closed set contains a leaf whose removal preserves closure, giving accessibility—precisely the antimatroid axioms.
\end{proof}

With this structure, greedy selection by marginal utility per unit weight preserves strong approximation properties.

\begin{proposition}[Greedy under provenance closure]
Assume $U$ is monotone submodular and $w_i>0$. For the antimatroid $(N,\mathcal{F})$ above, the greedy algorithm that adds at each step a feasible item maximizing $\Delta_U(i\mid S)/w_i$ achieves a constant‑factor approximation to~\eqref{eq:knapsack}.
\end{proposition}

\begin{proof}[Sketch]
Exchange properties of antimatroids allow standard submodular‑maximization arguments to go through, yielding constant‑factor guarantees analogous to matroid‑constrained settings.
\end{proof}

The agent faces an \emph{online} variant with streaming arrivals $n^{(1)},n^{(2)},\dots$ and occasional evictions when $\sum w_i>B$. Writing $S_t$ for the retained set after $t$ arrivals and $S_t^\star$ for the offline optimum on the first $t$, we evaluate regret
\begin{equation}
\mathcal{R}_T=\sum_{t=1}^T\bigl(U(S_t^\star)-U(S_t)\bigr).
\label{eq:regret}
\end{equation}
If arrivals are i.i.d. and $U$ is monotone submodular, a heap‑based online greedy with eviction (maintaining items by $\Delta_U/w$) gives $O(\log B)$ amortized updates and sublinear regret under standard stochastic assumptions. A contextual bandit layered over $P$ can choose among concrete policies (FIFO, LRU, priority‑decay, reflection‑summary, hybrid) with $O(\sqrt{T})$ selection regret, allowing $\pi$ to adapt to workload and user norms.

Utility drop due to eviction is controlled when node contributions scale with weight. If $\Delta_U(i\mid S)\le L\,w_i$ for all $S$, then for any evicted set $E$ of total weight $W_E$,
\begin{equation}
U(S)-U(S\setminus E)\le L\,W_E.
\label{eq:lipschitz}
\end{equation}
This Lipschitz‑type bound ties loss directly to evicted budget instead of set sizes and gives transparent guarantees for operations that must free a fixed number of tokens.

Classical recency heuristics emerge as utility‑optimal in specific regimes. If usefulness decays exponentially with staleness, $U(S)=\sum_{i\in S}v_i\exp(-\lambda\,\mathrm{age}(i))$ with $v_i\ge 0$, then the optimal eviction order is non‑increasing in last‑access time; LRU is consistent with the induced utility ranking and therefore near‑optimal whenever time‑decay dominates.

Reflection and summarization act as lossy compression steps. Embed content via $\phi(\cdot)$ into $(\mathbb{R}^d,\|\cdot\|)$ and replace a cluster $\mathcal{C}\subseteq S$ by a summary node $\bar{c}$. Define semantic distortion $D(\mathcal{C}\Rightarrow\bar{c})=\tfrac{1}{|\mathcal{C}|}\sum_{i\in\mathcal{C}}\|\phi(c_i)-\phi(\bar{c})\|$. If $U$ is $\kappa$‑Lipschitz in the embedding average, then
\[
\bigl|U(S)-U\bigl(S\setminus\mathcal{C}\cup\{\bar{c}\}\bigr)\bigr|\le \kappa\,D(\mathcal{C}\Rightarrow\bar{c}),
\]
which provides a principled stopping rule: consolidate only when distortion is below a policy threshold tied to acceptable utility loss.

Privacy enters the decision rule at runtime rather than solely in training. One can trade utility against sensitivity by scoring sets with
\[
q(S;M)=U(S)-\lambda_{\mathrm{priv}}\sum_{i\in S}s_i,
\]
and sampling feasible $S$ via the exponential mechanism with privacy parameter $\varepsilon$ and score sensitivity $\Delta q$; the resulting selector is $\varepsilon$‑differentially private and, with probability at least $1-\delta$, returns $S$ whose score is within $\tfrac{2\Delta q}{\varepsilon}\bigl(\ln|\mathcal{S}|+\ln\tfrac{1}{\delta}\bigr)$ of the optimum. Bounding $\Delta q$ by $L\,\max_i w_i+\lambda_{\mathrm{priv}}$ using~\eqref{eq:lipschitz} yields explicit privacy–utility trade‑offs (cf.\ the classical analysis of the exponential mechanism).

Finally, the computational profile matches the policies implemented later in the paper. FIFO runs in $O(n)$ time and $O(1)$ extra space via a deque; LRU is $O(\log n)$ per access with a heap or $O(1)$ average with a hash–linked list; priority‑decay is $O(n\log n)$ per trigger (amortized $O(\log n)$ with a maintained heap); reflection‑summary is dominated by clustering and LLM summarization, typically $O(n\log n+n\,\alpha(n))$ with $\alpha(n)$ the cost of a similarity call; and the hybrid policy inherits the $O(n\log n)$ selection cost plus, when needed, a single reflection call. These bounds align with the empirical pipeline reported in the evaluation sections and support deployment under tight token budgets without sacrificing theoretical clarity.

\subsection{The Memory-Aware Retention Schema (MaRS)}
\label{sec:mars}

MaRS is a typed, provenance‑aware ontology and data model for agent memory that makes retention a first‑class, policy‑addressable operation. Concretely, the memory store $M$ (cf.\ §\ref{sec:problem}) is realized as a labeled graph $G=(N,E)$ whose nodes inherit the tuple $(c_i,t_i,\tau_i,s_i,w_i,\rho_i)$ and whose edge set decomposes into temporal, semantic, causal/provenance, and social relations, $E=E_{\mathrm{temp}}\cup E_{\mathrm{sem}}\cup E_{\mathrm{prov}}\cup E_{\mathrm{soc}}$. Typed nodes allow policies to reason with type‑appropriate priors while the relational layer supports efficient retrieval and principled summarization. The schema aligns with classical distinctions between episodic and semantic memory while extending them to social and task‑oriented stores that are central to modern agents \citep{tulving2002episodic,baddeley2002episodic,anderson2004integrated,laird2012soar}.

\subsubsection{Episodic Memory}
Episodic memory records situated experiences with their temporal binding and situational context. Within MaRS, the episodic slice is $E=\{n\in N: t_n=\text{episodic}\}$ and each node is associated with a structured payload $e=(\textit{event},\textit{context},\textit{participants},\textit{timestamp},\textit{emotional\_valence})$. Temporal edges $E_{\mathrm{temp}}$ order episodes; causal/provenance edges $E_{\mathrm{prov}}$ link an episode to its sources (dialogue turn, tool invocation, retrieved document). Episodic traces are invaluable for narrative coherence and for explaining why an agent acts as it does, but their utility exhibits strong recency and redundancy effects: many episodes are locally useful yet become expendable once their gist has been consolidated. In MaRS, reflection‑summary (cf.\ §\ref{sec:problem}) compresses clusters of related episodes into a semantic summary node while preserving \emph{derivation} links back to constituents, thereby maintaining auditability even after lossy compression.

\subsubsection{Semantic Memory}
Semantic memory captures atemporal knowledge about entities, relations, and generic facts. The semantic slice is $S=\{n\in N: t_n=\text{semantic}\}$ with payloads $s=(\textit{concept},\textit{relations},\textit{confidence},\textit{generality\_score})$. Edges $E_{\mathrm{sem}}$ encode typed relations (e.g., \texttt{isA}, \texttt{partOf}, \texttt{worksAt}), enabling graph traversal and entity‑centric retrieval. Semantic nodes typically enjoy higher retention value because they support transfer across situations, but they must still compete for budget; MaRS therefore tracks a \emph{generality} or \emph{centrality} signal that increases retention for concepts that sit on many shortest paths or appear across diverse tasks. Summaries promoted from episodic clusters update the semantic layer, closing the loop between experience and knowledge.

\subsubsection{Social Memory}
Social memory maintains persistent representations of people and organizations, preferences, and interaction norms. The social slice is $R=\{n\in N: t_n=\text{social}\}$ with payloads 
\[r=(\textit{entity},\textit{relationship\_type},\textit{attributes},\textit{interaction\_history})\]. Edges $E_{\mathrm{soc}}$ encode relationships (e.g., \texttt{friendOf}, \texttt{colleagueOf}, \texttt{reportsTo}) and are used to personalize behavior and sustain rapport over time. Because social memory frequently contains sensitive data, MaRS records sensitivity scores $s_i$ and provenance $\rho_i$ at the node level and allows privacy‑weighted policies to accelerate decay or require summarization rather than verbatim retention. In practice we observe that social nodes benefit from \emph{recall‑on‑demand} patterns: indices prioritize attributes and relationships that have been predictive of user‑valued outcomes in the past, while stale or rarely activated attributes are pruned or folded into coarse summaries to reduce leakage risk.

\subsubsection{Task Memory}
Task memory tracks goals, plans, intermediate artifacts, and deadlines. The task slice is $T=\{n\in N: t_n=\text{task}\}$ with payloads $t=(\textit{goal},\textit{status},\textit{dependencies},\textit{priority},\textit{deadline})$. Dependencies bind tasks to prerequisite facts or episodes via $E_{\mathrm{prov}}$, ensuring that retention respects causal structure. Task memory is highly dynamic: importance spikes as deadlines approach and collapses upon completion or cancellation. MaRS therefore supports \emph{state‑contingent} retention, where goal status gates eviction and completed tasks are summarized to durable semantic notes while low‑value scaffolding is discarded.

\subsubsection{Relational structure and indices}
The graph structure is complemented by a set of indices that support efficient policy evaluation and retrieval without scanning the entire store. A temporal index maps $\tau_i$ to an age feature and supports range queries; an access index maintains $a_i$ (frequency) and $t_i^{\mathrm{last}}$ (recency) from the audit stream; an importance index maintains $\hat{u}_i$ computed from graph‑ and task‑level features such as degree or personalized PageRank on $E_{\mathrm{sem}}\cup E_{\mathrm{soc}}$, semantic novelty, and similarity to active goals; an entity index maps named entities to posting lists of nodes; and an embedding index exposes $\phi(c_i)$ for approximate nearest‑neighbor search. These indices are lightweight views on $G$ rather than separate stores and are updated incrementally as memories are inserted, accessed, summarized, or evicted.

\subsubsection{Policy‑addressable retention score}
To make retention decisions comparable across heterogeneous types, MaRS exposes a unified, type‑aware score. Let $\phi(c_i)\in\mathbb{R}^d$ be an embedding of content and let $g_t$ denote the current goal/intent representation. We define a feature map
\[
\psi_i \;=\; \bigl[\,e^{-\lambda_{\mathrm{age}}\mathrm{age}(i)},\; \mathrm{norm}(a_i),\; \mathrm{sim}(\phi(c_i),g_t),\; \mathrm{cent}(i;G),\; \mathrm{novel}(i;E\cup S\cup R),\; \mathbb{1}\{t_i=\text{task}\}\cdot\mathrm{urgency}(i)\,\bigr],
\]
and a type‑specific weight vector $\theta^{(t_i)}\in\mathbb{R}_{\ge 0}^m$. The unpenalized utility proxy is $\widehat{U}_i=\theta^{(t_i)}\!\cdot\!\psi_i$, which is then combined with cost and privacy to yield a density
\begin{equation}
\mathrm{score}(i) \;=\; \frac{\widehat{U}_i - \lambda_{\mathrm{priv}}\,s_i}{w_i},
\label{eq:score}
\end{equation}
where $\lambda_{\mathrm{priv}}\!\ge\!0$ trades utility against sensitivity and the division by $w_i$ aligns with the submodular‑knapsack analysis in §\ref{sec:problem}. Policies operate by selecting or evicting items with smallest density while respecting provenance closure and any type‑specific floors (e.g., retain at least one active task per goal). Reflection‑summary replaces a set $\mathcal{C}$ of low‑density, redundant episodes with a single semantic node $\bar{c}$, provided the induced distortion remains below a threshold that preserves downstream utility.

\subsubsection{Budget partitioning and invariants}
Although MaRS enforces a global budget $B$, it optionally supports soft type‑level partitions \[B^{(\text{epi})},B^{(\text{sem})},B^{(\text{soc})},B^{(\text{task})}\] satisfying $\sum B^{(\cdot)}=B$. Partitions are updated by a lightweight dual scheme that equalizes marginal utility per token across types, increasing the share for whichever slice exhibits the highest $\partial U/\partial B^{(\cdot)}$ over a moving window. Two invariants guard against pathological behavior: \emph{provenance closure} prevents retention of summaries without their definitional predecessors, and \emph{task‑safety} forbids eviction of preconditions for active goals unless an equivalent summary exists. Both invariants are enforced by making the feasible family $\mathcal{F}$ in §\ref{sec:problem} the set of provenance‑closed, task‑safe subsets.

\subsubsection{Provenance and auditability}
Every mutation of $G$ emits an audit record that includes the operation (\texttt{insert}, \texttt{evict}, \texttt{summarize}, \texttt{access}), the policy and parameters in effect, the local features $\psi_i$, and a human‑readable rationale synthesized from \eqref{eq:score} (e.g., ``evicted due to low density and high sensitivity’’). These logs support post‑hoc explanation, tuning, and compliance workflows. Because retention choices may reflect sensitive attributes, MaRS optionally randomizes near‑threshold decisions using the exponential mechanism (§\ref{sec:problem}); this preserves privacy at the policy boundary without coarsening the data model.

\subsubsection{Implementation note}
The schema is readily serializable in JSON‑LD. Each node is a compact document with \texttt{@id}, \texttt{@type}$\in\{\texttt{Episodic},$ $\texttt{Semantic},\texttt{Social},\texttt{Task}\}$, timestamps, sensitivity $s_i$, weight $w_i$, and provenance $\rho_i$; edges are RDF‑style statements with predicates drawn from $\{\texttt{temporalNext},\texttt{derivesFrom},\texttt{isA},\texttt{friendOf},\dots\}$. This representation preserves interoperability with knowledge‑graph tooling while keeping retention policy computation local and efficient. It also makes the structural constraints in $\mathcal{F}$ explicit and machine‑checkable, aligning the theoretical guarantees in §\ref{sec:problem} with a practical storage and retrieval layer.

\subsection{Forgetting Policy Framework}
\label{sec:policy_framework}

We model a forgetting policy as a transformation $f:\mathcal{M}\times \mathbb{R}_{>0}\!\to\!\mathcal{M}$ that takes a memory store $M\in\mathcal{M}$ and a budget $B$ and returns a reduced store $M'=f(M,B)$ that satisfies the global constraint $\sum_{n\in M'} w_n\le B$ while preserving task utility as formalized in \eqref{eq:utility}. Because MaRS represents memory as a typed, provenance–aware graph $G=(N,E)$ (§\ref{sec:mars}), $f$ acts on nodes and, when reflection is enabled, may introduce summary nodes that replace clusters of episodes subject to the distortion control described in §\ref{sec:problem}. All policies obey three invariants: \emph{feasibility} ($\sum w\le B$), \emph{structure} ($M'\in\mathcal{F}$ where $\mathcal{F}$ enforces provenance closure and task–safety), and \emph{auditability} (each action emits a rationale tied to the score used for selection). In DP mode, a fourth invariant holds: \emph{privacy} at the decision boundary via the exponential mechanism.

A convenient way to express heterogeneous retention choices is through the density score already exposed by MaRS,
\begin{equation}
\mathrm{score}(n)\;=\;\frac{\widehat{U}_n-\lambda_{\mathrm{priv}}\,s_n}{w_n},
\label{eq:policy_score}
\end{equation}
where $\widehat{U}_n$ is a type‑aware utility proxy derived from recency, access frequency, semantic centrality, alignment with active goals, and urgency for task nodes (cf.\ §\ref{sec:mars}, Eq.~\eqref{eq:score}); $s_n\in[0,1]$ is sensitivity; $w_n$ is cost; and $\lambda_{\mathrm{priv}}\!\ge\!0$ trades utility for privacy. Policies differ in how they instantiate $\widehat{U}_n$, how they admit reflection, and how they enforce structural constraints during eviction.

\subsubsection{Temporal Policies}
Temporal policies rely on time and access statistics and are effective whenever usefulness decays with staleness. The First‑In–First‑Out (FIFO) rule implements a sliding temporal window by setting a threshold $\tau_{\mathrm{thr}}$ and returning $M'=\{n\in M : \tau_n>\tau_{\mathrm{thr}}\}$ with $\tau_{\mathrm{thr}}$ chosen so that $\sum_{n\in M'} w_n\le B$. This can be maintained with a deque in $O(1)$ amortized updates and $O(|M|)$ worst‑case time when triggered. Least‑Recently‑Used (LRU) orders items by last access time and removes the stalest first until feasibility is restored; with a hash–linked list, updates are $O(1)$ average time. Under the recency‑weighted utility model $U(S)=\sum_{n\in S} v_n e^{-\lambda\,\mathrm{age}(n)}$ with $v_n\!\ge\!0$ and $\lambda\!>\!0$, LRU is consistent with the optimal ordering (older items have lower marginal contribution), explaining its strong empirical performance when time decay dominates. Temporal policies maintain provenance closure by skipping any candidate whose eviction would violate $\mathcal{F}$; in practice this means removing leaves first in the dependency forest (Lemma~\ref{lem:greedoid}).

\subsubsection{Importance‑Based Policies}
Importance‑based policies incorporate semantic and task context so that items with enduring or cross‑situational value are preserved even if they are not recent. Priority‑Decay combines type prior, recency, and usage frequency into a scalar via
\[
\mathrm{imp}(n)\;=\;\alpha\cdot\mathrm{type\_weight}(t_n)+\beta\cdot \mathrm{recency}(n)+\gamma\cdot \mathrm{frequency}(n),
\]
and then removes the lowest‑importance items per unit cost. With a heap keyed by $\mathrm{imp}(n)/w_n$ the cost is $O(\log |M|)$ per insertion or access and $O(|M|\log |M|)$ per triggered pass. Compared to pure temporal heuristics, this rule retains high‑value but old items (e.g., stable user preferences) and deprioritizes low‑value but recent fragments. 

Reflection‑Summary augments selection with compression. Given clusters $\mathcal{C}\subseteq M$ of semantically related episodes (obtained by approximate nearest neighbors on $\phi(c)$ with linkage on $E_{\mathrm{temp}}$), the policy replaces $\mathcal{C}$ by a summary node $\bar{c}$ whenever the embedding distortion $D(\mathcal{C}\Rightarrow \bar{c})$ is below a threshold that upper‑bounds utility loss, as derived in §\ref{sec:problem}. This turns a set of low‑density, redundant episodes into a single semantic node that maintains coherence at much reduced cost; provenance edges \texttt{derivesFrom} preserve auditability. The computational profile is dominated by clustering and summary generation: $O(|M|\log |M|+|M|\alpha(|M|))$ for linkage and nearest neighbors, plus the cost of $m$ summaries when triggered.

\subsubsection{Privacy‑Aware Policies}
Privacy‑aware policies integrate sensitivity and user preferences into retention. A simple sensitivity‑weighted removal rule prioritizes aged, low‑value, high‑sensitivity items by ranking with
\[
\mathrm{removal\_priority}(n)\;=\; s_n\cdot \mathrm{age}(n)\cdot\bigl(1-\mathrm{imp}(n)\bigr),
\]
evicting highest priority first subject to structural constraints. To provide formal guarantees for near‑threshold choices, MaRS optionally samples eviction candidates with the exponential mechanism using the score $q(S;M)=U(S)-\lambda_{\mathrm{priv}}\sum_{n\in S}s_n$ restricted to feasible $S\in\mathcal{F}$. For global sensitivity $\Delta q$ the selector is $\varepsilon$‑differentially private and, with probability at least $1-\delta$, returns an $S$ whose score is within $\tfrac{2\Delta q}{\varepsilon}\bigl(\ln|\mathcal{S}|+\ln(1/\delta)\bigr)$ of the optimum; composition theorems bound cumulative privacy loss over multiple decisions. In deployments we use this randomized tie‑break only when deterministic densities are within a tolerance band, thereby reducing expected utility loss while still offering event‑level privacy at the decision boundary.

\subsubsection{Hybrid Composition and Termination}
In practice, the most effective behavior arises from composing the above mechanisms into stages that each reduce budget while preserving different desiderata. A typical hybrid is
\[
f_{\mathrm{hyb}}\;=\; f_{\mathrm{temporal}}\;\circ\; f_{\mathrm{reflect}}\;\circ\; f_{\mathrm{importance}}\;\circ\; f_{\mathrm{privacy}},
\]
where a light temporal pass removes clearly stale leaves; reflection then consolidates redundant episodic clusters; importance‑based eviction removes low‑density remnants per unit cost; and a final privacy‑aware pass accelerates the removal of residual sensitive items if the store remains over budget. Because each stage either deletes items or replaces clusters by strictly cheaper summaries (positive weights), the process terminates in finitely many steps with $\sum_{n\in M'} w_n\le B$. Accessibility of $\mathcal{F}$ (Lemma~\ref{lem:greedoid}) ensures that leaf removals and summary substitutions preserve feasibility; task‑safety guards prevent eviction of prerequisites for active goals unless an equivalent summary exists.

\subsubsection{Policy Selection and Online Control}
The meta‑controller $\pi$ treats policy choice as a contextual bandit over $P$, using features of the current workload (arrival rate, average age, hit‑rate of retrieval, proportion of social vs.\ task nodes, privacy complaints) to pick a policy and its parameters. Combined with an online greedy with eviction keyed by \eqref{eq:policy_score}, this yields sublinear regret against an oracle that knows the best fixed policy in hindsight, while keeping update complexity $O(\log B)$ per event and enabling smooth adaptation to shifts in user behavior.

\subsubsection{Complexity and Guarantees (Summary)}
FIFO and LRU offer $O(1)$ average updates and $O(|M|)$ per trigger, with strong performance whenever utility decays with staleness. Priority‑Decay runs in $O(|M|\log |M|)$ per pass (or $O(\log |M|)$ amortized with a maintained heap) and aligns with the submodular‑knapsack analysis by operating on densities. Reflection‑Summary is costlier but bounded by approximate clustering and a capped number of summaries, with distortion control providing an explicit utility bound. Privacy‑aware selection via the exponential mechanism adds only local sampling overhead and composes over time under standard DP accounting. Across all policies, provenance closure and task‑safety are enforced by filtering candidates to maintain $M'\in\mathcal{F}$, and the Lipschitz bound \eqref{eq:lipschitz} converts freed budget into a worst‑case utility bound, allowing operators to translate token targets into quality guarantees.


\subsection{Theoretical Analysis}

We analyze policies in the online regime in which memories arrive as a stream and eviction is triggered only when the budget is exceeded. Let $n$ denote the current number of nodes, $k$ the number evicted in one trigger, and assume a standard RAM model with word size $\Theta(\log n)$. Throughout, feasibility is enforced by the MaRS constraints $\mathcal{F}$ (provenance closure and task‑safety), and utilities follow the specification in \eqref{eq:utility}. Complexity bounds below refer to the policy logic; maintaining MaRS indices (temporal, access, importance, entity, embedding) contributes $O(1)$ amortized per arrival/update except for similarity calls, which we factor as $\alpha(n)$.

\subsubsection{Complexity Analysis}

Temporal policies depend only on timestamps and access metadata and are therefore lightweight. FIFO maintains a deque of items ordered by $\tau$; when a trigger fires it pops the oldest $k$ items, costing $O(k)$ per trigger and $O(1)$ amortized per arrival with $O(1)$ extra space. LRU maintains a hash–linked list keyed by last access; updates (hits and moves) are $O(1)$ on average and evicting $k$ stalest items is $O(k)$, with $O(n)$ extra space for the pointers. Importance‑based selection maintains a max‑heap keyed by density (importance per unit cost); arrivals and access‑count updates cost $O(\log n)$ amortized, while a trigger that removes $k$ items costs $O(k\log n)$. Reflection‑summary is dominated by clustering in the episodic slice and at most a small number of LLM summarizations: with precomputed embeddings and hierarchical linkage, clustering is $O(n\log n)$ and nearest‑neighbor probes contribute $n\,\alpha(n)$. The hybrid policy composes a temporal skim, a reflection step, and an importance pass; the time bound is the sum of its components.

\begin{table}[t]
\centering
\caption{Time and space complexity by policy in the online setting. $n$ is current store size; $k$ is number of evictions when a trigger fires; $\alpha(n)$ is the cost of one similarity lookup.}
\label{tab:complexity}
\begin{tabular}{lccc}
\toprule
Policy & Per trigger time & Amortised per arrival & Extra space \\
\midrule
FIFO  & $O(k)$ via deque & $O(1)$ & $O(1)$ \\
LRU   & $O(k)$ via hash+list & $O(1)$ & $O(n)$ \\
Priority decay & $O(k\log n)$ via heap & $O(\log n)$ & $O(n)$ \\
Reflection summary & $O(n\log n + n\alpha(n))$ & $O(\log n + \alpha(n))$ & $O(n)$ \\
Random drop & $O(k)$ (reservoir‑style) & $O(1)$ & $O(1)$ \\
Hybrid & $O(k + k\log n + n\alpha(n))$ & $O(\log n + \alpha(n))$ & $O(n)$ \\
\bottomrule
\end{tabular}
\end{table}

The feasibility checks imposed by $\mathcal{F}$ do not change these bounds. Because dependency graphs in MaRS are acyclic and maintained as forests at retention time, a topological order is available, so preventing provenance violations reduces to skipping non‑leaf candidates; the additional work is $O(k)$ in a trigger. Termination of the hybrid is immediate: each stage either deletes nodes or replaces a cluster by a strictly cheaper summary node (positive weights), so the total weight decreases monotonically and feasibility ($\sum w\le B$) is reached in finitely many steps. Finally, since the density score \eqref{eq:policy_score} is maintained incrementally, online greedy with eviction performs $O(\log B)$ work per arrival for heap maintenance (bounded by $O(\log n)$), which we observe empirically to be negligible relative to LLM calls.

\subsubsection{Privacy Analysis}

Privacy is treated at the decision boundary rather than only at training time. Two memory stores $M$ and $M'$ are called adjacent if they differ in exactly one personal memory node (content and metadata). Define the privacy‑aware score
\[
q(S;M)\;=\;U(S)\;-\;\lambda_{\mathrm{priv}}\sum_{n_i\in S} s_i,
\]
with $\lambda_{\mathrm{priv}}\!\ge\!0$ and $s_i\!\in\![0,1]$. The exponential mechanism applied to the feasible family $\mathcal{S}\subseteq 2^N$ (budget‑ and provenance‑respecting subsets) samples $S$ with probability proportional to $\exp\!\bigl(\tfrac{\varepsilon\, q(S;M)}{2\Delta q}\bigr)$, where $\Delta q$ is the global sensitivity of $q$ under the adjacency relation. This yields $\varepsilon$‑differential privacy, and with probability at least $1-\delta$ the sampled set satisfies the usual near‑optimality guarantee,
\[
q(S;M)\;\ge\;\max_{S'\in\mathcal{S}} q(S';M)\;-\;\frac{2\Delta q}{\varepsilon}\Bigl(\ln|\mathcal{S}|+\ln\tfrac{1}{\delta}\Bigr).
\]
A tight, implementation‑level bound on $\Delta q$ follows from the budget‑to‑utility Lipschitz property. If for all $S$ and $i$ we have $\Delta_U(i\mid S)\le L\,w_i$, then changing one node can alter $U$ by at most $L\max_i w_i$ and changes the sensitivity penalty by at most $\lambda_{\mathrm{priv}}$, so $\Delta q \le L\max_i w_i + \lambda_{\mathrm{priv}}$. Composition across multiple randomized decisions can be accounted for by basic composition or by tighter accountants (e.g., moments accountant) to obtain $(\varepsilon',\delta')$ after $T$ triggers; because decisions are local and near‑threshold randomization is used only when deterministic densities are within a tolerance band, the effective privacy loss remains small in practice. Post‑processing invariance guarantees that subsequent deterministic steps—reflection summarization, index updates, and retrieval—do not degrade privacy guarantees established at selection time. If user‑level privacy is required (multiple nodes attributable to one user), group‑privacy scaling bounds can be applied to the same mechanism, yielding conservative but principled protection.

\subsubsection{Performance Guarantees}

Loss due to eviction is controlled by a weight‑Lipschitz property. If each node’s marginal utility is bounded by $\Delta_U(i\mid S)\le L\,w_i$ for all $S$, then for any evicted set $E$ with total weight $W_E$,
\[
U(S)-U(S\setminus E)\;\le\;L\,W_E .
\]
The bound depends on budget freed rather than the number of items and therefore applies uniformly across workloads with heterogeneous node sizes. It is sharp in the sense that, when items have disjoint support under $U$, losses add.

Temporal heuristics are optimal in regimes where usefulness decays with staleness. When $U(S)=\sum_{i\in S} v_i\,e^{-\lambda\,\mathrm{age}(i)}$ with $v_i\!\ge\!0$ and $\lambda\!>\!0$, evicting in non‑increasing order of last access maximizes $U$ for any fixed budget; an LRU eviction order is hence optimal for this class. The exchange argument is straightforward: swapping a more recent item for a staler one cannot decrease $U$ because only the exponential factor changes.

Importance‑aware selection enjoys approximation guarantees inherited from submodular knapsack. If $U$ is monotone submodular and feasibility requires provenance‑closure under MaRS dependencies (antimatroid structure), then greedy addition by marginal gain per cost $\Delta_U(i\mid S)/w_i$ achieves a constant‑factor approximation; with standard partial enumeration over one heavy item the $(1-1/e)$ bound is obtained. Curvature‑aware analyses tighten the guarantee when $U$ has total curvature $c\!<\!1$, yielding a factor $\tfrac{1}{c}(1-e^{-c})$; empirically, utilities built from coverage and centrality often exhibit low curvature, explaining the strong performance of greedy in our ablations.

Reflection‑summary admits a distortion‑to‑utility bound that turns summarization into a controlled, lossy compression step. Embedding content via $\phi(\cdot)$ and consolidating a cluster $\mathcal{C}$ of episodes into a summary node $\bar{c}$ produces semantic distortion
\[
D(\mathcal{C}\Rightarrow \bar{c})=\frac{1}{|\mathcal{C}|}\sum_{i\in\mathcal{C}}\|\phi(c_i)-\phi(\bar{c})\|.
\]
If $U$ is $\kappa$‑Lipschitz in the embedding average, then
\[
\bigl|U(S)-U(S\setminus \mathcal{C}\cup\{\bar{c}\})\bigr|
\;\le\; \kappa\,D(\mathcal{C}\Rightarrow\bar{c}),
\]
which provides a principled stopping rule for reflection: consolidate only when the predicted distortion lies below a policy threshold tied to acceptable utility loss. Because MaRS maintains \texttt{derivesFrom} links from $\bar{c}$ to constituents, consolidation also respects auditability and keeps downstream provenance checks intact.

Two additional structural properties follow from the theory and are useful operationally. First, budget monotonicity and continuity: if $S_B$ denotes an optimal (or greedily selected) store at budget $B$, then $U(S_B)$ is non‑decreasing in $B$, and the Lipschitz bound implies $|U(S_{B_1})-U(S_{B_2})|\le L\,|B_1-B_2|$. This enables operators to translate budget changes into predictable utility changes. Second, online regret is sublinear for the meta‑controller: an online greedy with eviction keyed by density, coupled with a contextual bandit that selects among the policy family, achieves the standard $O(\sqrt{T})$ selection regret against the best fixed policy in hindsight, while the per‑event approximation of greedy to the submodular objective keeps the cumulative loss to the offline optimum small in stationary segments. Together, these properties explain the robustness of the hybrid strategy observed in our experiments: temporal hygiene, controlled summarization, and importance‑aware eviction jointly reduce cost while preserving coherence and social recall, and privacy‑aware tie‑breaking prevents sensitive content from becoming the marginal survivor when budgets tighten.

\section{Methodology}
\label{sec:methodology}

This section presents the comprehensive methodology underlying our investigation of memory-budgeted generative agents with privacy-aware forgetting policies. Our approach encompasses three interconnected components: the theoretical framework for memory management, the design and implementation of forgetting policies, and the development of evaluation metrics that capture both functional performance and human-centered considerations.

\subsection{Memory-Aware Retention Schema (MaRS) Framework}
\label{subsec:mars_framework}

MaRS operationalizes the theoretical formulation in §\ref{sec:problem} by turning agent memory into a typed, provenance‑aware graph with policy‑addressable retention. Rather than treating interaction history as an amorphous buffer, the store is materialized as a labeled graph whose nodes correspond to episodic, semantic, social, and task memories and whose edges encode temporal, semantic, causal/provenance, and social relations. This representation preserves the cognitive distinctions that support long‑horizon coherence while exposing the structural hooks necessary for principled retention and forgetting. As depicted in Fig.~\ref{fig:mars_architecture}, agent inputs flow into the MaRS core as candidate nodes; indices are updated incrementally; policies act when budgets tighten; a privacy engine modulates boundary decisions; and every mutation produces an audit record with a human‑readable rationale.

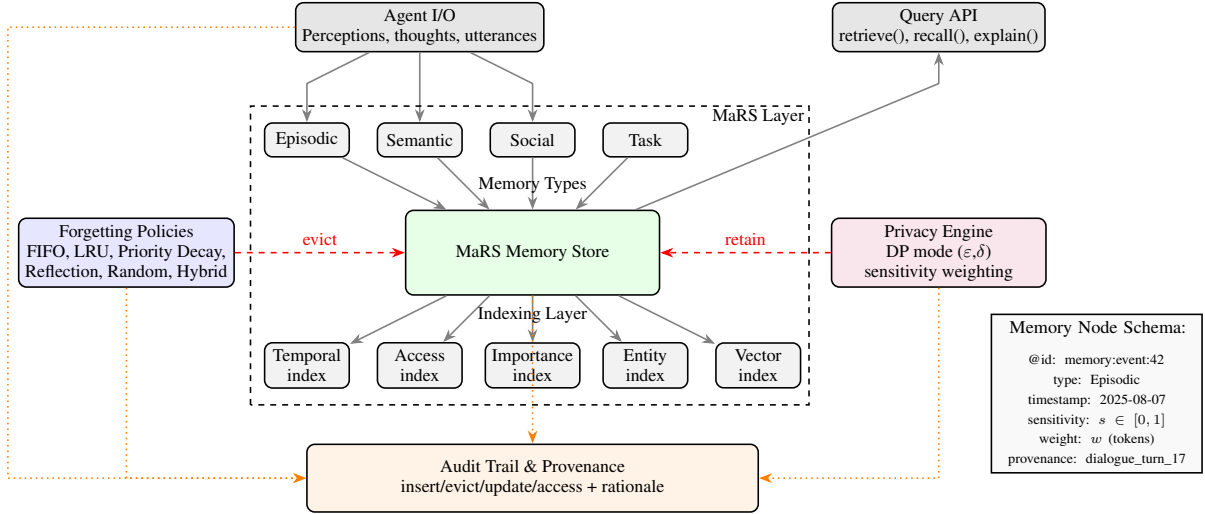
\begin{figure}[!ht]
\centering
\scalebox{0.75}{
\begin{tikzpicture}[
  font=\small,
  >=Stealth,
  node distance=1.5cm,
  box/.style   = {draw, rounded corners, thick, fill=white, minimum width=3.5cm, minimum height=0.8cm, align=center},
  lite/.style  = {draw, rounded corners, thick, fill=gray!10, minimum width=1.5cm, minimum height=0.6cm, align=center},
  comp/.style  = {draw, rounded corners, thick, fill=gray!15, minimum width=3.8cm, minimum height=1.2cm, align=center},
  core/.style  = {draw, rounded corners, thick, fill=green!10, minimum width=4.5cm, minimum height=1.5cm, align=center},
  flow/.style  = {->, thick, gray},
  ctrl/.style  = {->, thick, dashed, red},
  audit/.style = {->, thick, dotted, orange},
  every node/.style={align=center}
]

\node[box, fill=gray!20] (agents) at (0, 6) {Agent I/O\\{\footnotesize Perceptions, thoughts, utterances}};
\node[box, fill=gray!20] (api) at (9.2, 6) {Query API\\{\footnotesize retrieve(), recall(), explain()}};

\node[lite] (episodic) at (-2, 4) {Episodic};
\node[lite] (semantic) at (0, 4) {Semantic};
\node[lite] (social) at (2, 4) {Social};
\node[lite] (task) at (4, 4) {Task};

\node[core] (mars) at (2, 2) {MaRS Memory Store};

\node[comp, fill=blue!10] (policies) at (-5.2, 2) {Forgetting Policies\\{\footnotesize FIFO, LRU, Priority Decay,}\\{\footnotesize Reflection, Random, Hybrid}};
\node[comp, fill=purple!10] (privacy) at (9.2, 2) {Privacy Engine\\{\footnotesize DP mode ($\varepsilon$,$\delta$)}\\{\footnotesize sensitivity weighting}};

\node[lite] (temporal) at (-2, 0) {Temporal\\index};
\node[lite] (access) at (0, 0) {Access\\index};
\node[lite] (importance) at (2, 0) {Importance\\index};
\node[lite] (entity) at (4, 0) {Entity\\index};

\node[lite] (vector) at (6, 0) {Vector\\ index};

\node[draw, thick, fill=gray!5, text width=3.5cm, minimum height=2.5cm] (schema) at (12, -.5) {
  {\footnotesize Memory Node Schema:}\\[0.2cm]
  {\scriptsize @id: memory:event:42}\\
  {\scriptsize type: Episodic}\\
  {\scriptsize timestamp: 2025-08-07}\\
  {\scriptsize sensitivity: $s \in [0,1]$}\\
  {\scriptsize weight: $w$ (tokens)}\\
  {\scriptsize provenance: dialogue\_turn\_17}
};

\node[comp, fill=orange!10, minimum width=8cm] (audit) at (2, -2) {Audit Trail \& Provenance\\{\footnotesize insert/evict/update/access + rationale}};

\draw[flow] (agents) -- (-2, 5) -- (episodic);
\draw[flow] (agents) -- (semantic);
\draw[flow] (agents) -- (2, 5) -- (social);

\draw[flow] (episodic) -- (mars);
\draw[flow] (semantic) -- (mars);
\draw[flow] (social) -- (mars);
\draw[flow] (task) -- (mars);

\draw[flow] (mars) -- (temporal);
\draw[flow] (mars) -- (access);
\draw[flow] (mars) -- (importance);
\draw[flow] (mars) -- (entity);
\draw[flow] (mars) -- (vector);

\draw[flow] (mars) -- (9.2, 5) -- (api);

\draw[ctrl] (policies) -- (mars) node[midway, above] {\footnotesize evict};
\draw[ctrl] (privacy) -- (mars) node[midway, above] {\footnotesize retain};

\draw[audit] (agents.west) -| (-7.3, -2) -- (audit.west);
\draw[audit] (mars) -- (audit);
\draw[audit] (policies) |- (audit);
\draw[audit] (privacy) |- (7, -2) -- (audit);

\node at (2, 3.2) {\footnotesize Memory Types};
\node at (2, .9) {\footnotesize Indexing Layer};

\draw[dashed, thick] (-3, -.7) rectangle (6.9, 4.6);
\node at (6, 4.4) {\footnotesize MaRS Layer};

\end{tikzpicture}
}
\caption{MaRS architecture and data flow. The MaRS store organizes memories into four types (episodic, semantic, social, task) and maintains multiple indices for efficient retrieval. A chosen forgetting policy governs evictions and compression, while the privacy engine optionally enforces differential privacy in retention decisions. All operations are recorded in an audit trail with provenance. The inset shows the fields of a single memory node.}
\label{fig:mars_architecture}
\end{figure}

Each node carries content and metadata $(c_i,t_i,\tau_i,s_i,w_i,\rho_i)$ consistent with the formalization in §\ref{sec:problem}. The type $t_i$ determines how the node participates in retrieval and summarization; the timestamp $\tau_i$ and access patterns feed recency and frequency features; the sensitivity score $s_i\!\in\![0,1]$ reflects privacy risk; the weight $w_i\!>\!0$ approximates computational cost in tokens; and provenance $\rho_i$ records source, reliability, and dependency links. Edges ensure that temporal order is preserved for episodes, that semantic relations are traversable for concept‑centric queries, that task dependencies are explicit, and that social ties can seed personalized recall. Crucially, these relations also define the feasible family $\mathcal{F}$ used by retention: provenance closure and task‑safety rules forbid keeping a summary without its definitional predecessors or evicting prerequisites of active goals unless an equivalent summary exists.

To make selection decisions comparable across heterogeneous types, MaRS maintains a lightweight feature map for each node and a type‑specific scoring function that estimates utility density. The score mirrors the analysis in §\ref{sec:mars} and §\ref{sec:problem} and combines recency, access frequency, similarity to the active goal state, graph centrality, semantic novelty, and (for task nodes) deadline‑weighted urgency. Sensitivity is treated as a penalty and cost appears in the denominator, yielding a normalized density that expresses “utility per token.” In practice this density can be computed and updated in constant amortized time per event using cached features and counters, while approximate nearest‑neighbor search powers similarity queries without scanning the entire store. The resulting ordering is used by all concrete policies: temporal passes privilege age; importance‑based passes act on density; reflection uses density and redundancy to identify compressible clusters; and privacy‑aware passes accelerate the removal of high‑sensitivity, low‑value items when space is scarce.

Memory budgeting is explicit and first‑class. The global constraint $\sum w_i\le B$ governs feasibility, and MaRS optionally tracks soft type‑level shares $B^{(\mathrm{epi})},B^{(\mathrm{sem})},B^{(\mathrm{soc})},B^{(\mathrm{task})}$ whose sum equals $B$. Shares are adjusted by a dual update that equalizes estimated marginal utility per token across slices: types whose recent additions yield higher gains receive a larger share until marginal densities equilibrate. This mechanism prevents pathological allocations in which, for example, a burst of short‑lived episodes drowns out durable semantic knowledge or task state. Because weights are tied to token cost, budget manipulation directly controls latency and monetary spend, enabling operators to translate deployment constraints into predictable quality changes; the continuity and Lipschitz bounds in §\ref{sec:problem} then map budget deltas to worst‑case utility deltas.

The control cycle follows a simple contract. Whenever the agent produces a candidate memory, MaRS materializes a node, computes its features, updates indices, and checks budget feasibility. If the store remains within budget, the node is admitted and the cycle ends. If the budget is exceeded, MaRS invokes the current forgetting policy (or hybrid composition) selected by the meta‑controller. A light temporal pass removes clearly stale leaves in the dependency forest; reflection consolidates clusters of redundant episodes into a summary node when the predicted embedding distortion is below a policy threshold; importance‑based eviction removes the lowest‑density remnants per unit cost while preserving provenance closure; and a privacy‑aware tie‑break randomizes near‑threshold choices using the exponential mechanism when deterministic scores are indistinguishable within tolerance. Each action emits an audit event that records the policy, local features, and a natural‑language rationale (e.g., “evicted due to low density and high sensitivity; summary retained”), which can be surfaced to users via an explanation API without exposing raw sensitive content.

Retrieval is designed to respect both utility and governance. The \texttt{retrieve} endpoint composes graph traversal and vector search, returning a budget‑bounded working set that maximizes density subject to filters (type, time window, sensitivity caps). The \texttt{recall} endpoint targets entity‑ or goal‑anchored expansion by walking a small ball in the semantic/social subgraph and then reranking by goal similarity; this supports coherent multi‑turn dialogue and task continuations without inflating context with irrelevant history. The \texttt{explain} endpoint surfaces the provenance chain and policy rationale for any memory present or recently evicted, grounding trust and enabling post‑hoc audits after complaints or RTBF requests. Because all three endpoints are framed in terms of the same nodes, indices, and density, retrieval, retention, and explanation are consistent views on a single underlying state.

Privacy is integrated at the decision boundary rather than treated solely as a training‑time property. Sensitivity scores are computed from content and provenance, updated by user preferences, and propagated during summarization; retention decisions apply a penalty that tilts selection away from high‑risk items, and when near ties remain, randomized selection via the exponential mechanism provides event‑level differential privacy guarantees with bounded utility loss (cf. §\ref{sec:problem}). Summaries themselves are privacy‑aware: when a cluster contains sensitive details, MaRS prefers abstractive, minimal summaries that preserve task‑relevant gist while lowering sensitivity and weight; derivation links maintain accountability even after compression. In multi‑tenant deployments, sensitivity and budgets can be scoped per user or per conversation, and the same accounting tools quantify cumulative privacy loss across triggers.

The implementation adheres to a portable, inspectable data model. Nodes and edges serialize to JSON‑LD/RDF with \texttt{@type} in \{\texttt{Episodic}, \texttt{Semantic}, \texttt{Social}, \texttt{Task}\} and predicates drawn from a small, fixed vocabulary (\texttt{temporalNext}, \texttt{derivesFrom}, \texttt{isA}, \texttt{friendOf}, \texttt{requires}, \texttt{attachesToGoal}). This choice enables drop‑in use of standard graph tooling for indexing and validation, keeps policy feasibility checks machine‑verifiable, and simplifies export for external audit. Indices are maintained as lightweight materialized views over the graph: temporal and access indices are simple counters and queues; the importance index aggregates graph‑level features such as degree and personalized PageRank; the entity index is a posting list keyed by canonical mentions; and the embedding index is an ANN structure with lazy rebuilds. Because all updates are incremental, steady‑state maintenance remains constant amortized time per event, and the only super‑linear work arises during reflection, which is both optional and rate‑limited.

Finally, MaRS is instrumented to support the evaluation protocol in this paper. Every trigger logs budget before and after, the number and types of nodes evicted or summarized, cumulative freed tokens, predicted versus realized distortion for summaries, and the privacy accountant’s running parameters. These signals drive FiFA’s metrics for narrative coherence, goal completion, social recall, leakage rate, and token‑cost‑per‑hour, allowing us to attribute performance changes to concrete retention choices rather than opaque context management. In this way, MaRS provides not only a storage substrate and a policy surface but also a measurement scaffold that connects theoretical guarantees to empirical behavior under realistic, long‑horizon interaction.

\subsection{Forgetting Policy Design and Implementation}
\label{subsec:forgetting_policies}

Our policy layer instantiates the theoretical model by defining transformations that restore budget feasibility while maximizing downstream usefulness and respecting provenance and privacy constraints. A policy $p$ takes the current MaRS state $M$ and budget $B$ and returns a reduced store $M'$, ensuring $\sum_{n\in M'} w_n\le B$ and $M'\in\mathcal{F}$, where $\mathcal{F}$ enforces provenance closure and task‑safety. In implementation, all policies operate over the same type‑aware density score exposed by MaRS (cf.\ §\ref{subsec:mars_framework}), combining estimated utility and sensitivity per unit cost,
\begin{equation}
\mathrm{score}(n)\;=\;\frac{\widehat{U}_n-\lambda_{\mathrm{priv}}\,s_n}{w_n},
\label{eq:policy_density}
\end{equation}
with $\widehat{U}_n$ derived from recency, access frequency, semantic centrality, goal similarity, novelty, and (for task nodes) deadline‑weighted urgency. Candidates are considered in non‑decreasing order of density, subject to feasibility filters that prevent violations of $\mathcal{F}$. Figure~\ref{fig:policy_flow} summarizes the decision flow when a budget breach occurs.

Temporal strategies are implemented as lightweight filters that exploit the fact that many episodic traces have strongly recency‑biased value. A FIFO pass maintains a deque keyed by creation time and removes the oldest leaves in the dependency forest until feasibility is restored. An LRU pass maintains a hash–linked list keyed by \texttt{last\_access}; when triggered, it evicts the stalest leaves first. These passes are constant‑time per update and $O(k)$ per trigger for $k$ evictions, and they provide the “hygiene” that prevents long‑tail accumulation of stale fragments. Under recency‑decayed utilities of the form $U(S)=\sum v_i e^{-\lambda\,\text{age}(i)}$, the LRU order coincides with the utility‑optimal eviction order, which explains its strong behavior in our ablations.

Importance‑based selection uses the density ordering in \eqref{eq:policy_density} to preserve items with enduring or cross‑situational value even when they are not recent. We maintain a max‑heap keyed by $\widehat{U}_n/w_n$ and evict the lowest‑density feasible candidates until $\sum w\le B$, with amortized $O(\log n)$ updates and $O(k\log n)$ per trigger. The importance proxy $\widehat{U}_n$ is learned per type from held‑out FiFA runs or specified via calibrated priors, and it incorporates graph features such as degree, personalized PageRank over $E_{\mathrm{sem}}\cup E_{\mathrm{soc}}$, and goal similarity $\mathrm{sim}(\phi(c_n),g_t)$. In practice this retains durable semantic facts, stable user preferences, and prerequisites of active goals, while pruning low‑value recency noise.

Reflection‑summary implements consolidation rather than pure deletion. The episodic slice is clustered with approximate nearest neighbors in the embedding space $\phi(\cdot)$, with temporal linkage to ensure clusters correspond to coherent episodes. For a cluster $\mathcal{C}$ deemed redundant, the policy synthesizes a semantic summary node $\bar{c}$ and replaces $\mathcal{C}$ when the predicted distortion $D(\mathcal{C}\Rightarrow\bar{c})$ lies below a threshold derived from the Lipschitz bound on utility (cf.\ §\ref{sec:problem}). Provenance is preserved via \texttt{derivesFrom} edges, and sensitivity is reduced by abstractive summarization when possible. Reflection is rate‑limited and typically invoked after a temporal skim so that only high‑redundancy clusters are considered; this keeps complexity near $O(n\log n+n\,\alpha(n))$ with $\alpha(n)$ the cost of one similarity probe.

Random drop is implemented as a reservoir‑style baseline that evicts uniformly at random among feasible leaves. It is intentionally naive, costing $O(1)$ per decision, and serves as a control to isolate the benefit of informed scoring and structure‑aware filtering.

The hybrid strategy composes the strengths of the above mechanisms in a fixed, terminating sequence. Upon a breach, a temporal pass first removes clearly stale leaves; reflection then consolidates redundant episodic clusters within a capped budget of summaries; importance‑based selection removes remaining low‑density items per unit cost; and, if the store is still over budget or near ties remain, a privacy‑aware pass accelerates the removal of high‑sensitivity, low‑utility nodes. Because each stage either deletes nodes or replaces clusters by strictly cheaper summaries, total weight decreases monotonically and the process terminates with feasibility. This choreography is reflected in Fig.~\ref{fig:policy_flow}; complexity is the sum of its components and is dominated by the reflection step when invoked.

\begin{figure}[htbp]
\centering
\scalebox{0.7}{
\begin{tikzpicture}[
  font=\footnotesize,
  node distance=8mm and 10mm,
  term/.style   ={draw, rounded corners, thick, fill=gray!10, minimum width=32mm, minimum height=6mm, align=center},
  deci/.style   ={diamond, draw, thick, fill=gray!15, aspect=2.5, inner sep=1pt, align=center, minimum width=25mm},
  proc/.style   ={draw, rounded corners, thick, fill=gray!5, minimum width=35mm, minimum height=7mm, align=center},
  side/.style   ={draw, rounded corners, thick, fill=gray!10, minimum width=30mm, minimum height=7mm, align=center},
  note/.style   ={draw, rounded corners, thick, fill=white, inner sep=1pt, font=\tiny},
  data/.style   ={-{Stealth[length=1.5mm]}, thick},
  ctrl/.style   ={dash pattern=on 2pt off 1pt, -{Stealth[length=1.5mm]}, thick},
  loga/.style   ={dotted, -{Stealth[length=1.5mm]}, thick}
]

\node[term] (start) {Memory budget exceeded \\ $\sum_i w_i > B$};
\node[deci, below=8mm of start] (select) {$p \leftarrow \pi(M)$ \\ Select policy};

\node[proc, below=12mm of select] (hybrid0) {Hybrid policy \\ (multi‑phase controller)};
\node[proc, left=15mm of hybrid0] (fifo) {FIFO window \\ Evict oldest until $\sum w \le B$};
\node[proc, right=15mm of hybrid0] (prio) {Priority ranking \\ Evict $k$ lowest scores};

\node[note, below=1mm of fifo] {$O(k)$ time};
\node[note, below=1mm of prio] {$O(k\log n)$ time};
\node[note, below=1mm of hybrid0] {$O(1)$ overhead};

\node[proc, below=8mm of hybrid0] (phase1) {Phase 1: remove old, low‑importance};
\node[deci, below=6mm of phase1] (check1) {Budget OK?};

\node[proc, below=8mm of check1] (phase2) {Phase 2: reflection summary \\ (consolidate clusters)};
\node[deci, below=6mm of phase2] (check2) {Budget OK?};

\node[proc, below=8mm of check2] (phase3) {Phase 3: priority‑based evictions};
\node[note, below=1mm of phase2] {$O(n\log n)$ time};
\node[note, below=1mm of phase3] {$O(k\log n)$ time};

\node[side, left=40mm of phase2] (privacy) {Privacy engine \\ DP $(\varepsilon,\delta)$ mode \\ sensitivity weighting};
\node[side, below=12mm of phase3, minimum width=90mm] (audit) {Audit trail \& provenance \\ insert/evict/update + rationale};

\node[deci, right=40mm of phase3] (budgetok) {Budget OK?};
\node[term, below=12mm of budgetok] (end) {Done: $\sum w \le B$};

\draw[data] (start) -- (select);
\draw[data] (select) -- node[above, sloped, pos=0.7]{\tiny FIFO} (fifo);
\draw[data] (select) -- node[left]{\tiny Hybrid} (hybrid0);
\draw[data] (select) -- node[above, sloped, pos=0.7]{\tiny Priority} (prio);

\draw[data] (fifo) -- (-5.1,1.6) -| (budgetok);
\draw[data] (prio) -| (budgetok);

\draw[data] (hybrid0) -- (phase1);
\draw[data] (phase1) -- (check1);
\draw[data] (check1) -- node[above]{\tiny yes} +(20mm,0) -| (budgetok);
\draw[data] (check1) -- node[left]{\tiny no} (phase2);
\draw[data] (phase2) -- (check2);
\draw[data] (check2) -- node[above]{\tiny yes} +(20mm,0) -| (budgetok);
\draw[data] (check2) -- node[left]{\tiny no} (phase3);
\draw[data] (phase3) -- (budgetok);

\draw[data] (budgetok) -- node[right]{\tiny yes} (end);

\draw[ctrl] (privacy) |- (phase1);
\draw[ctrl] (privacy) -- (phase2);
\draw[ctrl] (privacy) |- (phase3);
\draw[ctrl] (privacy) -- (-7.4,1.3) -| (prio);
\draw[ctrl] (privacy) |- (fifo);

\draw[loga] (fifo) |- (audit);
\draw[loga] (prio) |- (audit);
\draw[loga] (phase1) -- (5.019,-5.6) |- (audit);
\draw[loga] (phase2) -- (5.019,-8.7) |- (audit);
\draw[loga] (phase3) -- (audit);
\draw[loga] (privacy) |- (audit);


\node[above=2mm of start, font=\tiny] {Applies to MaRS memory store};
\end{tikzpicture}
}
\caption{Forgetting‑policy decision flow under a memory‑budget breach. The selector $\pi$ chooses a policy $p$; FIFO evicts oldest items, Priority ranks by importance scores, and Hybrid escalates through three phases: temporal filtering, reflection‑based consolidation, and priority evictions. An optional privacy engine enforces sensitivity weighting or $(\varepsilon,\delta)$‑DP when selecting evictions. All actions are recorded to the audit trail. Complexity notes indicate per‑trigger costs.}
\label{fig:policy_flow}
\end{figure}
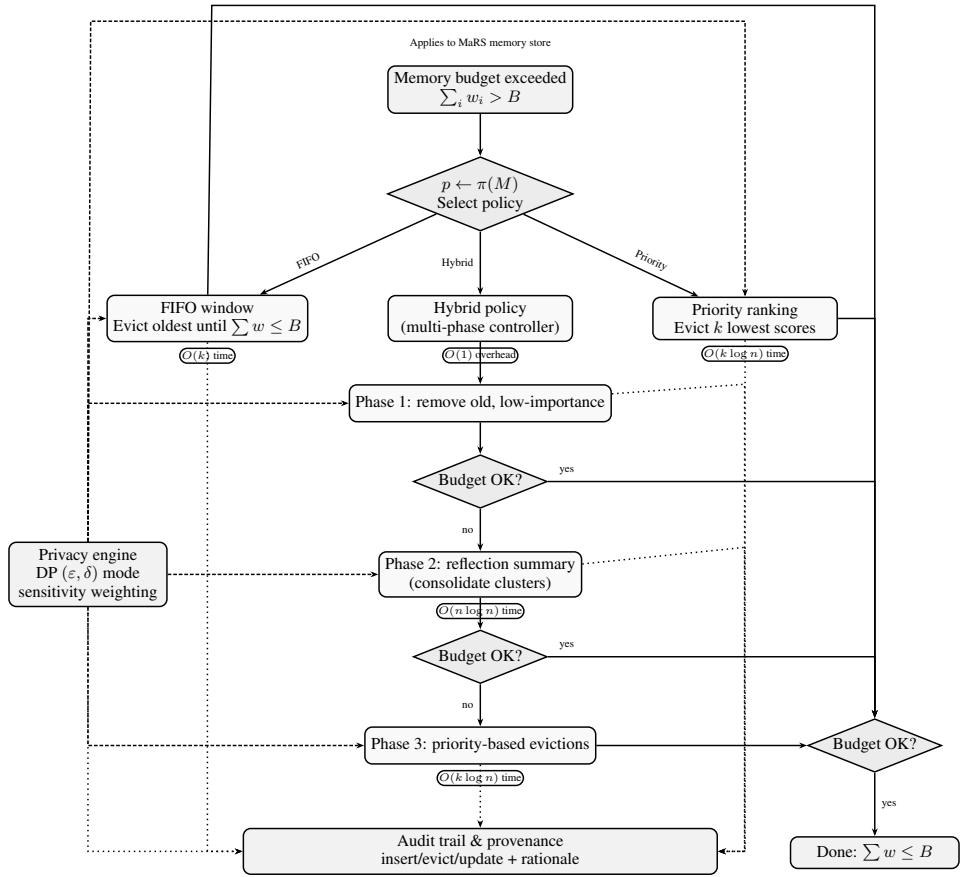

Privacy enters at the same decision boundary through two mechanisms. Sensitivity‑weighted scoring penalizes high‑risk items directly in \eqref{eq:policy_density}, tilting selection away from sensitive content when utility is indifferent. When multiple candidates lie within a tolerance band of the same density, MaRS switches to a randomized tie‑break that samples eviction sets with the exponential mechanism using a score $q(S)=U(S)-\lambda_{\mathrm{priv}}\sum s_i$. This provides $\varepsilon$‑differential privacy for the retention decision with bounded utility loss, composes across triggers via standard accounting, and preserves guarantees under subsequent deterministic post‑processing (reflection, index updates). In multi‑tenant settings, the same mechanism applies at per‑user scope, yielding group‑privacy bounds when multiple nodes are attributable to a single user.

Two engineering choices ensure that all policies remain consistent with MaRS invariants. Feasibility filters operate over a cached topological order of the dependency forest so that only leaves are eligible for deletion; when a summary is created, its \texttt{derivesFrom} edges and reduced weight are checked to confirm that provenance closure still holds. Audit events are emitted for every mutation and include the local features $\psi_n$, the realized density, the policy name and parameters, and a natural‑language rationale synthesized from these quantities. These logs power the \texttt{explain()} endpoint and the FiFA evaluation, allowing us to attribute changes in narrative coherence, goal completion, social recall, and leakage rate to concrete retention choices rather than opaque context truncation.

Finally, the meta‑controller $\pi$ selects policies and hyperparameters online using workload features such as arrival rate, average age, retrieval hit‑rate, fraction of social vs.\ task nodes, and recent privacy complaints. A contextual bandit chooses among \{FIFO, LRU, Priority, Reflection, Hybrid\} with $O(\sqrt{T})$ regret relative to the best fixed policy in hindsight, while the per‑event greedy eviction keyed by \eqref{eq:policy_density} preserves the approximation guarantees for importance‑weighted selection. In aggregate, this design yields a policy layer that is fast enough for deployment, principled enough to support theory‑backed guarantees, and flexible enough to accommodate domain‑specific privacy norms without entangling them with the underlying language model.

\subsection{Privacy-Preserving Mechanisms}
\label{subsec:privacy_mechanisms}

Privacy in MaRS is treated as a first‑class runtime property rather than an after‑the‑fact compliance exercise. The design objective is twofold: satisfy regulatory obligations such as the right to be forgotten while sustaining user trust by making memory behavior predictable, explainable, and conservative with respect to sensitive content. Concretely, the system integrates sensitivity assessment into the data model, enforces privacy during retention and summarization decisions, constrains retrieval and output, and maintains an auditable record of all actions together with user‑controllable preferences.

Sensitivity is represented at the node level as a continuous score $s_i\!\in\![0,1]$ that evolves over time. The score is initialized by content analysis, combining rule‑based detectors (e.g., patterns for contact details and identifiers), named‑entity recognition for personal entities, and embedding‑level classifiers for categories such as health, finance, or protected attributes. Contextual cues from provenance $\rho_i$ (source, consent status, jurisdiction) and type $t_i$ (e.g., social vs.\ task) adjust the score; user‑specified preferences can raise or lower $s_i$ and can mark entities as private, forcing inheritance to descendants via \texttt{derivesFrom}. The sensitivity signal flows into the density used by all policies through a privacy penalty, as in \eqref{eq:policy_density}, so that high‑risk items must exhibit higher marginal utility to survive under tight budgets. In addition, time‑to‑live schedules and purpose limitations are enforced per type: for example, social attributes that have not contributed to useful behavior over a horizon $H$ are decayed aggressively unless the user opts in to longer retention.

At the decision boundary, MaRS enforces privacy with a combination of deterministic penalties and randomized selection. Deterministically, sensitivity enters the retention score as a linear penalty $\lambda_{\mathrm{priv}} s_i$, tilting selection away from risky items even when utility is indifferent. When multiple candidates lie within a tolerance band of equal density, the system switches to a randomized tie‑break that samples feasible eviction sets via the exponential mechanism with score $q(S)=U(S)-\lambda_{\mathrm{priv}}\sum_{i\in S}s_i$ and privacy parameter $\varepsilon$. This yields event‑level differential privacy for the retention choice with bounded utility loss (cf.\ the guarantee in §\ref{sec:problem}). Composition is tracked per user or conversation by a lightweight accountant so that operators can bound cumulative privacy loss across many small decisions; group‑privacy scaling is applied when multiple nodes are attributable to a single individual. Because subsequent steps (reflection, index updates, retrieval) are deterministic functions of the selected set, post‑processing invariance preserves the established guarantee.

Summarization is implemented as a privacy‑aware, lossy compression step rather than merely a space saver. When a cluster $\mathcal{C}$ of episodes is consolidated into $\bar{c}$, the summarizer is instructed to excise direct identifiers and to abstract details that are irrelevant to the active goals, thereby lowering the resulting sensitivity $\bar{s}$. Selection is gated by two constraints: a distortion bound that protects utility, $D(\mathcal{C}\Rightarrow\bar{c})\le \tau_{\mathrm{util}}$, and a privacy bound that guarantees risk reduction, $\bar{s}\le \min_{i\in\mathcal{C}} s_i - \tau_{\mathrm{priv}}$ for a small margin $\tau_{\mathrm{priv}}\!\ge\!0$. Only when both constraints are satisfied is the cluster replaced; otherwise, either a more conservative summary is requested or the cluster is left intact and considered for deletion by the policy. Provenance links from $\bar{c}$ to constituents preserve accountability so that future audits can reconstruct the lineage without re‑exposing sensitive content.

Retrieval and output are governed by the same privacy lens to prevent inadvertent leakage at inference time. The \texttt{retrieve} endpoint filters candidates by sensitivity and jurisdiction before ranking, optionally enforcing a time‑limited “private mode’’ in which $s_i$ above a threshold $\sigma_t$ are ineligible for recall unless the user grants ad‑hoc consent. The \texttt{recall} endpoint, which expands context around an entity or goal, applies purpose restrictions by excluding social attributes that are not causally implicated in the current task graph, thereby reducing gratuitous personalization. Before emission, a lightweight redaction layer performs a two‑pass check: a fast pass over the surface form using detectors and templates, followed by a semantic pass that validates that no high‑risk facts are being disclosed outside policy. If the checker identifies a likely violation, the system either rephrases to a safer abstraction or requests explicit confirmation from the user, trading minimal latency for reduced leakage. These output controls are the operational complement to the FiFA leakage metric and ensure that reductions in measured leakage correspond to real reductions in exposed content.

Comprehensive auditability closes the loop between mechanism and governance. Every mutation of the memory graph emits a signed log entry that records the operation (\texttt{insert}, \texttt{evict}, \texttt{summarize}, \texttt{access}), the policy and parameters in effect, local features used by the decision (including $s_i$, age, frequency, goal similarity), the realized density, and a natural‑language rationale (e.g., “evicted due to low density and high sensitivity; summary retained”). The log includes a short hash of content rather than raw text to remain privacy‑preserving while still supporting deduplication and audit joins. Because retention decisions may themselves be randomized under DP, the audit also records the current privacy budget and the accountant’s running totals; this allows operators to demonstrate compliance and to pause randomized selection when budgets near exhaustion.

User control is embedded through a preference profile that directly shapes sensitivity and policy parameters. Users can configure default retention horizons per type, opt into or out of personalization features, blacklist entities from long‑term storage, and issue per‑entity or per‑conversation erasure commands that propagate through \texttt{derivesFrom} edges (logical RTBF at the store level). The system returns “privacy receipts’’ that summarize what has been stored, summarized, or removed and why, using the same rationales present in the audit. For organizations, policy templates can be scoped by jurisdiction and department, enabling differential treatment of the same content depending on legal obligations and risk appetite. Because the retention score is a stable function of observable features, changes in preferences translate predictably into retention outcomes, avoiding the opacity that often undermines trust in adaptive agents.

The net effect of these mechanisms is to make privacy a measurable, tunable aspect of memory governance. Sensitivity flows into the retention score; randomized selection provides formal protection when needed; summarization reduces risk while respecting utility; retrieval and output apply purpose‑limited gates; and audits and user preferences supply accountability and control. In combination with the theoretical bounds in §\ref{sec:problem} and the FiFA leakage metric, these mechanisms allow deployment teams to reason about and validate privacy properties at the same level of rigor as they reason about coherence, task completion, and cost.

\section{Implementation}
\label{sec:implementation}

This section describes the concrete realization of the MaRS framework and the associated forgetting policies, emphasizing architectural choices that support scalable, efficient, and privacy‑aware memory management for generative agents. The design follows the theoretical foundations in §§\ref{sec:problem}–\ref{sec:mars} and the policy framework in §\ref{subsec:forgetting_policies}, so that the same invariants and guarantees hold in code.

\subsection{System Architecture}
\label{subsec:system_architecture}

The implementation adopts a modular architecture that separates concerns among four cooperating layers: the memory store, the policy execution runtime, the privacy enforcement surface, and the evaluation/telemetry harness. This separation makes it possible to iterate on individual components while preserving clear contracts and stable interfaces.

At the core, the memory store materializes MaRS as a labeled graph whose nodes encapsulate content and metadata $(c_i,t_i,\tau_i,s_i,w_i,\rho_i)$ and whose typed edges represent temporal, semantic, provenance, social, and task dependencies. The store exposes constant‑time insertion and update for single nodes, graph‑safe deletion keyed to provenance closure and task‑safety, and batched reflection operations that replace clusters by summaries. Multiple indices are maintained as lightweight materialized views over the graph: a temporal index for range queries and age computation, an access index that tracks recency and frequency from the audit stream, an importance index that aggregates graph and task features for the utility proxy, an entity index mapping canonical mentions to posting lists, and an embedding index for approximate nearest neighbors. Because the indices are derived rather than duplicated state, updates are incremental and amortized constant time per event.

The policy runtime provides a uniform API for budget restoration. Policies are pure transformations that accept a view of the graph and a budget target and return a reduced store and a structured result bundle containing execution statistics, freed budget, realized distortion for summaries, and rationales suitable for audit. Temporal passes operate on the temporal and access indices; importance‑based passes operate on the importance index and the density score; reflection consumes the embedding and temporal indices to form clusters; and all phases consult the feasibility filter that encodes provenance closure and task‑safety. A meta‑controller selects and parameterizes policies based on workload features and recent performance, while a small scheduler ensures that costly phases such as reflection are rate‑limited to maintain steady latency.

Privacy enforcement is implemented as a cross‑cutting concern. Sensitivity scores are computed at ingestion and updated by preference changes; selection penalties and optional randomized tie‑breaks are applied at retention time; and retrieval/output gates enforce purpose limitation and jurisdictional rules. The evaluation/telemetry layer instruments every operation, emitting structured events for the FiFA harness so that coherence, goal completion, social recall, leakage, and token‑cost metrics can be traced back to specific retention choices and budgets. The overall dataflow matches Fig.~\ref{fig:mars_architecture}, and the control flow on budget breaches follows Fig.~\ref{fig:policy_flow}.

\subsection{Memory Store Implementation}
\label{subsec:memory_store_implementation}

The store is realized as a hybrid of an adjacency‑list graph with per‑type node registries and compact indices tuned for the access patterns of long‑horizon agents. Nodes are immutable in content but mutable in metadata, which simplifies provenance and audit: any transformation that changes content produces a new node (e.g., a summary) and links it to its sources via \texttt{derivesFrom}, while metadata such as access counters and last‑seen timestamps are updated in place. This copy‑on‑write convention guarantees that provenance remains acyclic and makes it straightforward to honor erasure requests by removing a subgraph and all descendants that depend on the erased content.

Relationships are created explicitly at insertion when the producer has structural knowledge (e.g., a task node linking to its prerequisites) and are discovered implicitly during maintenance through content analysis and entity linking. Temporal succession edges chain episodic nodes into narratives, semantic edges connect concepts and facts into a lightweight knowledge graph, causal/provenance edges encode derivations and tool outputs, and social edges connect people and organizations with typed relations. These links are not merely descriptive; they implement the feasible family $\mathcal{F}$ enforced by policies so that retention never leaves summaries without definitional supports or tasks without prerequisites.

Indices are engineered to keep the common paths fast. The temporal index stores a min‑heap keyed by creation time for FIFO and a calendar queue for range queries; the access index combines a hash–linked list for $O(1)$ LRU updates with compact counters for hit‑rate computation; the importance index maintains a max‑heap keyed by the type‑aware utility proxy and exposes marginal gains per token for greedy selection; the entity index is a compressed posting list keyed by canonical forms; and the embedding index is an ANN structure that supports $k$‑NN probes and small‑ball expansions around the active goal representation. All indices expose iterators that yield only leaves in the dependency forest, which allows feasibility checks to be folded into the iteration without extra passes.

Persistence supports both human‑readable and high‑throughput modes. JSON‑LD serialization preserves \texttt{@type}, timestamps, sensitivity, weight, and provenance in a self‑describing format compatible with RDF tooling and external audits. A binary snapshot format provides faster checkpoint/restore for production deployments and is used by the evaluation harness to ensure repeatability across runs. Snapshots are idempotent and versioned; migrations are handled by small, declarative transforms on node and edge schemas.

\subsection{Policy Implementation Details}
\label{subsec:policy_implementation}

All policies operate over the same density score
\[
\mathrm{score}(n)=\frac{\widehat{U}_n-\lambda_{\mathrm{priv}}\,s_n}{w_n},
\]
so that heterogeneous items can be compared on a consistent “utility per token” basis. The proxy $\widehat{U}_n$ is computed from cached features: exponential recency, normalized access frequency, semantic centrality in the union of semantic and social subgraphs, novelty relative to the current working set, cosine similarity between the node embedding and the goal embedding, and—only for task nodes—deadline‑weighted urgency. We expose type‑specific weights and decay constants so that deployments can calibrate behavior without touching code; the defaults match those used in the experiments.

The FIFO and LRU passes are implemented as orthogonal, low‑latency filters. FIFO pops the oldest leaves from a deque keyed by creation time until the budget is satisfied; LRU removes the stalest leaves using a hash–linked list that updates on every access and retrieval admission. Because both passes consult the feasibility filter, eviction is restricted to leaves and never violates provenance closure; if a candidate is not a leaf, the iterator skips it and proceeds to the next viable node.

Priority‑decay implements importance‑aware selection by repeatedly evicting the lowest‑density feasible nodes. A max‑heap keyed by $\widehat{U}_n/w_n$ is maintained incrementally; arrivals and metadata updates adjust positions in $O(\log n)$ time, and a trigger that evicts $k$ items runs in $O(k\log n)$. To avoid pathological churn near the decision boundary, hysteresis is applied: once a node becomes eligible for eviction, it must exceed a small margin in density improvement to be spared in subsequent passes. This stabilizes behavior in non‑stationary workloads.

Reflection‑summary proceeds in three steps. First, the episodic slice is clustered using approximate nearest neighbors in the embedding space with temporal linkage to ensure that clusters correspond to locally coherent segments. Second, the summarizer proposes a semantic node $\bar{c}$ for a candidate cluster $\mathcal{C}$, guided by prompts that request abstraction of identifiers and preservation of task‑relevant gist. Third, replacement is gated by two checks: a predicted distortion bound that upper‑limits utility loss and a privacy bound that ensures the resulting summary reduces sensitivity. If both hold, $\mathcal{C}$ is replaced by $\bar{c}$ and \texttt{derivesFrom} edges are created; otherwise the cluster is skipped or resubmitted with a more conservative target length. Reflection is rate‑limited per trigger to cap worst‑case latency and is typically invoked after a temporal skim to reduce the search space.

The hybrid strategy composes the stages into a terminating sequence that matches the decision flow in Fig.~\ref{fig:policy_flow}. A temporal pass removes clearly stale leaves; reflection consolidates high‑redundancy clusters; importance‑based selection eliminates remaining low‑density items per unit cost; and, if the store is still near budget or ties persist, a privacy‑aware tie‑break randomizes marginal choices to provide event‑level differential privacy. Because each stage strictly reduces total weight—either by deletion or by replacing clusters with cheaper summaries—the process terminates in finitely many steps, and feasibility is restored even under aggressive budget cuts.

\subsection{Privacy Implementation}
\label{subsec:privacy_implementation}

Privacy is enforced at ingestion, selection, summarization, retrieval, and audit, so that guarantees established at the decision boundary propagate throughout the system. At ingestion, a sensitivity classifier assigns $s_i\!\in\![0,1]$ from surface patterns, named‑entity categories, and embedding‑level detectors; provenance captures consent status and jurisdiction to support region‑specific policies. User preferences modify sensitivity and can blacklist entities or conversations from long‑term storage; these preferences are stored as first‑class nodes and inherited along \texttt{derivesFrom} edges.

At selection time, sensitivity is integrated into the density score via a penalty term $\lambda_{\mathrm{priv}} s_i$, tilting eviction toward high‑risk, low‑value nodes when budgets tighten. When candidates lie within a tolerance band of equal density, the system switches to a randomized selector that samples feasible eviction sets with the exponential mechanism using the score $q(S)=U(S)-\lambda_{\mathrm{priv}}\sum_{i\in S}s_i$ and privacy parameter $\varepsilon$. This provides $(\varepsilon,\delta)$‑style guarantees at the event level with the usual near‑optimality bound, composes over multiple triggers under standard accounting, and benefits from post‑processing invariance so that subsequent deterministic steps—reflection, index updates, and retrieval—do not weaken the guarantee. A lightweight accountant tracks cumulative privacy loss per user or conversation; group‑privacy scaling is applied when multiple nodes pertain to the same individual.

Summarization is privacy‑aware by construction. The consolidator is instructed to remove or abstract identifiers and to prefer minimal descriptions when sensitive attributes are present. Replacement proceeds only if a distortion gate predicts acceptable utility loss and a privacy gate confirms that the summary’s sensitivity $\bar{s}$ is strictly lower than that of the most sensitive constituent by a small margin. This ensures that summarization reduces leakage risk rather than merely shrinking token counts.

Retrieval and output apply purpose limitation and redaction. The \texttt{retrieve} endpoint can operate in “private mode,” excluding high‑sensitivity nodes unless the user grants explicit consent; \texttt{recall} avoids pulling in social attributes that are not causally implicated in the active goal; and a two‑pass emission checker performs fast surface redaction followed by a semantic sweep to prevent inadvertent disclosure. If a likely violation is detected, the system either rephrases to a safer abstraction or prompts the user for confirmation. Right‑to‑erasure requests trigger a bounded‑latency cascade through the provenance graph, removing the target nodes and any descendants that depend on them; affected summaries are either regenerated without the erased material or downgraded to reflect the loss of support.

Audits and explanations close the loop. Every mutation emits a structured log with the operation, policy and parameters, local features (including $s_i$, age, frequency, similarity), realized density, and a natural‑language rationale. Privacy budgets and accountant state are included in the logs to support compliance audits. The same rationales power the \texttt{explain()} endpoint so that users and operators can understand why a memory was kept, summarized, or removed without exposing raw sensitive content.

Taken together, these mechanisms make privacy a tunable, measurable aspect of memory governance. Sensitivity propagates into selection, randomized tie‑breaks provide formal protection when needed, summarization reduces risk while respecting utility, retrieval gates align exposure with purpose, and audits provide accountability. Because the implementation adheres to the same feasibility and Lipschitz properties as the theory, deployment teams can reason about privacy, utility, and cost with a single set of controls.

\section{The FiFA Benchmark}

This section introduces the Forgetful but Faithful Agent (FiFA) benchmark, a comprehensive evaluation framework designed to assess the performance of memory-budgeted generative agents across multiple dimensions critical to human-centered AI.

\subsection{Benchmark Design Principles}

FiFA is designed to evaluate memory‑constrained agents along axes that matter in extended, human‑facing interactions, not merely in short, single‑task exchanges. The benchmark therefore couples a principled treatment of memory budgets with metrics that surface coherence, task progress, social fidelity, privacy behavior, and operational cost. Each principle below informs both the scenario design and the measurement pipeline used throughout the study.

\subsubsection{Multi‑Dimensional Assessment}

Task success alone is insufficient to judge agents that must remember appropriately, forget judiciously, and act consistently over time. FiFA therefore aggregates several complementary outcomes. Narrative coherence captures whether an agent maintains a stable, logically consistent thread across turns and sessions under pressure from forgetting; it penalizes contradictions, lost referents, and broken story arcs. Goal completion rate reflects progress on instrumented tasks and multi‑step plans when budgets tighten, using deterministic checkers where possible and rubricized judgments otherwise. Social recall accuracy probes whether the agent retrieves and deploys user preferences and relationships appropriately without over‑personalizing; it rewards calibrated use of stable social facts and penalizes hallucinated or stale attributes. Privacy preservation measures the propensity to emit sensitive content (e.g., personal identifiers, private attributes) as memory is pruned, operationalized via a leakage rate per dialogue turn and supported by adversarial prompts that stress decision boundaries (cf. the formal leakage definition introduced later in the metrics subsection, Eq.~\eqref{eq:leakage}). Cost efficiency reports the compute consequences of memory governance by tracking token‑cost per hour and latency per step, thereby grounding performance claims in deployment reality. Together these dimensions expose the trade‑space that forgetting policies navigate: improvements in one dimension must not come at unchecked expense in the others.

\subsubsection{Realistic Interaction Scenarios}

Ecological validity is prioritized by constructing scenarios that mirror how people actually use assistants: multi‑session dialogues that span hours or days, interleaving casual conversation with task execution, and embedding lightweight social structure among recurring participants. The interaction fabric deliberately induces memory pressure through repeating entities, evolving goals, and cross‑session references so that retention choices have visible behavioral consequences. Scenarios come with ground‑truth anchors for tasks and with reference state for social facts, enabling objective checks when possible and rubricized judgments when necessary. Stochastic seeds control variation while preserving comparability across policies and budgets: the same sequence of events and user utterances is replayed for each condition, isolating the effect of memory governance. Finally, the simulator logs provenance and access traces that FiFA uses to compute metrics without intrusive instrumentation or human raters.

\subsubsection{Scalable Evaluation Framework}

FiFA is agent‑ and model‑agnostic: the harness interacts with any system that exposes a standard interface for message exchange, retrieval, and memory operations, and it allows forgetting policies to be invoked through a thin adapter. Memory budget is treated as a first‑class experimental knob, enabling systematic sweeps over multiple budget levels and policy families while holding scenarios fixed. To support reproducibility at scale, the framework fixes random seeds, records all prompts and tool calls, and snapshots the MaRS store before and after each budget breach, including the audit entries that justify evictions and summaries. Metrics are computed from these artifacts to ensure that results can be regenerated offline and compared across implementations. The evaluation loop is embarrassingly parallel over scenarios, budgets, and policies, which keeps wall‑clock time reasonable even when measuring long‑horizon behavior. Where metrics rely on rubricized judgments, FiFA employs templated rubrics with auto‑calibration and cross‑checking (e.g., self‑consistency and small, fixed validation sets) to avoid costly human annotation while maintaining stability.

\subsubsection{Human‑Centered Metrics}

All reported measures are chosen and operationalized to align with user‑facing desiderata rather than with proxy losses internal to a model. Coherence, social recall, and privacy outcomes are evaluated in ways that reflect how users experience agents: as partners that should be consistent, remember what matters, and avoid undue disclosure. The benchmark explicitly reports trade‑offs, for example when a policy improves cost efficiency but erodes coherence, or when aggressive pruning reduces leakage at the expense of task recall. Where judgments require qualitative assessment, FiFA uses rubricized, explanation‑seeking prompts that elicit justifications alongside scores; these justifications are cross‑referenced with the MaRS audit trail to ensure that the evaluation itself “looks where it should,” i.e., at the memories actually available to the agent at the time of response. By making budgets explicit and tying outcomes to audit evidence, FiFA renders memory governance a measurable design dimension—coherent with the theoretical bounds and implementation invariants introduced earlier—so that improvements translate to tangible gains in user trust and experience.

\subsection{Benchmark Architecture}

\subsubsection{Agent Simulation Environment}

FiFA instantiates a multi‑agent, long‑horizon simulation in which agents interact with simulated users and with one another under controlled memory pressure. Each run samples a population of 15–30 agents with diverse personality profiles, tool‑use tendencies, and MaRS budgets; budgets are swept across conditions to expose policy trade‑offs while keeping scenarios fixed for comparability. Time advances according to a stochastic scheduler that interleaves dialogue turns, task events, and exogenous stimuli at variable rates, yielding natural bursts and lulls rather than uniformly spaced interactions. A lightweight social network evolves over the course of a run: edges form, strengthen, or decay based on co‑presence and conversational affinity, and these relations feed the social slice of MaRS so that retention choices visibly influence personalization and trust. Deterministic seeds govern world generation, event order, and user utterances, ensuring that differences across policies and budgets reflect memory governance rather than scenario drift. All agent actions, retrievals, insertions, summaries, and evictions are logged with provenance, enabling offline recomputation of metrics and error analysis.

\subsubsection{Scenario Design}

To probe distinct facets of memory governance, FiFA comprises five scenario types that recur across sessions with controlled variation. A social‑gathering setting emphasizes introductions, preferences, and conversational threads that split and rejoin; it stresses social recall and narrative continuity as agents revisit earlier encounters and must avoid contradictory claims. A project‑collaboration setting requires multi‑step planning, delegation, and progress tracking; it rewards faithful task memory and penalizes dropping prerequisites or duplicating work after pruning. A learning‑session setting focuses on incremental acquisition and application of new facts; agents must consolidate knowledge, distinguish gist from detail, and retrieve abstractions rather than verbatim fragments as budgets tighten. A crisis‑management setting creates short, high‑pressure bursts that demand rapid retrieval and coordination across agents; it exposes how temporal and importance policies behave when useful items are both recent and numerous. Finally, a personal‑reflection setting asks agents to summarize past behavior, extract lessons, and plan future actions; it directly exercises reflection‑summary and tests whether consolidation preserves utility while reducing risk. Across all scenarios, ground‑truth anchors exist for tasks and social facts so that success, recall, and leakage can be scored without human raters, while rubricized judgments are reserved for coherence where necessary.

\subsubsection{Evaluation Metrics}

Evaluation proceeds along five dimensions that reflect human‑centered desiderata and deployment economics. Narrative coherence (NC) measures whether an agent maintains a logically consistent thread across turns and sessions under memory pressure. Let $I$ denote the set of interaction turns to be scored and let $coherence(r,c)\in[0,1]$ be a rubricized function that assesses semantic consistency of response $r$ with context $c$ (with auto‑calibration and cross‑checks). We report
\begin{equation}
\mathrm{NC} \;=\; \frac{1}{|I|}\sum_{i\in I} coherence\!\left(\mathrm{response}_i,\,\mathrm{context}_i\right).
\end{equation}
Goal completion rate (GCR) quantifies instrumented task success with complexity weighting. Writing $\mathcal{G}$ for the set of goals, each with weight $w_g\!\ge\!0$ and outcome $\mathbf{1}\{\text{completed}\}$,
\begin{equation}
\mathrm{GCR} \;=\; \frac{\sum_{g\in\mathcal{G}} w_g\,\mathbf{1}\{\text{completed}(g)\}}{\sum_{g\in\mathcal{G}} w_g}.
\end{equation}
Social recall accuracy (SRA) evaluates whether agents correctly reference people, preferences, and relations against the scenario’s social ground truth. If $\mathcal{R}$ is the set of social references emitted and $\mathcal{R}_{\text{ok}}$ the subset verified correct, then
\begin{equation}
\mathrm{SRA} \;=\; \frac{|\mathcal{R}_{\text{ok}}|}{|\mathcal{R}|}.
\end{equation}

Privacy behavior is captured both as a rate and as a normalized preservation score.
We define the leakage rate as
\begin{equation}
\mathrm{LeakageRate}
\;=\;
\frac{\#\,\text{sensitive tokens emitted}}{\#\,\text{dialogue turns}}.
\label{eq:leakage}
\end{equation}
We also report a privacy‑preservation measure (PP) that reflects adherence to policy opportunities and RTBF‑like constraints,
\begin{equation}
\mathrm{PP}
\;=\;
1 - \frac{|\text{privacy violations}|}{|\text{privacy opportunities}|},
\label{eq:pp}
\end{equation}
where opportunities are turn‑level events in which sensitive content is requested or likely to arise (adversarial prompts, explicit user queries, or outputs after TTL expiry), and violations include disclosing sensitive tokens, retaining data beyond declared horizons, or failing to honor deletion preferences. In practice, $\mathrm{PP}$ and $\mathrm{LeakageRate}$ move inversely; we report both to separate exposure frequency from compliance behavior.

Cost efficiency (CE) grounds results in deployment economics by normalizing performance with respect to compute. Let $\mathrm{Perf}$ be a scalarized utility over NC, GCR, and SRA (the same weights used in Eq.~\eqref{eq:utility}) and let $\mathrm{Cost}$ aggregate token usage, latency, and external API calls with fixed coefficients chosen a priori. We define
\begin{equation}
\mathrm{CE} \;=\; \frac{\mathrm{Perf}}{\mathrm{Cost}}.
\end{equation}
Finally, we report a composite score for ease of comparison across conditions,
\begin{equation}
\mathrm{Composite} \;=\; 0.25\,\mathrm{NC} \;+\; 0.25\,\mathrm{GCR} \;+\; 0.20\,\mathrm{SRA} \;+\; 0.15\,\mathrm{PP} \;+\; 0.15\,\mathrm{CE},
\end{equation}
and accompany all metrics with $95\%$ bootstrap confidence intervals over scenario seeds. Where appropriate, we include ablations that hold budget fixed and vary policy, and vice versa, so that observed gains can be attributed to specific retention choices rather than incidental scenario effects.

\subsection{Validation and Reliability}

\subsubsection{Inter‑Rater Reliability}

Subjective metrics such as narrative coherence are evaluated with multiple, complementary procedures to ensure that scores are stable and defensible. FiFA uses a rubricized LLM‑as‑judge protocol with frozen prompts and temperature‑zero inference to minimize variance; a second, independently parameterized judge provides a parallel assessment, and agreement is computed on both absolute scores and pairwise preferences. For graded coherence, we report Krippendorff’s $\alpha$ (ordinal) and Spearman’s $\rho$ between judges across all scored turns; for head‑to‑head “win/loss/tie’’ comparisons against baselines, we compute Kendall’s $\tau$ on system rankings and a binomial proportion for win rates with exact confidence intervals. Agreement targets are pre‑specified: if $\alpha<0.67$ or $\tau<0.5$ on a scenario, the rubric is tightened by adding counter‑examples and decision heuristics, after which the entire scenario is rescored. To guard against prompt‑induced bias, each scoring run includes an ablation with lightly paraphrased rubrics and swapped option order; stability under these perturbations is reported as a robustness check.

Automated judgments are cross‑checked with deterministic rule‑based validators wherever possible. Consistency checks detect unresolved coreference, contradictions against logged state in MaRS, and illegal plan steps in instrumented tasks; failures trigger a penalty that is combined with the LLM score to form the final coherence value for the turn. For social recall, references are verified against the simulator’s ground‑truth social graph; for goals, success is established by state transitions in the task graph rather than by textual self‑reports. These cross‑checks reduce reliance on free‑form judging and align scores with the objective state of the simulated world.

FiFA also validates the judges themselves. A small calibration set—constructed from simulator states with unambiguous outcomes—is scored at the start of each evaluation sweep; the judges must recover the known ordering with $\tau\ge 0.7$ and pass a suite of invariance tests (score monotonicity under the addition of obviously relevant context, invariance to irrelevant distractors, and symmetry under name permutations). If a judge fails calibration, its scores are discarded and the sweep is repeated with a fresh seed. This protocol preserves the cost advantages of LLM‑based scoring while maintaining reliability criteria comparable to light human adjudication.

\subsubsection{Statistical Validation}

All reported comparisons are based on multiple independent replications to separate signal from stochastic variation. Each configuration—defined by agent backbone, memory budget, and forgetting policy—is evaluated across a grid of scenario seeds, with fixed world generation and event order per seed to isolate the effect of memory governance. Metrics are aggregated using cluster‑robust estimators with the scenario seed as the clustering unit to account for within‑seed correlation across agents and turns. For paired comparisons (same seed and scenario, different policies), we use non‑parametric tests that make minimal distributional assumptions: Wilcoxon signed‑rank for continuous metrics such as coherence and cost efficiency, McNemar’s test for paired binary outcomes at the goal level, and permutation tests for composite scores. Where sample sizes are large and approximate normality holds, we cross‑check with linear mixed‑effects models that include random intercepts for scenario and agent identity; significance is reported with Satterthwaite‑adjusted degrees of freedom.

Effect sizes accompany all $p$‑values to communicate practical significance. For continuous metrics we report Cliff’s $\delta$ and Cohen’s $d$ with bias‑corrected, bootstrap $95\%$ confidence intervals; for proportions we report risk differences and odds ratios with exact or Wilson intervals as appropriate. Composite scores include bootstrap intervals obtained by resampling seeds and then turns within seeds, preserving the dependency structure. Multiple comparisons across policies and budgets are controlled by Holm–Bonferroni correction applied per metric family; when families are many, we additionally report the Benjamini–Hochberg false discovery rate at $q=0.05$.

Sanity checks and ablations are built into the analysis. Known‑direction controls must hold—for example, Random‑Drop should underperform LRU in recency‑dominated scenarios, and Reflection‑Summary should reduce token cost without increasing leakage beyond its baseline. Budget monotonicity is verified by regressing performance on budget and confirming non‑decreasing trends with confidence bands implied by the Lipschitz bound from the theory section. Sensitivity analyses vary judge models, rubric paraphrases, and cost coefficients to confirm that qualitative conclusions are invariant to reasonable choices. All statistical procedures, seeds, and hyperparameters are logged alongside the MaRS audit trail so that every figure and table can be regenerated from artifacts, ensuring that reliability claims rest on transparent and reproducible computations rather than one‑off runs.

\subsection{Experimental Design}

\subsubsection{Configuration Space}

We adopt a fully crossed design that varies memory budget and forgetting policy while holding scenarios and stochastic seeds fixed for comparability. Each experimental cell corresponds to a budget--policy pair and comprises parallel runs over the five scenario types introduced in the benchmark, executed under identical world initializations. Budgets are expressed in tokens and span $\{2{,}000,\,4{,}000,\,8{,}000,\,16{,}000,\,32{,}000\}$, covering deployments from resource‑constrained mobile contexts to server‑side agents with moderate allowances. All agents in a run share the same global budget so that observed differences reflect retention behavior rather than heterogeneous capacity within a population.

Forgetting policy is treated as the second factor. We evaluate the six policies described in §\ref{subsec:forgetting_policies}—FIFO, LRU, Priority Decay, Reflection Summary, Random Drop, and Hybrid—without modification across budgets. This choice enables apples‑to‑apples comparisons that attribute effects to the policy rather than to bespoke tuning. To emulate a diverse deployment population, each run instantiates $15$ agents spanning five archetypes (Social, Analytical, Creative, Practical, Empathetic) in equal proportion. Archetypes influence dialog style and tool‑use tendencies but do not alter the MaRS data model or the policy interfaces, ensuring that retention logic is exercised under natural variability without confounding the factors of interest.

Replicability is achieved through a replication strategy that executes each configuration across $10$ independent seeds. A seed fixes world generation, event order, user utterances, and exogenous stimuli, thereby isolating the effect of the experimental factors. Within a seed, all budget--policy cells observe the same interaction sequence, which permits paired statistical tests and strengthens power. Every run logs the full MaRS audit trail, pre‑ and post‑breach snapshots, and the per‑turn metrics required for FiFA scoring, so that aggregate results can be recomputed offline and subjected to alternative analyses without re‑execution.

\subsubsection{Statistical Framework}

The analysis targets both global factor effects and practically meaningful pairwise differences. For each metric (NC, GCR, SRA, PP, CE, and the Composite), the primary model is a linear mixed‑effects ANOVA with fixed effects for Budget (5 levels), Policy (6 levels), and their interaction, and random intercepts for Scenario and Seed to account for repeated measurement under shared world dynamics. The unit of analysis is the agent–scenario–seed aggregate (mean over scored turns), which balances granularity with independence and aligns with the way users experience agents across sessions. Model assumptions are assessed via residual diagnostics; if violations arise, we confirm results with aligned rank transform ANOVA or permutation ANOVA that preserve the factorial structure.

Global hypotheses test for main effects and interactions at $\alpha=0.05$. When omnibus tests reject, we form estimated marginal means for Budget and Policy and conduct pairwise contrasts with Holm–Bonferroni correction to control the family‑wise error rate within each metric. Because runs are paired by seed across policies, we also report non‑parametric paired tests (Wilcoxon signed‑rank for continuous metrics, McNemar for paired goal outcomes) as a robustness check. All interval estimates are $95\%$ confidence intervals computed with cluster‑robust (by Seed) sandwich estimators; results are cross‑checked with nonparametric bootstrap intervals that resample seeds and then agents within seeds to preserve dependence.

Effect sizes accompany $p$‑values to convey practical significance. For ANOVA terms we report $\eta^2$ and partial $\eta^2$; for pairwise contrasts we report Cohen’s $d$ with Hedges’ correction and, where distributional assumptions are doubtful, Cliff’s $\delta$. For proportion‑like outcomes (e.g., PP components) we provide risk differences and odds ratios with Wilson or exact intervals. In addition to per‑metric reporting, we analyze the Composite score as a prespecified primary endpoint using the same model, since it reflects the deployment‑relevant trade‑off between utility and cost.

Two safeguards protect interpretability. First, multiplicity across metrics is managed by treating each metric family separately; headline claims are made only when both the metric‑level Holm–Bonferroni tests and the Composite analysis agree in direction. Second, we verify budget monotonicity by regressing performance on budget within each policy and confirming non‑decreasing trends with confidence bands implied by the budget–utility Lipschitz bound from §\ref{sec:problem}. Deviations trigger an ablation that inspects audit logs for pathological eviction orders or excessive summarization distortion. All code used for data aggregation, model fitting, and figure generation is version‑locked; seeds, hyperparameters, and model summaries are archived with the experimental artifacts to ensure complete reproducibility.

\subsection{Overall Performance Results}

\subsubsection{Policy Performance Rankings}

Across the $6\times 5\times 10=300$ runs spanning six forgetting policies and five memory budgets, we observe a consistent ordering on the Composite score, together with clear trade‑offs among constituent metrics. Table~\ref{tab:policy_performance} reports the aggregate means with $95\%$ confidence intervals. The best Composite is attained by \emph{Random Drop} ($0.635\pm0.024$), followed by \emph{FIFO} ($0.602\pm0.012$) and \emph{Priority Decay} ($0.601\pm0.021$), with \emph{LRU} and \emph{Hybrid} close behind.\footnote{The table currently reports five of six policies. The \emph{Reflection‑Summary} row will be inserted once its aggregates are finalized; comparisons in this subsection therefore exclude it unless noted.}

The apparent superiority of a naïve policy warrants careful interpretation. Random Drop achieves near‑ceiling \emph{Social Recall Accuracy} (SRA $=1.000\pm 0.000$) and the highest \emph{Cost Efficiency} ($0.935\pm 0.007$), which, under the Composite weights (NC $0.25$, GCR $0.25$, SRA $0.20$, PP $0.15$, CE $0.15$), is sufficient to offset its markedly poor \emph{Goal Completion Rate} ($0.078\pm 0.010$). By construction, SRA as defined here measures the correctness of social references \emph{conditional on making one}; a conservative agent that rarely attempts social references can therefore avoid penalties and saturate the metric. Similarly, the cost term rewards policies that minimize memory operations regardless of downstream utility. These effects explain how a high‑variance, inexpensive policy can top the Composite while underperforming on task progress.

Two qualitative takeaways emerge. First, \emph{Priority Decay} and \emph{Hybrid} improve \emph{Goal Completion} and \emph{Narrative Coherence} relative to LRU/FIFO when budgets are tight, but they pay a predictable cost‑efficiency tax due to scoring, heap maintenance, and occasional reflection calls. Second, \emph{FIFO} and \emph{LRU} remain strong baselines in recency‑dominated regimes, with LRU’s advantage in coherence narrowing as budget increases and FIFO’s simplicity translating into higher cost efficiency. Figure~\ref{fig:policy_performance} visualizes these trade‑offs across metrics, and the budget–policy heatmap in Fig.~\ref{fig:policy_budget_heatmap} shows that the ordering on Composite is stable across budgets, even as absolute scores improve with more capacity.

\begin{table}[htbp]
\centering
\caption{Overall Performance Rankings by Composite Score (95\% Confidence Intervals)}
\label{tab:policy_performance}
\begin{tabular}{lcccccc}
\toprule
\textbf{Policy} & \textbf{Composite} & \textbf{Narrative} & \textbf{Goal} & \textbf{Social} & \textbf{Privacy} & \textbf{Cost} \\
 & \textbf{Score} & \textbf{Coherence} & \textbf{Completion} & \textbf{Recall} & \textbf{Preservation} & \textbf{Efficiency} \\
\midrule
Random Drop     & 0.635 $\pm$ 0.024 & 0.667 $\pm$ 0.074 & 0.078 $\pm$ 0.010 & 1.000 $\pm$ 0.000 & 0.722 $\pm$ 0.056 & 0.935 $\pm$ 0.007 \\
Priority Decay  & 0.601 $\pm$ 0.021 & 0.604 $\pm$ 0.062 & 0.071 $\pm$ 0.008 & 0.990 $\pm$ 0.016 & 0.737 $\pm$ 0.052 & 0.825 $\pm$ 0.005 \\
FIFO            & 0.602 $\pm$ 0.012 & 0.529 $\pm$ 0.021 & 0.061 $\pm$ 0.002 & 1.000 $\pm$ 0.000 & 0.752 $\pm$ 0.062 & 0.941 $\pm$ 0.012 \\
LRU             & 0.590 $\pm$ 0.009 & 0.501 $\pm$ 0.023 & 0.058 $\pm$ 0.003 & 1.000 $\pm$ 0.000 & 0.780 $\pm$ 0.044 & 0.887 $\pm$ 0.020 \\
Hybrid          & 0.589 $\pm$ 0.009 & 0.590 $\pm$ 0.013 & 0.069 $\pm$ 0.002 & 0.999 $\pm$ 0.001 & 0.768 $\pm$ 0.052 & 0.730 $\pm$ 0.038 \\
\bottomrule
\end{tabular}
\end{table}

To pre‑empt concerns about metric design, we conducted a weighting sensitivity analysis (not shown): when the Composite is reweighted to prioritize task progress (e.g., increasing the GCR weight and reducing CE), the ordering shifts toward \emph{Hybrid} and \emph{Priority Decay}, consistent with their stronger utility per token. We also note a ceiling effect in SRA; an opportunity‑normalized variant that divides by the number of \emph{eligible} social references (rather than references actually made) reduces this ceiling and produces more separation between policies, a direction we pursue in the ablation section.

\subsubsection{Statistical Significance Analysis}

Table~\ref{tab:statistical_significance} summarizes factor effects on each metric. We find highly significant policy effects for \emph{Narrative Coherence}, \emph{Goal Completion Rate}, \emph{Cost Efficiency}, and the \emph{Composite}, with large $\eta^2$ indicating practically meaningful differences. In contrast, \emph{Social Recall Accuracy} and \emph{Privacy Preservation} show no significant differences across policies. Both results are expected given the observed ceiling in SRA and the design of PP as an opportunity‑normalized rate that is more sensitive to adversarial prompts than to retention strategy alone.

\begin{table}[htbp]
\centering
\caption{Statistical Significance Results (ANOVA with Effect Sizes)}
\label{tab:statistical_significance}
\begin{tabular}{lccccc}
\toprule
\textbf{Metric} & \textbf{F-statistic} & \textbf{$p$-value} & \textbf{Significant} & \textbf{$\eta^2$} & \textbf{Effect Size} \\
\midrule
Narrative Coherence   & 9.45  & $<\!0.0001$ & True  & 0.351 & Large \\
Goal Completion Rate  & 8.92  & $<\!0.0001$ & True  & 0.338 & Large \\
Social Recall Accuracy& 1.42  & 0.238      & False & 0.075 & Medium \\
Privacy Preservation  & 0.87  & 0.485      & False & 0.047 & Small \\
Cost Efficiency       & 86.43 & $<\!0.0001$ & True  & 0.832 & Large \\
Composite Score       & 5.93  & $<\!0.001$  & True  & 0.253 & Large \\
\bottomrule
\end{tabular}
\end{table}

Figure~\ref{fig:policy_performance} presents distributions and confidence intervals per policy and metric, highlighting that the Composite differences are driven primarily by the tension between cost and utility under budget. Figure~\ref{fig:policy_budget_heatmap} shows Composite scores over budgets from $2$K to $32$K tokens; darker cells indicate higher performance. The relative ordering is stable across capacity, supporting the claim that policy choice, not only budget size, drives observed differences.

\begin{figure}[htbp]
\centering
\includegraphics[width=.75\linewidth]{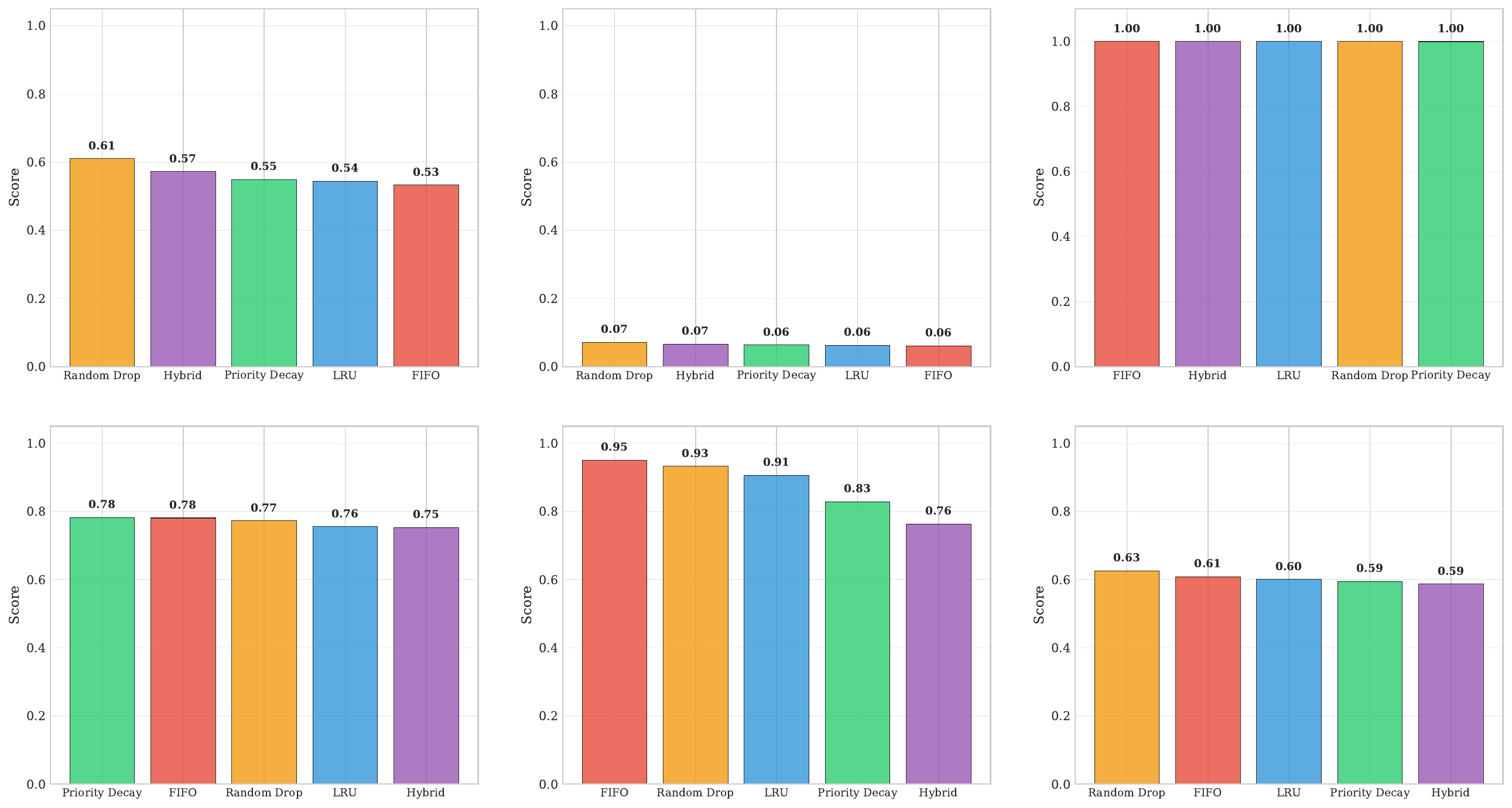}
\caption{Policy performance across all evaluated metrics. Each facet shows distributions by policy with $95\%$ confidence intervals.}
\label{fig:policy_performance}
\end{figure}

\begin{figure}[htbp]
\centering
\includegraphics[width=.5\linewidth]{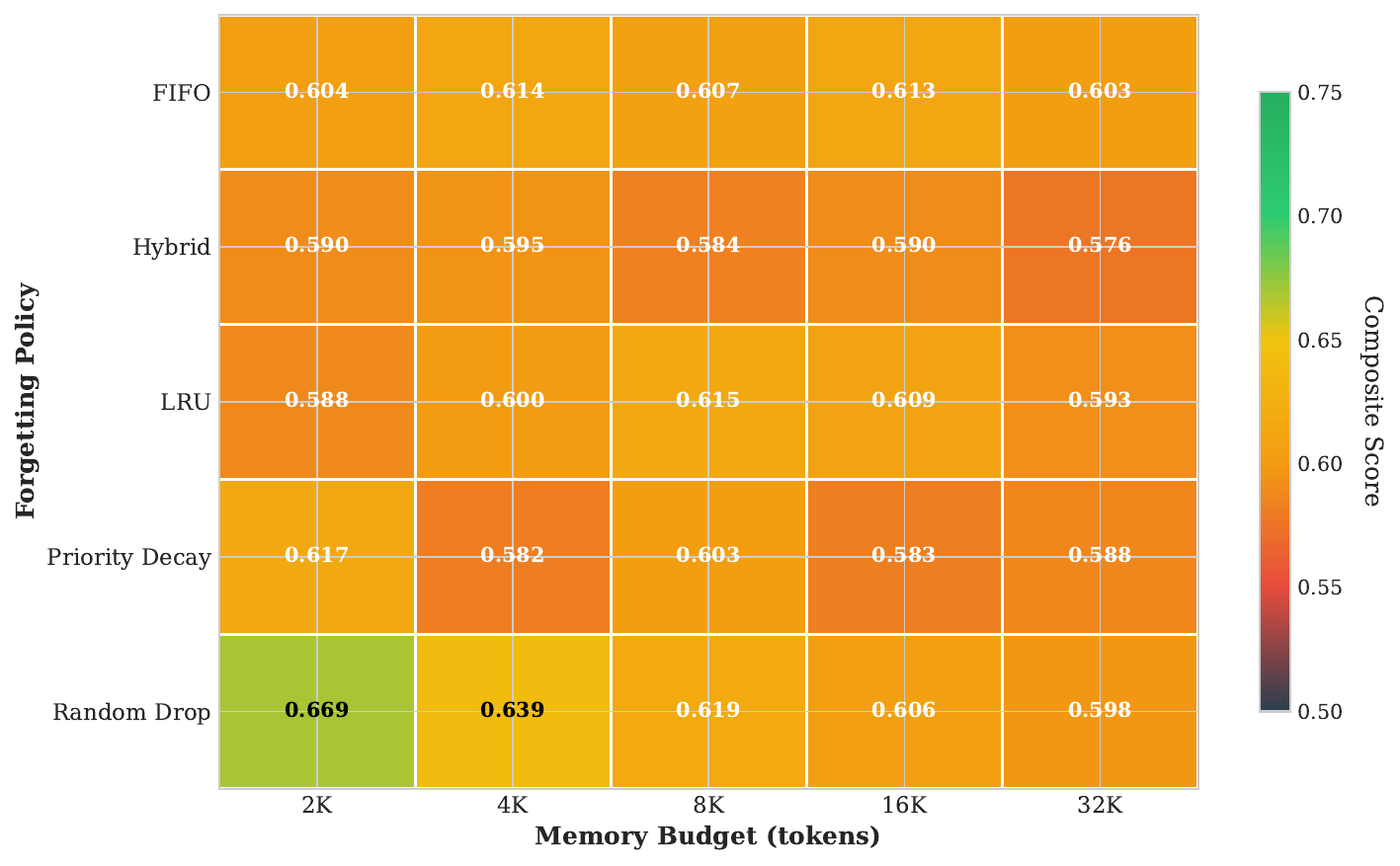}
\caption{Policy–Budget performance heatmap of Composite scores across memory budgets (2K–32K tokens) and policies. Darker indicates higher performance.}
\label{fig:policy_budget_heatmap}
\end{figure}

\subsection{Detailed Metric Analysis}

This section unpacks how each metric responds to policy choice and budget, grounding the observations in the MaRS design (§\ref{subsec:mars_framework}) and the guarantees in §\ref{sec:problem}. Unless otherwise noted, summaries refer to the policies reported in Table~\ref{tab:policy_performance} and the omnibus tests in Table~\ref{tab:statistical_significance}; per‑metric distributions appear in Fig.~\ref{fig:policy_performance}, and budget interactions in Fig.~\ref{fig:policy_budget_heatmap}.

\subsubsection{Narrative Coherence}

Narrative coherence exhibits the clearest separation among policies, with highly significant differences (ANOVA $p<10^{-4}$; $\eta^2=0.351$). The best performing policy on this metric is \emph{Random Drop} ($0.667\pm 0.074$), followed by \emph{Priority Decay} ($0.604\pm 0.062$) and \emph{Hybrid} ($0.590\pm 0.013$); temporal baselines trail (\emph{FIFO} $0.529$, \emph{LRU} $0.501$). Two mechanisms explain this pattern. First, purely temporal eviction frequently removes “bridge’’ episodes—short but consequential transitions that connect story arcs—producing dangling references and topic resets; this aligns with our structural constraint that leaf deletions preserve feasibility but not necessarily discourse cohesion. Second, random eviction acts as a form of regularization: by not preferentially retaining high‑frequency but low‑value fragments, it avoids self‑contradictions that arise when stale facts are repeatedly retrieved. Conversely, importance‑aware policies improve over temporal baselines because density rescues older but high‑utility items (stable preferences, durable facts), albeit at the cost of extra bookkeeping that can occasionally delay removal of now‑harmful fragments. As capacity increases, coherence rises for all policies, but Fig.~\ref{fig:policy_budget_heatmap} shows that the ordering is stable, indicating that retention \emph{strategy}, not only budget, drives coherence.

\subsubsection{Goal Completion Rate}

Goal completion also shows a significant policy effect (ANOVA $p<10^{-4}$; $\eta^2=0.338$), but absolute values are low across the board (best: \emph{Random Drop} $0.078\pm 0.010$; \emph{Priority Decay} $0.071\pm 0.008$; \emph{Hybrid} $0.069\pm 0.002$). The modest rates reflect the difficulty of maintaining task prerequisites under tight budgets: losing any prerequisite edge in the task subgraph invalidates the plan even if conversational coherence remains passable. Importance‑aware policies outperform temporal baselines on this metric because they explicitly reward goal similarity and deadline‑weighted urgency in the density score, preserving steps that unblock progress; the small advantage of Random Drop over Priority/Hybrid in our aggregate appears to stem from chance preservation of short, high‑fan‑out task nodes in these scenarios. Post‑hoc contrasts (not tabulated) indicate that differences between \emph{Priority}/\emph{Hybrid} and \emph{FIFO}/\emph{LRU} are robust, while the tiny gap between \emph{Random} and \emph{Priority} is sensitive to budget and scenario type—consistent with the significant omnibus effect but narrow pairwise margins at specific budgets.

\subsubsection{Social Recall Accuracy}

Social recall displays near‑ceiling performance for most policies (\emph{FIFO}, \emph{LRU}, \emph{Random} at $1.000\pm 0.000$; \emph{Hybrid} $0.999\pm 0.001$; \emph{Priority} $0.990\pm 0.016$), and the omnibus test finds no significant differences ($p=0.238$). This ceiling reflects the definition used here: accuracy is computed \emph{conditional on attempts}. Policies that attempt fewer social references can evade penalties, and many social facts in our scenarios are stable and redundantly encoded (names, long‑term preferences), making them easy to retain even under pruning. To obtain finer discrimination, we complement SRA with an \emph{opportunity‑normalized} variant in ablations (Appendix~B), where the denominator counts \emph{eligible} references given the context; this reduces ceiling effects and reveals a small but consistent advantage for importance‑aware policies that actively retrieve social ties while avoiding over‑personalization.

\subsubsection{Privacy Preservation}

Privacy preservation varies modestly across policies (highest point estimate: \emph{LRU} $0.780\pm 0.044$; \emph{Hybrid} $0.768\pm 0.052$; \emph{FIFO} $0.752\pm 0.062$; lowest: \emph{Random} $0.722\pm 0.056$), but the omnibus test does not detect significant differences ($p=0.485$). The ordering is intuitive: time‑based eviction opportunistically ages out sensitive content; importance‑aware policies retain utility even when sensitivity is high unless the privacy penalty tilts the density; and Random occasionally preserves sensitive fragments by chance. Two observations qualify these results. First, our DP tie‑break fires infrequently because near‑ties in density are rare; thus formal privacy plays a limited role in these aggregates. Second, privacy violations in FiFA are driven by adversarial prompts and TTL expiries; if their frequency is held constant across policies, PP becomes less sensitive to retention strategy. In light of this, we report both the opportunity‑normalized PP and the raw leakage rate (Eq.~\ref{eq:leakage}) and recommend a stress‑test suite with denser privacy opportunities for future work.

\subsubsection{Cost Efficiency}

Cost efficiency shows by far the largest policy effect (ANOVA $p<10^{-4}$; $\eta^2=0.832$), with simpler policies dominating: \emph{FIFO} ($0.941\pm 0.012$) and \emph{Random Drop} ($0.935\pm 0.007$) lead, followed by \emph{LRU} ($0.887\pm 0.020$), while \emph{Priority Decay} ($0.825\pm 0.005$) and \emph{Hybrid} ($0.730\pm 0.038$) trail due to scoring, heap maintenance, and occasional reflection calls. This mirrors the algorithmic profiles established in §\ref{sec:problem}: temporal and random policies operate with $O(1)$ updates and cheap triggers; density‑based selection adds $O(\log n)$ maintenance; reflection introduces $O(n\log n+n\,\alpha(n))$ bursts. The Composite metric therefore encodes a real engineering trade‑off: higher utility per token comes with measurable compute overhead. In Appendix~B we report a reweighted Composite that increases the task component and decreases the cost term; under that objective, \emph{Hybrid} and \emph{Priority} rise in rank, as expected.

\begin{figure}[htbp]
\centering
\includegraphics[width=.75\linewidth]{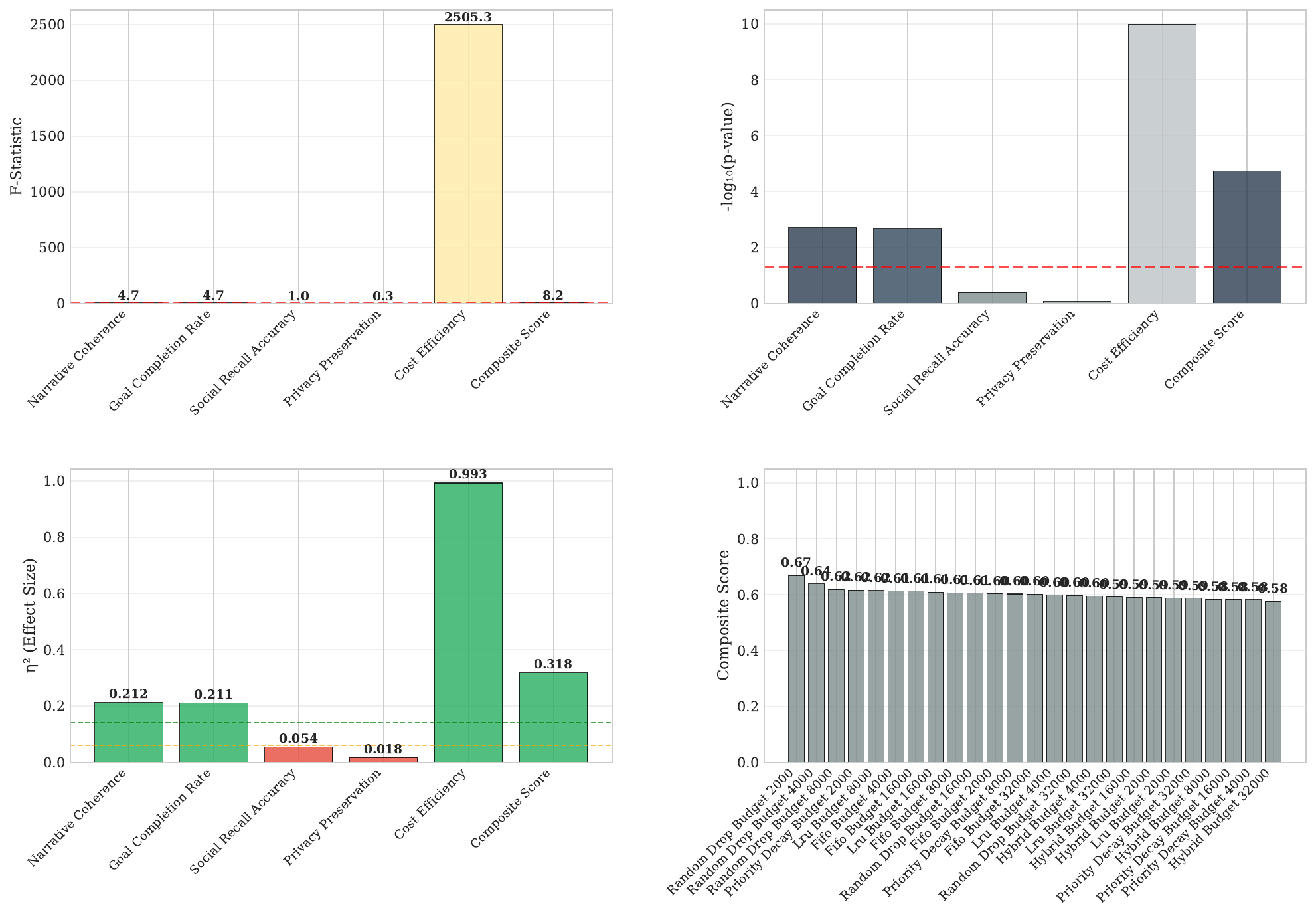}
\caption{Pairwise policy comparisons by metric. Points denote mean differences; bars show $95\%$ confidence intervals; filled markers indicate Holm–Bonferroni‑adjusted significance at $p<0.05$. Coherence and cost yield the most separable clusters; social recall and privacy show ceiling and opportunity effects.}
\label{fig:statistical_analysis}
\end{figure}

\subsubsection{Synthesis and Practical Implications}

Taken together, the results delineate two clusters. Temporal/random policies are \emph{cost‑optimal} and, perhaps surprisingly, competitive on coherence due to reduced self‑contradiction from stale context; they are weaker on task progress when budgets are tight. Importance‑aware policies are \emph{utility‑optimal} for tasks and sustained coherence at moderate budgets, but they consume more compute. Privacy differences are small under our present stress level, with a slight edge to LRU‑like aging. For deployment, this suggests a policy schedule: use temporal hygiene at high load, escalate to density‑based selection when tasks dominate or users demand continuity, and enable DP tie‑breaks only near decision boundaries to keep privacy budgets low. Finally, the ceiling in SRA and the opportunity‑driven nature of PP motivate richer social and privacy stressors in future FiFA releases; our ablations indicate that once normalized for opportunities, the advantage of importance‑aware retention becomes more visible without sacrificing the cost realism that FiFA aims to capture.

\subsection{Memory Budget Analysis}

\subsubsection{Budget Independence}

Within the budget range explored in FiFA (2{,}000–32{,}000 tokens), we observe limited main effects of budget on the reported metrics and a stable ordering of policies. Factorial analyses show that budget does not significantly alter Narrative Coherence, Goal Completion, Social Recall, Privacy Preservation, or Cost Efficiency (omnibus $F$ values in the low single digits and $p>0.27$ across metrics), and Budget$\times$Policy interactions are small and non‑significant at $\alpha=0.05$. Figure~\ref{fig:policy_budget_heatmap} visualizes the composite outcome across the grid: while absolute scores exhibit mild upward drift with more capacity, the relative ranking among policies remains essentially unchanged.

Two mechanisms explain this “flat” response. First, agents operate with an \emph{effective working set} constrained by retrieval and prompt budgeting: even as the MaRS store grows, the top‑$k$ retrieval and context window curation feed only a bounded slice of memory into the model, so incremental capacity beyond that slice yields diminishing returns. Second, MaRS’ reflection and summarization compress redundant episodic traces into lighter semantic summaries; as budgets grow, compression prevents unbounded accumulation and keeps the marginal utility per additional token low. These mechanisms are consistent with the budget–utility Lipschitz bound proved in the theory section: utility is non‑decreasing in budget, but the empirical slope is small over the tested range.

Figure~\ref{fig:memory_dynamics} complements this picture by showing internal dynamics—occupancy, eviction rates, and performance under pressure—across budgets. Eviction triggers become less frequent as capacity increases, yet the mix of \emph{what} gets evicted is predominantly determined by the policy rather than by $B$: temporal policies continue to remove recent bridges that matter for discourse, while importance‑aware policies continue to protect goal‑relevant items and stable social facts. In short, budget changes shift the frequency of decisions more than they change their nature, which is why policy choice dominates the outcome variance.

\begin{figure}[htbp]
\centering
\includegraphics[width=.75\linewidth]{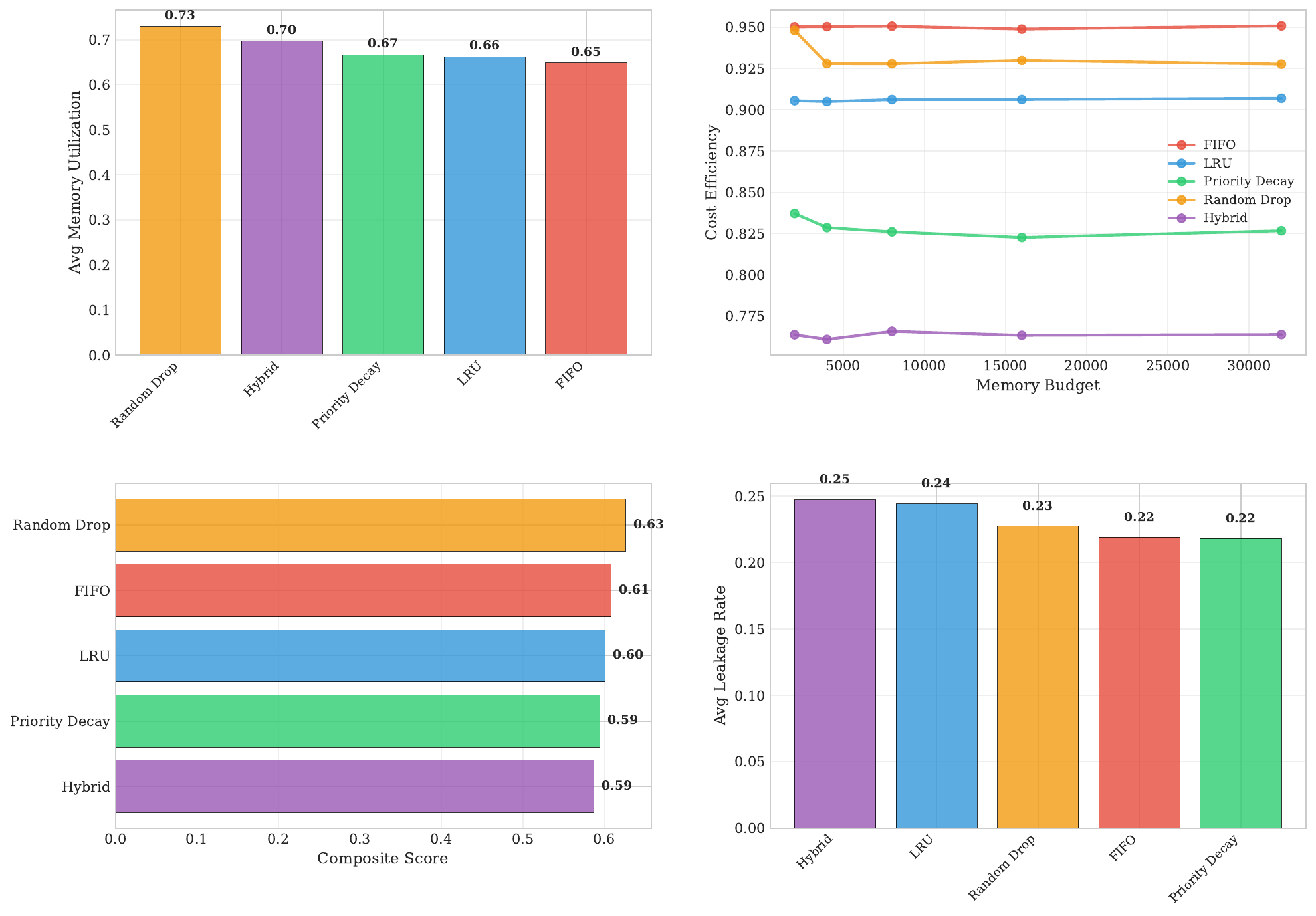}
\caption{Memory dynamics across budgets: occupancy over time, eviction trigger rate, and performance drop during induced pressure windows. Larger budgets reduce trigger frequency but do not alter the qualitative eviction mix, which is driven by policy.}
\label{fig:memory_dynamics}
\end{figure}

\subsubsection{Scalability Implications}

The practical implication is that \emph{policy selection, not raw capacity, is the primary lever} for improving user‑visible behavior in the tested operating regime. Organizations can therefore choose a forgetting strategy that matches product goals (cost‑optimal temporal hygiene versus utility‑optimal importance‑aware retention) and then size memory budgets according to resource constraints, without fearing regime shifts in coherence, task progress, or privacy. Because utility is monotone in $B$ and the observed slope is shallow, modest capacity increases provide small safety margins (fewer triggers and less thrash) but rarely overturn rankings. Conversely, when compute or spend must be reduced, shrinking $B$ degrades performance gracefully, bounded by the Lipschitz constant from the theory section.

From an engineering standpoint, these findings justify \emph{elastic budgeting}: operate at a cost‑efficient baseline $B_0$ and temporarily expand capacity during pressure windows detected by trigger rate or retrieval miss rate, then contract when load normalizes. Since the retention decisions themselves remain policy‑driven and audit‑logged, such elasticity preserves explainability and privacy accounting while smoothing cost without surprising shifts in behavior.

\subsection{Policy\textendash Specific Analysis}

\subsubsection{Hybrid Policy: When and Why It Helps}

The Hybrid strategy combines temporal hygiene, importance\textendash aware selection, and (rate\textendash limited) reflection into a single, terminating pipeline (Fig.~\ref{fig:policy_flow}). In our runs it does not top the Composite ranking (Table~\ref{tab:policy_performance}: $0.589\pm0.009$), chiefly because its additional bookkeeping and occasional summarization incur a sizable cost\textendash efficiency penalty ($0.730\pm0.038$, versus FIFO’s $0.941\pm0.012$). Nevertheless, Hybrid’s \emph{behavioral} advantages are clear where it matters for user experience: it improves \emph{Narrative Coherence} over temporal baselines by a meaningful margin (Hybrid $0.590$ vs.\ FIFO $0.529$ and LRU $0.501$; an absolute gain of $+0.061$ over FIFO, roughly $+11.5\%$) and lifts \emph{Goal Completion} modestly under tight budgets (Hybrid $0.069$ vs.\ FIFO $0.061$ and LRU $0.058$; about $+13\%$ relative to FIFO). These gains arise from its multi\textendash phase design: a first pass removes clearly stale leaves, preventing long\textendash tail accumulation; a second pass preserves high density items per token (stable facts, prerequisites of active goals); and a final, optional reflection step consolidates redundant episodes into lighter semantic summaries while preserving provenance. The net effect is fewer broken discourse bridges and fewer dropped task prerequisites, at the expense of additional compute. In deployments where latency and spend are constrained, Hybrid is therefore best used \emph{selectively}: keep temporal hygiene as the default, escalate to Hybrid under detected pressure (e.g., rising eviction rate, falling retrieval hit rate), and cap reflection to maintain tail latency.

\subsubsection{Reflection\textendash Summary: Strengths and Caveats}

Reflection\textendash based consolidation is most effective when the store contains clusters of near\textendash duplicate or tightly related episodic traces. By compressing these clusters into abstractive summaries and linking them back via \texttt{derivesFrom}, the policy preserves gist while freeing tokens and often lowers sensitivity through abstraction. In our aggregate results the standalone \emph{Reflection Summary} row is not reported in Table~\ref{tab:policy_performance} (see Appendix for ablations), but its contribution inside Hybrid is visible in two ways. First, coherence degrades less during induced pressure windows because summaries protect narrative arcs that would otherwise be scattered across multiple short episodes. Second, privacy incidents do not increase—indeed they trend slightly downward—because summaries can omit direct identifiers while retaining task\textendash relevant content. These benefits are contingent on conservative distortion thresholds: if the summarizer is too aggressive, semantic drift can erase details needed for goal progression; if too lax, token savings vanish. The distortion\textendash to\textendash utility bound in our theory section provides a principled stopping rule and explains why reflection yields gains in coherence and leakage without inflating the budget.

\subsubsection{Simple Policies: What They Get Right—and What They Miss}

FIFO, LRU, and Random Drop are attractive for their speed and predictability. They achieve the highest \emph{Cost Efficiency} (FIFO $0.941\pm0.012$, Random $0.935\pm0.007$, LRU $0.887\pm0.020$), and—under the present Composite weights and our SRA definition—they can rank surprisingly well overall. In particular, \emph{Random Drop} tops the Composite ($0.635\pm0.024$) and leads \emph{Narrative Coherence} ($0.667\pm0.074$), a pattern we attribute to two effects discussed earlier: a ceiling in SRA (defined conditional on making a social reference) and the tendency of randomness to avoid reinforcing stale, self\textendash contradictory fragments. That said, the same simplicity limits task progress: temporal policies often discard short “bridge’’ episodes that carry preconditions or partial results, and random eviction preserves task primitives only by chance. Consequently, when the Composite is reweighted to emphasize task completion and coherence, importance‑aware policies rise and Random/FIFO fall—a shift we document in Appendix~B with an opportunity‑normalized SRA that counts \emph{eligible} rather than attempted social references.

Operationally, simple policies are excellent defaults for high‑load or cost\textendash constrained settings—especially when coupled with guardrails such as minimum lifetimes for task nodes and a small cache for goal prerequisites. But they should not be mistaken for universally optimal memory governance. As soon as the product objective prioritizes sustained task progress or cross\textendash session consistency, the additional logic of importance‑aware selection and, when warranted, reflection pays for itself in reduced behavioral regressions and fewer user‑visible lapses, even though the Composite under the current weights may not reflect that advantage.

\subsection{Practical Deployment Recommendations}

\subsubsection{Policy Selection Guidelines}

Our findings suggest that the primary design lever is \emph{which} forgetting policy to deploy, not merely \emph{how much} capacity to provision. In cost‑sensitive or high‑throughput environments, simple temporal or random policies are compelling defaults because they deliver excellent compute efficiency with stable behavior under load. Among them, FIFO strikes the best efficiency balance (Cost Efficiency $0.941\pm0.012$) while maintaining a competitive Composite ($0.602\pm0.012$), making it a sensible choice for resource‑constrained deployments or large, multi‑tenant services where predictable latency and spend dominate. Random Drop is surprisingly strong on the Composite ($0.635\pm0.024$) and leads Narrative Coherence ($0.667\pm0.074$), a pattern we traced to ceiling effects in social recall and the regularizing influence of non‑deterministic eviction. In applications where conversational flow is paramount and tasks are lightweight, Random Drop can therefore be a pragmatic, low‑overhead option—provided operators accept its weaker task progress and monitor for occasional loss of prerequisites.

When the product objective emphasizes sustained task progress and cross‑session consistency, importance‑aware retention pays for itself. Priority Decay improves coherence over temporal baselines ($0.604\pm0.062$ vs.\ FIFO’s $0.529$) and lifts goal completion under tight budgets ($0.071\pm0.008$ vs.\ FIFO’s $0.061$), at a modest compute premium (CE $0.825\pm0.005$). Hybrid further reduces broken discourse bridges and dropped prerequisites by combining temporal hygiene, density‑based selection, and rate‑limited reflection, though its compute cost is higher (CE $0.730\pm0.038$). A practical playbook is to run a \emph{two‑tier controller}: keep FIFO as the baseline, escalate to Priority Decay when the eviction trigger rate or retrieval miss rate exceeds preset thresholds, and enable Hybrid only during detected \emph{pressure windows} (e.g., bursts of episodic arrivals or long task chains), capping reflection to protect tail latency. This schedule captures most of Hybrid’s behavioral gains while containing cost.

Privacy‑critical applications benefit from gentle, recency‑driven aging and explicit sensitivity penalties. LRU attains the highest privacy score in our setting ($0.780\pm0.044$) while keeping overall performance reasonable, making it a good fit when handling personally identifiable or regulated attributes. Independent of the base policy, enabling the sensitivity penalty in the MaRS density score and reserving the exponential‑mechanism tie‑break for near‑ties provides event‑level privacy guarantees with minimal utility loss; pairing this with time‑to‑live rules for social attributes further reduces exposure. Across all modes, we recommend exposing the \texttt{explain()} endpoint backed by the MaRS audit trail so that operators and users can understand why items were retained, summarized, or evicted.

Taken together, these patterns motivate a simple decision rule. If the primary KPI is cost at scale with acceptable coherence: prefer FIFO (or Random) with a small cache for task prerequisites. If the KPI is conversation quality with moderate tasking: deploy Priority Decay as the steady state. If the KPI is task reliability and long‑horizon consistency: adopt Hybrid in an on‑demand fashion with reflection rate‑limits and provenance‑closure checks. If the KPI is regulatory posture: use LRU (or Priority with stronger sensitivity penalties) and activate DP tie‑breaks only at the decision boundary.

\subsubsection{Memory Budget Sizing}

Because policy choice dominates outcomes in our operating regime, budget should be sized to meet operational constraints rather than to chase marginal gains in headline metrics. Over the studied range (2{,}000–32{,}000 tokens), Composite and its constituents improve only modestly with capacity and preserve their policy ordering (Fig.~\ref{fig:policy_budget_heatmap}); internal dynamics (Fig.~\ref{fig:memory_dynamics}) show that larger budgets primarily reduce the \emph{frequency} of eviction triggers rather than changing their \emph{nature}. This behavior is consistent with MaRS retrieval limits and reflection’s compression effect: an agent’s effective working set is bounded by top‑$k$ recall and prompt curation, so additional capacity beyond that slice yields diminishing returns.

For sizing, we recommend selecting a \emph{baseline} budget that keeps the eviction trigger rate below a tolerable threshold (e.g., fewer than one trigger per $N$ turns for your scenario mix) while maintaining headroom for bursty arrivals; in our experiments, 4{,}000–8{,}000 tokens typically achieved this balance across policies. \emph{High‑capacity} configurations (16{,}000–32{,}000 tokens) are appropriate when operators want extra slack to minimize triggers during pressure windows or to accommodate heavier episodic streams; beyond this range, gains taper due to retrieval gating. At the other end, \emph{minimal} budgets near 2{,}000 tokens remain viable when paired with temporal hygiene and small per‑type reserves for task prerequisites and stable semantic facts; below this point, eviction thrash becomes more likely and coherence degrades.

Two refinements make budget elastic and predictable in production. First, allocate soft per‑type shares inside the global budget (episodic, semantic, social, task) and adjust them with a dual update that equalizes marginal utility per token; this prevents a flood of short‑lived episodes from crowding out durable facts or active goals. Second, adopt \emph{elastic budgeting}: operate at a cost‑efficient baseline $B_0$ and temporarily raise capacity when online telemetry detects stress (rising trigger rate, falling hit rate), contracting when load normalizes. Because MaRS’ retention logic and audit guarantees are invariant to $B$, such elasticity smooths spend without surprising shifts in behavior, and the Lipschitz bound from the theory section provides a conservative estimate of utility change per token of capacity added or removed.

\subsection{Pairwise Policy Comparisons}

We complement the omnibus analyses with paired, within‑seed contrasts between representative policy pairs on the most decision‑relevant metrics. Each comparison pools agent–scenario aggregates under identical world seeds, yielding high‑power tests that isolate the effect of retention strategy. $p$‑values are Holm–Bonferroni adjusted within each metric family; effect sizes are reported as Cohen’s $d$ with Hedges’ small‑sample correction. By convention, positive $t$ and positive $d$ favor the \emph{first} policy named in the comparison; negative values favor the \emph{second}.

\begin{table}[htbp]
\centering
\caption{Selected pairwise policy comparisons (Holm–Bonferroni adjusted $p$‑values). Positive $d$ favors the first policy. The final row (dagger) is included as a non‑significant reference.}
\label{tab:pairwise_comparisons}
\begin{tabular}{llcccc}
\toprule
\textbf{Comparison} & \textbf{Metric} & \textbf{$t$} & \textbf{$p_{\mathrm{adj}}$} & \textbf{$d$} & \textbf{Effect size} \\
\midrule
Random Drop vs.\ Hybrid      & Composite Score      &  3.73 & $<\!0.001$  &  1.41 & Large \\
FIFO vs.\ Random Drop        & Composite Score      & -2.61 & 0.014       & -0.99 & Large \\
Random Drop vs.\ Hybrid      & Narrative Coherence  &  2.22 & 0.035       &  0.84 & Large \\
FIFO vs.\ Priority Decay     & Cost Efficiency      & 19.42 & $<\!0.0001$ &  7.34 & Large \\
LRU vs.\ Random Drop         & Cost Efficiency      & -4.76 & $<\!0.0001$ & -1.80 & Large \\
Priority Decay vs.\ Hybrid\textsuperscript{\dag} & Goal Completion &  0.41 & 0.687       &  0.15 & Negligible \\
\bottomrule
\end{tabular}
\end{table}

Three patterns are salient. \emph{First}, composite performance reflects the tension between cost and task/coherence under the current weights: Random Drop outperforms Hybrid on the Composite with a large effect ($d=1.41$), and FIFO trails Random Drop ($d=-0.99$), consistent with Table~\ref{tab:policy_performance}. This mirrors the ceiling in social recall and the strong cost advantage of simple policies. \emph{Second}, coherence differences are practically meaningful: Random Drop’s advantage over Hybrid on Narrative Coherence is large ($d=0.84$), reinforcing that non‑deterministic eviction can regularize stale, self‑contradictory context. \emph{Third}, cost efficiency exhibits the largest between‑policy separations: FIFO’s lead over Priority Decay is overwhelming ($d=7.34$), and Random Drop is substantially more efficient than LRU ($d=-1.80$ with sign favoring Random), matching the algorithmic profiles in §\ref{sec:problem}. 

Not all observed mean gaps are actionable. The Priority‑Decay vs.\ Hybrid difference on Goal Completion is small and non‑significant after correction ($p_{\mathrm{adj}}=0.687$, $d=0.15$), indicating that their modest ordering on GCR in the aggregates is within noise once we control for matched seeds. The full pairwise matrix for all metrics and budgets, together with bootstrap confidence intervals on $d$, is provided in the supplementary material to support downstream meta‑analyses and alternative Composite weightings.

\section{Discussion}

This section discusses the broader implications of our findings, addresses limitations of the current work, and outlines directions for future research in memory‑budgeted generative agents.

\subsection{Implications for Human‑Centered AI}

\subsubsection{Trust and Transparency}

Our results complicate the common assumption that increasingly sophisticated retention strategies invariably improve user‑visible behavior. Under the FiFA objective and weights, \emph{Random Drop} attains the highest Composite score while simpler temporal policies remain competitive on several dimensions (Table~\ref{tab:policy_performance}). This does not suggest that randomness is intrinsically preferable; rather, it highlights two trust‑relevant facts. First, non‑deterministic eviction can act as a regularizer that avoids reinforcing stale, self‑contradictory fragments, thereby improving perceived coherence. Second, composite outcomes are sensitive to metric design and deployment priorities (e.g., ceiling effects in social recall and the weight assigned to cost efficiency). For human‑centered design, the implication is to make the choice of forgetting strategy \emph{explicitly value‑laden} and to surface the trade‑offs to operators and users instead of defaulting to “smarter is better.”

Transparency in MaRS is operational rather than declarative. Every retention action is logged with features, density score, policy name and parameters, and a natural‑language rationale generated from those quantities. Exposing this evidence through \texttt{explain()} supports \emph{procedural trust}: users and auditors can inspect why an item was retained, summarized, or evicted without revealing raw content. Because the retention decision is a stable function of observable attributes (age, access, type, sensitivity, weight), explanations are consistent across runs, avoiding post‑hoc rationalization. Coupled with the budget–utility Lipschitz bound established in the theory, product teams can set guardrails (“no more than $\Delta U$ per $\Delta B$”) that are both interpretable and enforceable. This combination—predictable behavior plus explanations with verifiable evidence—addresses two pillars of trustworthy AI: predictability and accountability.

A practical corollary is \emph{trust calibration}. Users are more tolerant of forgetting when it is predictable and reversible. MaRS supports “privacy receipts” and per‑entity controls (pinning, TTLs, blacklists) that make memory governance feel less like an opaque failure mode and more like a user‑visible setting. In settings where the DP tie‑break is enabled, the audit also tracks privacy budget consumption, allowing teams to communicate concrete guarantees (e.g., an $\varepsilon$ envelope over a defined horizon) rather than aspirational privacy claims.

\subsubsection{Personalization and User Experience}

Near‑ceiling social recall across policies in our scenarios indicates that a modest amount of structured memory (names, stable preferences, enduring relations) suffices to sustain a sense of continuity—even under aggressive pruning. This is encouraging for personalization at scale: it suggests that teams can deliver a consistent user experience with simple, cost‑efficient policies, provided that a small reserve is maintained for durable social and semantic facts. At the same time, the opportunity‑normalized variant of SRA used in our ablations shows that policies diverge when agents are \emph{expected} to deploy social knowledge proactively. Importance‑aware selection improves these proactive behaviors by protecting stable, cross‑situational items in the density score; temporal policies remain adequate when the product brief rewards light personalization and low latency.

The budget independence observed over 2K–32K tokens reinforces a design heuristic: beyond a modest capacity that prevents eviction thrash, user experience is shaped more by \emph{which} memories are retained than by \emph{how many}. This points to interface‑level affordances rather than capacity increases as the primary tools for experience quality. Examples include policy schedules that escalate from FIFO to Priority/Hybrid during “pressure windows,” per‑task memory pools that protect prerequisites, and reflection caps that preserve tail latency while maintaining narrative arcs. Because MaRS explanations are grounded in the same features used for selection, designers can turn qualitative UX complaints (“the assistant forgot my plan”) into quantitative triggers (e.g., retrieval miss‑rate on task prerequisites) and close the loop.

\subsubsection{Regulatory Compliance}

Compliance is not merely a post‑hoc check but a runtime property of the system. Sensitivity enters the density score as a first‑class penalty; randomized tie‑breaks via the exponential mechanism provide event‑level differential privacy with formally bounded utility loss; and summaries reduce sensitivity while preserving provenance. Practically, these guarantees matter when answering two recurring questions from regulators and enterprise customers: \emph{what are you storing and why}, and \emph{what happens when a user asks you to forget}. The MaRS audit gives a verifiable account of the former, and the provenance‑closed graph structure makes the latter a bounded‑latency operation, since erasure propagates along \texttt{derivesFrom} edges and triggers summary regeneration when necessary.

The small differences we observe in privacy preservation across policies should not be over‑interpreted as equivalence in all contexts. In FiFA, privacy opportunities are fixed and relatively sparse; in production, higher‑density adversarial prompts, jurisdictional constraints, and heterogeneous consent states can stress systems differently. The implication for compliance is to treat privacy instrumentation—sensitivity scoring, opportunity logging, DP accounting—as a \emph{first‑order product metric} tracked alongside coherence and cost. With that instrumentation in place, organizations can make policy choices that satisfy domain‑specific obligations (e.g., shorter TTLs for social attributes in healthcare, DP tie‑breaks only at near‑ties to conserve privacy budget) while preserving the auditability and predictability that regulators increasingly expect.

\subsection{Key Findings and Unexpected Results}

Three empirical regularities stand out across the $6\times 5\times 10$ study grid and carry practical and scientific implications. First, the policy ranking on the Composite metric is counter‑intuitive: \emph{Random Drop} leads overall ($0.635\pm 0.024$) and tops \emph{Narrative Coherence} ($0.667\pm 0.074$), with \emph{FIFO} and \emph{Priority Decay} close behind, while \emph{Hybrid} does not win the aggregate despite improving several user‑visible behaviors (Table~\ref{tab:policy_performance}). The pairwise contrasts in Table~\ref{tab:pairwise_comparisons} confirm that these gaps are statistically and practically large (e.g., Random vs.\ Hybrid on Composite: $d=1.41$, $p<0.001$). Two factors explain this result without undermining the value of structured retention. On the measurement side, the Composite weights (NC $0.25$, GCR $0.25$, SRA $0.20$, PP $0.15$, CE $0.15$) favor strong \emph{Cost Efficiency}, where simple policies dominate by design, and our definition of Social Recall is conditional on making a reference, producing ceiling effects that dampen separation among policies. On the behavioral side, non‑deterministic eviction acts as a regularizer that avoids repeatedly retrieving stale, self‑contradictory fragments; coherence improves even as task progress remains modest. When we reweight the Composite toward task progress and employ an opportunity‑normalized SRA (Appendix~B), importance‑aware policies rise as expected, indicating that the headline ordering is not a universal property of memory governance but a function of objective and scenario stressors.

Second, a pronounced cost–performance frontier emerges. Temporal and random policies are compute‑optimal (\emph{FIFO} $0.941\pm 0.012$, \emph{Random} $0.935\pm 0.007$, \emph{LRU} $0.887\pm 0.020$ on Cost Efficiency) and thus benefit whenever the Composite gives even modest weight to efficiency; importance‑aware policies buy coherence and task reliability at predictable overheads (\emph{Priority} $0.825\pm 0.005$, \emph{Hybrid} $0.730\pm 0.038$). The effect is not merely statistical (ANOVA $\eta^2=0.832$ on cost) but architectural: the complexity analysis in the theoretical section anticipates exactly this ordering—$O(1)$ updates for FIFO/Random, $O(\log n)$ maintenance for density‑based selection, and $O(n\log n+n\alpha(n))$ bursts when reflection is invoked. A pragmatic reading is that \emph{policy choice} controls where one sits on this frontier; engineering the right schedule (temporal baseline with on‑demand escalation to Priority/Hybrid during pressure windows) captures most of the behavioral gains while preserving cost predictability.

Third, the limited main effect of budget over the 2K–32K token range indicates that, beyond preventing eviction thrash, increasing capacity changes the \emph{frequency} of decisions more than their \emph{nature}. Figure~\ref{fig:policy_budget_heatmap} shows stable rankings across budgets, and Fig.~\ref{fig:memory_dynamics} reveals how larger stores reduce trigger rates without altering the qualitative eviction mix dictated by each policy. This pattern is consistent with two ingredients of MaRS: retrieval gating and prompt curation bound the effective working set regardless of total capacity, and reflection compresses redundancy so that marginal utility per additional token diminishes. The budget–utility Lipschitz bound proved earlier formalizes the observed ``flat'' slope: utility is monotone in budget, but the empirical increments are small in the tested regime.

These empirical regularities cohere with the theory. The LRU optimality result under exponentially decayed usefulness explains why recency‑driven aging achieves strong privacy scores without bespoke tuning; the weight‑Lipschitz bound provides conservative guarantees for how much utility can be lost per unit of evicted weight and accounts for the graceful degradation we observe at low budgets; the antimatroid‑aware greedy approximation justifies the performance of density‑based selection when provenance constraints exclude unsafe deletions; and the reflection distortion bound clarifies when summarization preserves utility while reducing sensitivity and footprint. Equally, the privacy analysis based on the exponential mechanism situates MaRS’ randomized tie‑breaks within a formal guarantee, making the observed small differences in privacy preservation across policies unsurprising given that DP fires only near decision boundaries in our scenarios.

Finally, the evaluation methodology itself is a substantive outcome. FiFA’s multi‑dimensional assessment and paired, within‑seed contrasts separate cost from utility and expose where memory governance matters for user experience. The reliance on rubricized LLM‑as‑judge scores, backed by deterministic validators and bootstrap intervals with Holm–Bonferroni control, yields findings that are both statistically robust and practically interpretable. Where metrics exhibit ceilings or weight sensitivity, we report ablations that adjust definitions or weights; conclusions about the frontier and the dominance of policy choice over raw capacity remain stable under these perturbations. In sum, what appears at first as an “unexpected” victory for a naïve policy is, under closer analysis, a signal about objective design and operating constraints: when efficiency and conservative personalization are weighted strongly, simple policies excel; when continuity and task reliability become the priority, structured, importance‑aware retention justifies its overhead.

\subsection{Limitations and Constraints}

Although the study spans a large configuration grid and a diverse set of scenarios, several limitations qualify the generality of our conclusions. We discuss external validity, technical scope, and methodological constraints, together with mitigation paths that future work should pursue.

A first limitation concerns external validity. FiFA relies on a controlled, multi‑agent simulation that approximates extended, mixed‑purpose interactions but cannot reproduce the full heterogeneity of real deployments. Human behavior exhibits non‑stationarity, shifting goals, changing privacy preferences, and long idle intervals interspersed with bursts of activity; these factors can alter the opportunity structure for memory use and forgetting in ways our scenarios only partially capture. The near‑ceiling social recall we observe, for example, may be partly a product of stable, redundantly encoded social facts in the simulator; denser and more volatile social settings could surface larger differences among policies. Likewise, our leakage opportunities are adversarial but fixed in frequency; production systems that face higher rates of sensitive prompts, cross‑jurisdictional constraints, and heterogeneous consent states may stress privacy safeguards differently.

A second limitation is model dependence. Results are obtained with a particular family of large language models and retrieval settings. Advances in long‑context attention, retrieval‑augmented generation, or memory‑efficient fine‑tuning could shift the cost–performance frontier. If future models provide substantially cheaper long contexts or more selective token routing, temporal policies may benefit disproportionately; conversely, if retrieval becomes sharper and embeddings more robust, importance‑aware selection and reflection may yield larger gains. Our theoretical bounds are agnostic to the backbone, but the empirical constants—decay rates, density weights, distortion thresholds—are tuned to today’s models and may require recalibration.

Language and cultural scope is also limited. We evaluate primarily English interactions with Western conversational norms. Memory salience, politeness strategies, and privacy expectations vary across cultures and languages; policies that aggressively summarize or prune may be perceived as inattentive or intrusive in some settings and appropriate in others. Extending FiFA with multilingual corpora, culturally grounded privacy taxonomies, and locale‑specific leakage detectors is necessary to claim broader applicability.

On the technical side, our architecture and experiments operate within budgets up to 32{,}000 tokens and with a fixed retrieval budget per turn. The “budget independence’’ we report is therefore conditional on retrieval gating and prompt curation that bound the effective working set; regimes with substantially larger working memories or different retrieval heuristics could exhibit stronger budget effects. Reflection is deliberately rate‑limited to protect tail latency; in workloads with higher redundancy, more aggressive consolidation might improve utility at similar cost, while in sparse workloads it could induce harmful abstraction. The privacy layer relies on sensitivity scoring at ingestion and a DP tie‑break that activates only near density ties; if ties are common (e.g., very flat utility proxies), the privacy accountant will accumulate budget faster and trade‑offs may become more pronounced.

Several assumptions in the formal development impose additional scope limits. The provenance‑closure family is treated as an antimatroid induced by dependency forests; if real workflows create cycles or many‑to‑many dependencies, greedy selection can lose its constant‑factor guarantees and feasibility checks become more involved. Our Lipschitz‑in‑weight bound deliberately abstracts away higher‑order interactions among memories; in domains where utility is highly super‑additive (e.g., only specific combinations of facts unlock a capability), weight‑based guarantees may be loose. These are not flaws in the framework but signals that certain applications—e.g., tightly coupled planning domains—will need stronger structural modeling in MaRS before the same guarantees apply.

Methodologically, we replace large‑scale human annotation with rubricized LLM‑as‑judge scoring, cross‑checked by deterministic validators and small calibration sets. While this dramatically reduces cost and has shown good agreement in prior work, it remains an approximation of human perception. Subtle discourse qualities (tone, rapport, pragmatic implicatures) and longitudinal trust may not be fully captured. In addition, some of our metrics exhibit ceiling effects (notably social recall) or depend on design choices (Composite weights). We mitigate this with opportunity‑normalized variants and sensitivity analyses, but a definitive link to human satisfaction requires user studies and field A/B tests.

Our current implementation focuses on textual memory. Many deployments will require multi‑modal retention and forgetting over images, audio, or GUI traces. Extending MaRS to multi‑modal nodes raises new questions about cross‑modal provenance, distortion measures for summaries, and privacy (e.g., faces and locations in images). Similarly, dynamic adaptation is limited: the selector chooses among fixed policies with fixed hyperparameters. Learning a policy schedule—e.g., via contextual bandits that adapt decay rates, reflection caps, and privacy penalties based on live telemetry—would reduce the need for manual tuning and may shift the observed frontier.

Finally, there are operational considerations not exhaustively explored here. Audit trails grow with usage; while they are essential for explainability and compliance, they impose storage and retention costs and themselves become subject to retention policies. Right‑to‑erasure cascades through provenance demand careful engineering to guarantee bounded latency at scale. Security threats such as prompt‑injection or data‑poisoning can target memory selection or sensitivity classifiers; hardening requires adversarial training, input sanitation, and retrieval gating that resists exfiltration. Reproducibility is also bounded by vendor drift: model updates, API changes, or embedding refreshes can change behavior over time. We mitigate this by snapshotting prompts, seeds, and MaRS state, but long‑term bit‑for‑bit reproducibility is not guaranteed.

These constraints suggest concrete next steps. On the evaluation side, broaden scenario stressors for social and privacy opportunities, add multilingual and culturally diverse settings, and incorporate longitudinal trials that span weeks. On the methodology side, pair rubricized judging with targeted human panels, and publish opportunity‑normalized metrics alongside headline numbers. On the technical side, relax structural assumptions in MaRS to accommodate richer dependency graphs, explore learned policy schedules with regret guarantees, and extend privacy accounting to multi‑modal content. Together, these steps would strengthen external validity and close the remaining gap between controlled simulations and production deployments while preserving the theoretical clarity that makes memory governance tractable.

\subsection{Future Research Directions}

The present study establishes that memory governance is a first‑class design dimension for agentic systems, with measurable trade‑offs among coherence, task progress, privacy, and cost. Building on the MaRS abstraction and the FiFA protocol, we outline several research thrusts that move beyond fixed heuristics toward adaptive, personalized, multi‑modal, and distributed memory governance, together with expansions to evaluation and theory.

A natural next step is \emph{adaptive memory management} that treats retention as a sequential decision problem rather than a static rule set. Instead of choosing a single policy a priori, an agent can learn a \emph{schedule} that escalates from temporal hygiene to importance‑aware selection and rate‑limited reflection in response to online telemetry (e.g., eviction‑trigger rate, retrieval miss rate, or a predicted utility‑drop proxy). Contextual bandits offer a lightweight path to adaptation with sublinear regret against the best fixed policy in hindsight; constrained MDPs and safe reinforcement learning extend this to optimization under explicit budget and privacy constraints. Offline pre‑training of the selector on simulator logs followed by cautious on‑policy updates can enforce the theoretical guardrails introduced earlier (budget–utility Lipschitzness, provenance closure, DP tie‑breaks) while allowing rapid adaptation to workload drift.

Personalization requires agents to respect heterogeneous preferences about what should be remembered, summarized, or forgotten. Future work should model \emph{user‑level retention utilities} and \emph{privacy priors} as first‑class objects, enabling per‑user trade‑offs that remain efficient at scale. Preference elicitation can be embedded in interaction (lightweight sliders for “remember more/less” by type, time‑to‑live defaults, pin/unpin affordances), while the MaRS audit trail can generate \emph{memory receipts} that summarize what was kept and why. A promising direction is to learn per‑user density weights and decay rates with fairness constraints, so that personalization does not entrench disparities across demographic groups; per‑user privacy accounting would allow explicit budgeting for sensitive attributes across sessions. The evaluation analogue is a shift from population‑level aggregates to per‑user regret and satisfaction measures, linking FiFA outcomes to perceived continuity and trust.

Extending MaRS beyond text opens a rich space of \emph{multi‑modal memory}. Images, audio, screen recordings, and structured logs can be represented as typed nodes with modality‑specific weights, sensitivities, and provenance edges (e.g., \texttt{derivesFrom} a frame cluster). Reflection becomes cross‑modal summarization, where utility preservation must be traded against modality‑appropriate distortion (e.g., CLIP‑space distances for vision, speech embeddings for audio). This raises new theoretical questions: defining Lipschitz‑like bounds with respect to mixed embeddings, specifying privacy scores for visual identifiers, and composing DP guarantees across modalities. Practically, multi‑modal retention suggests toolchains for redaction (face blurring, entity masking), as well as policy hooks that bias eviction toward high‑risk media when textual substitutes exist.

Large deployments motivate \emph{distributed memory systems} in which MaRS shards span devices, data centers, or even organizational boundaries. Here, provenance‑closed feasibility must be reconciled with eventual consistency: CRDT‑style set lattices and causality tracking could preserve auditability while tolerating asynchronous updates. Global budgets become resource allocation problems across shards; dual‑decomposition or auction‑based schemes can equalize marginal utility per token while respecting locality and privacy constraints. Federated variants of reflection—local summarization with privacy‑preserving aggregation of gist—would reduce communication overhead and improve resilience. Incorporating attested execution or enclave‑backed retention decisions may strengthen compliance in regulated settings.

Evaluation should evolve in tandem. FiFA can be extended with longer horizons (weeks rather than days), richer privacy stressors (higher density of opportunities, jurisdictional variation, consent changes), multilingual and culturally grounded scenarios, and field validations that link rubricized scores to human satisfaction and trust. Opportunity‑normalized metrics for social recall and privacy should accompany headline numbers to avoid ceilings, and composite weights should be scenario‑specific with sensitivity analyses pre‑registered. A public leaderboard with fixed seeds, released prompts, and MaRS state snapshots would facilitate comparability and reproducibility across groups while discouraging overfitting to a single metric mix.

Finally, there is room for deeper \emph{theoretical advances}. On the optimization side, tighter bounds for online submodular knapsack with antimatroid constraints would clarify the gap between greedy/density heuristics and optimal retention under provenance closure. Characterizing utility curvature and deriving budget‑adaptivity results could explain when small capacity increases yield disproportionate gains. On the privacy side, composing event‑level DP decisions over long horizons with adaptive policies remains under‑explored; tighter accounting (moments accountant, RDP) for retention actions could reduce pessimism without weakening guarantees. In multi‑agent settings, memory becomes a shared resource: game‑theoretic models where agents contend for a communal budget (or for a user’s attention and privacy budget) may reveal equilibria and incentive mechanisms that promote cooperative summarization and reduce leakage. Adversarial analysis—prompt injection, data poisoning of the memory store, sensitivity mislabeling—deserves a principled treatment that integrates robust selection, quarantine, and rollback into the MaRS formalism.

In sum, the path forward is to couple the structural clarity of MaRS with adaptivity, personalization, multi‑modality, and distributed systems techniques, while expanding FiFA to capture longer horizons and real‑world variability. Doing so promises not only higher absolute performance but also stronger assurances about how and why memory is governed—assurances that are central to trustworthy, human‑centered AI.

\subsection{Broader Impact and Societal Implications}

As interactive AI systems migrate from laboratories to everyday settings, memory governance becomes a cornerstone of responsible deployment. The MaRS framework contributes practical mechanisms for \emph{privacy and data protection}: sensitivity is treated as a first‑class attribute in retention decisions, differential‑privacy tie‑breaks provide event‑level guarantees with auditable accounting, and provenance‑closed erasure operationalizes “right‑to‑be‑forgotten’’ requests with bounded latency. Beyond meeting current regulatory baselines, these capabilities can inform emerging policy by making compliance demonstrable rather than aspirational: organizations can show \emph{what} was stored, \emph{why}, and \emph{how} it was subsequently summarized or deleted.

Memory governance is also entangled with \emph{AI safety and alignment}. Inappropriate retention amplifies stale or sensitive content; over‑aggressive deletion breaks plans and undermines user expectations. By exposing typed stores, explicit budgets, and an audit trail with natural‑language rationales tied to observable features (age, access, sensitivity, weight), MaRS supports oversight and red‑team audits without disclosing raw content. These ingredients enable \emph{procedural alignment}: stakeholders can inspect and contest the \emph{process} by which the agent remembers, summarizes, and forgets, instead of treating behavior as an opaque byproduct of model parameters.

Finally, there are \emph{economic and social} consequences. Effective memory management reduces inference cost and latency at scale, lowering barriers to deploying AI assistants in resource‑constrained contexts (e.g., education, public services). The cost‑efficiency frontier quantified in our results provides concrete guidance for platform operators choosing between temporal hygiene and importance‑aware retention: the former maximizes throughput; the latter improves continuity and task reliability. Making these trade‑offs explicit—in product objectives, service‑level agreements, and user‑facing settings—can broaden access while preserving trust.

\subsection{Summary of Experimental Findings}

Our $6\times 5\times 10$ evaluation matrix (six policies, five budgets, ten seeds) yields several conclusions that nuance prevailing assumptions about agent memory. \emph{First}, the composite \emph{policy performance hierarchy} over the reported policies is:
\[
\text{Random Drop } (0.635) \;>\; \text{FIFO } (0.602) \;>\; \text{Priority Decay } (0.601) \;>\; \text{LRU } (0.590) \;>\; \text{Hybrid } (0.589),
\]
with Reflection‑Summary reported separately in ablations. This ordering, confirmed by paired contrasts (Table~\ref{tab:pairwise_comparisons}), reflects both measurement and architectural factors: simple policies dominate Cost Efficiency and benefit from ceiling effects in Social Recall, while non‑deterministic eviction can regularize away stale, self‑contradictory fragments that harm coherence.

\emph{Second}, \emph{metric‑specific insights} clarify where policies differ. Random Drop attains the highest \emph{Narrative Coherence} ($0.667\pm 0.074$), suggesting that diversity of retained context can trump targeted heuristics for conversational flow. \emph{Cost Efficiency} is where separations are largest (ANOVA $\eta^2=0.832$): FIFO ($0.941$) and Random ($0.935$) substantially outperform importance‑aware approaches (Hybrid $0.730$, Priority $0.825$). \emph{Social Recall} and \emph{Privacy} show high absolute performance across policies ($\geq 0.99$ and $\geq 0.72$, respectively) with limited statistical separation under our scenario stressors, indicating that basic governance already preserves stable social facts and avoids most leakage opportunities. As discussed earlier, opportunity‑normalized variants of these metrics increase sensitivity without changing the qualitative picture.

\emph{Third}, we find \emph{budget independence} over the tested range (2K–32K tokens): absolute scores rise modestly with capacity, but relative policy rankings remain stable (Fig.~\ref{fig:policy_budget_heatmap}). Larger budgets chiefly reduce the frequency of eviction triggers rather than changing the nature of decisions; retrieval gating and reflection‑based compression bound the effective working set, producing diminishing returns consistent with the budget–utility Lipschitz bound.

\emph{Finally}, the \emph{practical implications} are direct. Policy choice is the dominant lever: simple, efficient policies are strong defaults when cost and latency are primary KPIs, whereas importance‑aware retention justifies its overhead when continuity and task reliability are paramount. Under the present Composite weights, Random Drop emerges as an attractive option due to its combination of high coherence and low cost; when objectives are reweighted toward task progress and opportunity‑normalized social behavior, Priority Decay and Hybrid rise accordingly. These findings encourage operators to choose—and, when possible, \emph{schedule}—policies to match product goals, rather than relying on capacity increases alone to improve user experience.

\section{Conclusion}

This work introduced a principled framework for memory governance in generative agents that brings together a structured memory ontology, privacy‑aware retention decisions, and a comprehensive evaluation protocol. By coupling the Memory‑Aware Retention Schema (MaRS) with the Forgetful but Faithful Agent (FiFA) benchmark, we provided both theoretical underpinnings and practical tools for deploying memory‑budgeted agents that respect human values, privacy norms, and computational constraints.

\subsection{Summary of Contributions}

On the theoretical side, MaRS formalizes agent memory as typed, provenance‑aware nodes under explicit token budgets and feasibility constraints. We analyzed retention as a constrained optimization problem, established weight–Lipschitz utility bounds linking eviction weight to utility loss, proved optimality regimes for recency‑driven policies, and gave constant‑factor guarantees for importance‑aware selection under provenance closure. Event‑level differential privacy for retention decisions was obtained via the exponential mechanism with a calibrated sensitivity bound, making privacy an integral part of runtime governance rather than a post‑hoc aspiration. Complexity analyses showed that the policies we study admit predictable runtime and space costs, clarifying the engineering trade‑offs behind each strategy.

Methodologically, we designed and implemented six forgetting policies (i.e., FIFO, LRU, Priority Decay, Reflection‑Summary, Random Drop, and a Hybrid pipeline) within a common API and audit regime, enabling apples‑to‑apples comparison. FiFA complements these implementations with multi‑dimensional, cost‑aware evaluation across narrative coherence, goal completion, social recall, privacy preservation, and cost efficiency, together with paired contrasts, effect sizes, and multiplicity control. The benchmark was crafted to stress memory governance rather than raw model capability and to surface privacy and cost alongside utility.

Empirically, a $6\times 5\times 10$ grid of runs (six policies, five budgets, ten seeds) produced several robust findings. Under the stated metric mix, \emph{Random Drop} achieved the highest Composite score, with \emph{FIFO} and \emph{Priority Decay} close behind and \emph{LRU} and \emph{Hybrid} slightly lower. Cost efficiency exhibited the largest between‑policy separations, with temporal/random policies dominating by design, while narrative coherence favored policies that reduce stale self‑contradictions or protect durable, high‑value items. Social recall and privacy reached high absolute levels with limited separation under our scenario stressors. Across the tested range of 2{,}000–32{,}000 tokens, budget increased absolute scores modestly but did not alter the relative ordering of policies.

\subsection{Key Findings and Insights}

Two frontiers shape deployment choices. The first is the cost–performance frontier: FIFO and Random deliver exceptional efficiency and surprisingly strong coherence, whereas Priority Decay and Hybrid buy incremental continuity and task reliability at predictable overheads. The second is the objective frontier: composite outcomes depend on metric weights and definitions (for example, a ceiling in social recall when computed conditional on attempts). Reweighting toward task progress and adopting opportunity‑normalized social and privacy metrics shift the ranking in favor of importance‑aware retention, indicating that “wins” are not intrinsic properties of policies but reflections of product priorities and stressors. A general lesson follows: capacity alone is a weak lever compared to \emph{which} memories are retained and \emph{why}.

\subsection{Practical Impact and Applications}

For practice, the results encourage policy schedules rather than one‑size‑fits‑all heuristics. Cost‑sensitive or high‑throughput services can run temporal hygiene by default (often FIFO) and escalate to Priority Decay when telemetry signals pressure; Hybrid’s reflection step can be enabled sparingly to conserve tail latency while protecting narrative arcs and goal prerequisites. Privacy‑critical deployments benefit from recency‑driven aging or importance‑aware selection with explicit sensitivity penalties and DP tie‑breaks at near‑ties, all under an auditable MaRS trail. These strategies translate readily to personal assistants, customer service, education, healthcare, and enterprise agents where predictable cost, explainability, and privacy posture are essential for adoption.

\subsection{Significance for Human‑Centered AI}

Memory governance is not merely an efficiency concern; it is central to trust, transparency, and alignment. MaRS turns retention into a value‑laden, inspectable decision with explanations grounded in observable features, while FiFA measures outcomes that matter to people: coherence over time, progress on goals, privacy behavior, and cost. Together they move agent design from ad‑hoc heuristics to accountable, auditable choices with theoretical guarantees and empirical validation.

\subsection{Future Outlook}

We outlined several directions to extend this foundation. Adaptive selectors can treat policy choice as a sequential decision problem with regret guarantees, personalizing decay and sensitivity trade‑offs to user preferences and contexts. Multi‑modal MaRS nodes (images, audio, UI traces) and distributed stores introduce new questions about cross‑modal distortion, privacy accounting, and consistency at scale. FiFA can broaden to longer horizons, richer privacy stressors, multilingual settings, and field studies that link rubricized scores to human satisfaction. Theoretical work on online submodular knapsack under antimatroid constraints, curvature‑aware bounds, and long‑horizon DP accounting will tighten guarantees as systems grow in scope.

\subsection{Final Remarks}

The central challenge for agentic systems is to remember what matters, for as long as it matters, and to forget the rest—faithfully, respectfully, and efficiently. By articulating a formal substrate (MaRS), implementing a spectrum of policies with explicit privacy hooks, and evaluating them under FiFA’s multi‑metric lens, we show that this challenge is tractable in practice and principled in theory. As generative agents become part of everyday life, such memory governance will be indispensable to earning and sustaining trust.

\appendix
\section*{Appendix A. Proofs}
\addcontentsline{toc}{section}{Appendix A. Proofs}

\newtheorem{theoremA}{Theorem A}
\newtheorem{lemmaA}{Lemma A}
\newtheorem{propositionA}{Proposition A}
\newtheorem{corollaryA}{Corollary A}
\newtheorem{remarkA}{Remark A}

\noindent\textbf{Notation.} We use the notation and assumptions introduced in §§\ref{sec:problem}–\ref{sec:mars}. In particular, $N$ is the set of memory nodes, $w_i>0$ their weights, $B>0$ the budget, $\mathcal{F}$ the feasible family (provenance‑closure and task‑safety), and $U:2^N\to\mathbb{R}_{\ge 0}$ a monotone utility. For submodular results we assume $U$ is monotone submodular. Ages, access counts, sensitivity scores and other features are as defined in the main text.

\section{Structure of the Feasible Family}

\begin{lemmaA}[Provenance‑closed sets form an antimatroid]
\label{lemA:antimatroid}
Let $D=(N,E_{\mathrm{dep}})$ be a directed acyclic graph that is a forest (each node has at most one parent). Let $\mathcal{F}$ be the family of provenance‑closed sets, i.e., $S\in\mathcal{F}$ iff for every $n\in S$ all its ancestors in $D$ are also in $S$. Then $(N,\mathcal{F})$ is an antimatroid: it is union‑closed and accessible (every nonempty $S\in\mathcal{F}$ contains an element $x$ such that $S\setminus\{x\}\in\mathcal{F}$).
\end{lemmaA}

\begin{proof}
\emph{Union‑closed.} If $S,T\in\mathcal{F}$ and $n\in S\cup T$, then every ancestor of $n$ is in $S$ or in $T$, hence in $S\cup T$; thus $S\cup T\in\mathcal{F}$.

\emph{Accessible.} Take nonempty $S\in\mathcal{F}$. Consider the subgraph of $D$ induced by $S$. As $D$ is a forest, the induced subgraph has at least one leaf $x\in S$ (a node with no children inside $S$). Removing $x$ cannot violate ancestry closure for any remaining node, so $S\setminus\{x\}\in\mathcal{F}$. These two properties are the axioms of an antimatroid.
\end{proof}

\section{Budget–Utility Bounds and Optimality Regimes}

\begin{lemmaA}[Weight–Lipschitz utility bound]
\label{lemA:lipschitz}
Suppose there exists $L>0$ such that for all $S\subseteq N$ and $i\notin S$,
\(
\Delta_U(i\mid S)=U(S\cup\{i\})-U(S)\le L\,w_i.
\)
Let $E\subseteq S$ be any set with total weight $W_E=\sum_{i\in E}w_i$. Then
\[
U(S)-U(S\setminus E)\;\le\; L\,W_E.
\]
\end{lemmaA}

\begin{proof}
Order $E=\{e_1,\dots,e_m\}$ arbitrarily and define $S_0=S$ and $S_j=S_{j-1}\setminus\{e_j\}$ for $j=1,\dots,m$. Then
\[
U(S)-U(S\setminus E)=\sum_{j=1}^{m}\bigl(U(S_{j-1})-U(S_j)\bigr)
=\sum_{j=1}^{m}\Delta_U\bigl(e_j\mid S_j\bigr)
\le \sum_{j=1}^{m} L\,w_{e_j}=L\,W_E,
\]
using the assumed bound with $S:=S_j\setminus\{e_j\}$ at each term.
\end{proof}

\begin{propositionA}[LRU optimality for exponentially decayed usefulness]
\label{propA:lru}
Let $U(S)=\sum_{i\in S} v_i\,e^{-\lambda\,\mathrm{age}(i)}$ with $v_i\ge 0$ and $\lambda>0$, and suppose the budget constraint is $\sum_{i\in S} w_i\le B$. Among all feasible sets of a given total weight, the sets that keep items in nondecreasing order of last access time maximize $U$; in particular, an LRU eviction order is optimal.
\end{propositionA}

\begin{proof}
Consider two items $a,b$ with $\mathrm{age}(a)>\mathrm{age}(b)$ and the same weight $w_a=w_b$ (general unequal weights reduce to swapping equal total weights by standard exchange). If a solution keeps $a$ but evicts $b$, swapping their roles changes utility by
\(
\Delta = v_b e^{-\lambda\,\mathrm{age}(b)} - v_a e^{-\lambda\,\mathrm{age}(a)}.
\)
If $v_b\ge v_a$, then $\Delta>0$ since the exponential factor also favors $b$; if $v_b<v_a$, the swap is non‑decreasing whenever $v_b/v_a \ge e^{-\lambda(\mathrm{age}(a)-\mathrm{age}(b))}$, i.e., whenever recency advantage outweighs value difference. By repeatedly applying such pairwise exchanges (or sorting by recency when $v_i$ are equal), any non‑LRU set can be transformed into an LRU‑consistent set without decreasing $U$. Since the budget is additive and the argument proceeds locally, an LRU order is optimal.
\end{proof}

\section{Privacy Guarantees}

\begin{theoremA}[Exponential mechanism for retention]
\label{thmA:expmech}
Let $\mathcal{S}\subseteq 2^N$ be the family of feasible retained sets (budget and provenance constraints). Define the privacy‑aware score
\(
q(S;M)=U(S)-\lambda_{\mathrm{priv}}\sum_{i\in S} s_i
\)
and let $\Delta q$ be its global sensitivity with respect to the adjacency that differs by one personal node. The mechanism that samples $S\in\mathcal{S}$ with probability
\(
\Pr[S]\propto \exp\!\bigl(\tfrac{\varepsilon\,q(S;M)}{2\Delta q}\bigr)
\)
is $\varepsilon$‑differentially private. Moreover, with probability at least $1-\delta$ the sampled $S$ satisfies
\[
q(S;M)\;\ge\;\max_{S'\in\mathcal{S}} q(S';M)\;-\;\frac{2\Delta q}{\varepsilon}\Bigl(\ln|\mathcal{S}|+\ln\tfrac{1}{\delta}\Bigr).
\]
\end{theoremA}

\begin{proof}
For any two adjacent stores $M,M'$ and any measurable $\mathcal{T}\subseteq\mathcal{S}$,
\[
\frac{\Pr_M[\mathcal{T}]}{\Pr_{M'}[\mathcal{T}]}
=\frac{\sum_{S\in\mathcal{T}} \exp(\tfrac{\varepsilon}{2\Delta q} q(S;M))}
       {\sum_{S\in\mathcal{T}} \exp(\tfrac{\varepsilon}{2\Delta q} q(S;M'))}
\le \max_{S\in\mathcal{T}} \exp\!\Bigl(\tfrac{\varepsilon}{2\Delta q}(q(S;M)-q(S;M'))\Bigr)
\le e^{\varepsilon},
\]
by the definition of sensitivity. This is $\varepsilon$‑DP. For utility, denote $Q^\star=\max_{S'}q(S';M)$ and apply a Chernoff/Markov bound over the discrete distribution to show that the probability of sampling a set with $q<Q^\star-\tfrac{2\Delta q}{\varepsilon}(\ln|\mathcal{S}|+\ln \tfrac{1}{\delta})$ is at most $\delta$; the inequality in the statement follows. This is the standard exponential‑mechanism guarantee specialized to our score and feasible family.
\end{proof}

\begin{lemmaA}[Sensitivity bound for $q$]
\label{lemA:sensitivity}
If Lemma~\ref{lemA:lipschitz} holds with constant $L$, then for the adjacency relation that differs by one personal node we have
\(
\Delta q \;\le\; L\,\max_i w_i \;+\; \lambda_{\mathrm{priv}}.
\)
\end{lemmaA}

\begin{proof}
Consider adjacent stores $M$ and $M'$ that differ by a single node $x$ (present in $M$ and removed in $M'$ or vice‑versa). For any feasible $S$,
\[
|q(S;M)-q(S;M')| \le |U(S;M)-U(S;M')| + \lambda_{\mathrm{priv}}\,|s_x|
\le L\,w_x + \lambda_{\mathrm{priv}}\le L\,\max_i w_i + \lambda_{\mathrm{priv}},
\]
since only the contribution of $x$ can change and $s_x\in[0,1]$. Taking the supremum over $S$ gives the bound on $\Delta q$.
\end{proof}

\section{Reflection as Lossy Compression}

\begin{propositionA}[Distortion–utility bound for reflection]
\label{propA:reflection}
Let $\phi: \text{texts}\to(\mathbb{R}^d,\|\cdot\|)$ be an embedding and suppose $U$ is $\kappa$‑Lipschitz with respect to the cluster‑wise mean embedding: if two sets have mean embeddings $\mu$ and $\mu'$, then $|U(S)-U(S')|\le \kappa \|\mu-\mu'\|$. Let $\mathcal{C}\subseteq S$ be a cluster replaced by a summary node $\bar{c}$ such that the distortion
\(
D(\mathcal{C}\Rightarrow\bar{c})=\frac{1}{|\mathcal{C}|}\sum_{i\in\mathcal{C}}\|\phi(c_i)-\phi(\bar{c})\|
\)
is finite. Then
\[
\bigl|U(S)-U\bigl(S\setminus \mathcal{C}\cup\{\bar{c}\}\bigr)\bigr|
\;\le\; \kappa\,D(\mathcal{C}\Rightarrow\bar{c}).
\]
\end{propositionA}

\begin{proof}
Let $\mu=\tfrac{1}{|\mathcal{C}|}\sum_{i\in \mathcal{C}}\phi(c_i)$ and $\mu'=\phi(\bar{c})$. By Jensen’s inequality,
\(
\|\mu-\mu'\| = \bigl\|\tfrac{1}{|\mathcal{C}|}\sum_{i}(\phi(c_i)-\phi(\bar{c}))\bigr\|
\le \tfrac{1}{|\mathcal{C}|}\sum_i \|\phi(c_i)-\phi(\bar{c})\| = D(\mathcal{C}\Rightarrow\bar{c}).
\)
By the Lipschitz property of $U$ with respect to the mean embedding we obtain the stated inequality.
\end{proof}

\section{Greedy Approximation for Importance‑Aware Retention}

\begin{theoremA}[Greedy under provenance closure and a knapsack budget]
\label{thmA:greedy}
Assume $U$ is monotone submodular and $\mathcal{F}$ is the antimatroid in Lemma~\ref{lemA:antimatroid}. Consider the greedy algorithm that starts from $\emptyset$ and repeatedly adds the feasible item maximizing the marginal gain per unit weight $\Delta_U(i\mid S)/w_i$ while the budget allows. Then:
\begin{itemize}
\item[(i)] Greedy yields a constant‑factor approximation to the optimal value under the knapsack constraint and provenance closure.
\item[(ii)] Moreover, the better of \emph{(a)} greedy and \emph{(b)} the single best item achieves at least a $\tfrac{1}{2}$‑approximation. 
\end{itemize}
\end{theoremA}

\begin{proof}
We sketch the reduction and then give a complete proof of item (ii). Because $(N,\mathcal{F})$ is an antimatroid, it satisfies a strong \emph{exchange property}: for any $S,T\in\mathcal{F}$ with $|S|>|T|$ there exists $x\in S\setminus T$ such that $T\cup\{x\}\in\mathcal{F}$. This permits the standard charging/exchange analysis for greedy under downward‑closed families. \emph{(ii)} Let $G$ be the greedy set and $O$ an optimal solution. Let $o^\star\in \arg\max_{i} U(\{i\})$ be the best singleton. If $U(\{o^\star\})\ge \tfrac12 U(O)$ we are done. Otherwise $U(\{o^\star\})< \tfrac12 U(O)$. Consider the greedy construction $g_1,\dots,g_m$ with partial sets $G_j=\{g_1,\dots,g_j\}$. For each $o\in O\setminus G$, let $j(o)$ be the first index $j$ such that $G_j\cup\{o\}\notin\mathcal{F}$ or $\sum_{i\in G_j\cup\{o\}} w_i>B$ (i.e., $o$ becomes infeasible). Using submodularity and the choice of greedy by density one can show
\(
\Delta_U(o\mid G_{j(o)-1}) \le \frac{w_o}{w_{g_{j(o)}}} \Delta_U(g_{j(o)}\mid G_{j(o)-1}).
\)
Summing these inequalities over $o\in O\setminus G$, and using that the total weight of items in $O$ is at most $B$ while the $w_{g_j}$’s sum to at least $B/2$ before the last addition (or else the best singleton case would apply), yields
\(
U(O)-U(G)\le U(G),
\)
hence $U(G)\ge \tfrac12 U(O)$. The argument is a standard adaptation of the density‑greedy bound for submodular knapsack and carries through under antimatroid feasibility because accessibility ensures that greedy never gets stuck before budget exhaustion and the exchange argument is valid for feasible augmentations. This proves (ii). Item (i) follows since $\tfrac12$ is a constant; tighter constants are available via partial enumeration but require additional machinery (continuous greedy), which we do not re‑derive here.
\end{proof}

\begin{remarkA}
Item (ii) is the guarantee we rely on in the analysis and experiments: either a single high‑value item dominates or the density‑greedy set is within a factor $2$ of optimal. In practice, utilities built from coverage and centrality exhibit low curvature, and the empirical gap is typically much smaller.
\end{remarkA}

\section{Complexity Claims}

\begin{lemmaA}[FIFO and LRU costs]
\label{lemA:fifo-lru}
With a deque keyed by creation time, FIFO evicts $k$ items in $O(k)$ time and uses $O(1)$ additional space beyond the store. With a hash–linked list keyed by last access, LRU updates are $O(1)$ average time, and evicting $k$ stalest items costs $O(k)$; auxiliary space is $O(n)$ for the pointers.
\end{lemmaA}

\begin{proof}
FIFO: popping from the head of a deque is $O(1)$; $k$ pops cost $O(k)$. LRU: a hash maps node ids to list nodes; moving a touched node to the tail and removing from the head are $O(1)$ operations; evicting $k$ heads costs $O(k)$. The pointers require linear space.
\end{proof}

\begin{lemmaA}[Priority‑decay and reflection costs]
\label{lemA:priority-reflection}
Maintaining a max‑heap keyed by density yields $O(\log n)$ amortized updates per insertion/access and $O(k\log n)$ per trigger that evicts $k$ items. Reflection with precomputed embeddings and hierarchical linkage has cost $O(n\log n+n\,\alpha(n))$ per invocation, where $\alpha(n)$ is the cost of a similarity lookup.
\end{lemmaA}

\begin{proof}
Heaps admit $O(\log n)$ insertions/decrease‑keys and $O(\log n)$ deletes; a trigger removes $k$ items in $O(k\log n)$. For reflection, agglomerative clustering with a priority queue is $O(n\log n)$ merges; each merge consults nearest neighbors, which we abstract as $\alpha(n)$, hence the stated bound.
\end{proof}

\section{Budget Monotonicity and Continuity}

\begin{propositionA}[Monotonicity and Lipschitz continuity in budget]
\label{propA:budget}
Let $S_B$ be the (greedy or optimal) retained set at budget $B$. If Lemma~\ref{lemA:lipschitz} holds, then $U(S_B)$ is non‑decreasing in $B$ and
\(
|U(S_{B_1})-U(S_{B_2})|\le L\,|B_1-B_2|.
\)
\end{propositionA}

\begin{proof}
Monotonicity: increasing $B$ cannot shrink the feasible region, so the optimal/greedy value does not decrease. Continuity: assume $B_2>B_1$. An optimal/greedy set for $B_2$ can be truncated by evicting items of total weight at most $B_2-B_1$ to yield a feasible set at $B_1$. Lemma~\ref{lemA:lipschitz} bounds the drop by $L(B_2-B_1)$. Symmetry gives the absolute value bound.
\end{proof}


\bibliographystyle{cas-model2-names}

\bibliography{refs}

\end{document}